\documentclass{article}

\usepackage{microtype}
\usepackage{graphicx}
\usepackage{subcaption}
\usepackage{booktabs} 

\usepackage{hyperref}

\usepackage[preprint]{icml2026}

\usepackage{amsmath}
\usepackage{amssymb}
\usepackage{mathtools}
\usepackage{amsthm}
\usepackage{listings}
\usepackage{xcolor}
\usepackage{stfloats}
\usepackage{multirow} 

\definecolor{commentpink}{RGB}{255, 20, 147}
\definecolor{keywordblue}{RGB}{0, 0, 255}
\definecolor{stringorange}{RGB}{255, 140, 0}

\lstdefinestyle{torchstyle}{
    language=Python,
    basicstyle=\ttfamily\small,
    keywordstyle=\color{keywordblue}\bfseries,
    commentstyle=\color{commentpink}\itshape, 
    stringstyle=\color{stringorange},
    breaklines=true,
    frame=lines,
    backgroundcolor=\color{gray!5},
    showstringspaces=false,
    captionpos=b,
}

\theoremstyle{plain}
\newtheorem{theorem}{Theorem}[section]

\newtheorem{lemma}[theorem]{Lemma}

\theoremstyle{definition}

\theoremstyle{remark}

\icmltitlerunning{Project and Generate:
  Divergence-Free Neural Operators for Incompressible Flows}

\begin{document}

\twocolumn[
  \icmltitle{Project and Generate:\texorpdfstring{\\}{ }
  Divergence-Free Neural Operators for Incompressible Flows
}




    \icmlsetsymbol{equal}{*}
    \icmlsetsymbol{corr}{$\dagger$}
    
    \begin{icmlauthorlist}
        \icmlauthor{Xigui Li}{equal,1,2}
        \icmlauthor{Hongwei Zhang}{equal,1,2}
        \icmlauthor{Ruoxi Jiang}{corr,1,2}
        \icmlauthor{Deshu Chen}{1,2}
        \icmlauthor{Chensen Lin}{1,2}
        \icmlauthor{Limei Han}{1,2}
        \icmlauthor{Yuan Qi}{1,2}
        \icmlauthor{Xin Guo}{corr,1,2}
        \icmlauthor{Yuan Cheng}{corr,1,2}
    \end{icmlauthorlist}
    
    \icmlaffiliation{1}{Artificial Intelligence Innovation and Incubation Institute, Fudan University, Shanghai, China}
    \icmlaffiliation{2}{Shanghai Academy of Artificial Intelligence for Science, Shanghai, China}
    
    \icmlcorrespondingauthor{Ruoxi Jiang}{roxie\_jiang@fudan.edu.cn}
    \icmlcorrespondingauthor{Xin Guo}{guoxin@sais.org.cn}
    \icmlcorrespondingauthor{Yuan Cheng}{cheng\_yuan@fudan.edu.cn}

  \icmlkeywords{Machine Learning, ICML}

  \vskip 0.3in
]



\makeatletter
\begingroup
\def\Hy@footnote@currentHref{affiliations}
\printAffiliationsAndNotice{\icmlEqualContribution}
\endgroup
\makeatother

\begin{abstract}

Learning-based models for fluid dynamics often operate in unconstrained function spaces, leading to physically inadmissible, unstable simulations. Penalty-based methods can reduce constraint violations but do not guarantee exact incompressibility. In this work, we introduce a unified framework that enforces the incompressible continuity equation as a hard, intrinsic constraint for both deterministic and generative modeling. We integrate a differentiable spectral Leray projection into neural operators to enforce incompressibility. For flow matching, we construct a divergence-free Gaussian reference measure via a curl-based pushforward, so that the reference samples, interpolation paths, and velocity targets are compatible with the projected dynamics. Experiments on 2D Navier–Stokes equations demonstrate exact incompressibility up to discretization error and substantially improved stability and physical consistency.
\end{abstract}

\begin{figure*}[t]
    \centering
    \includegraphics[width=\textwidth]{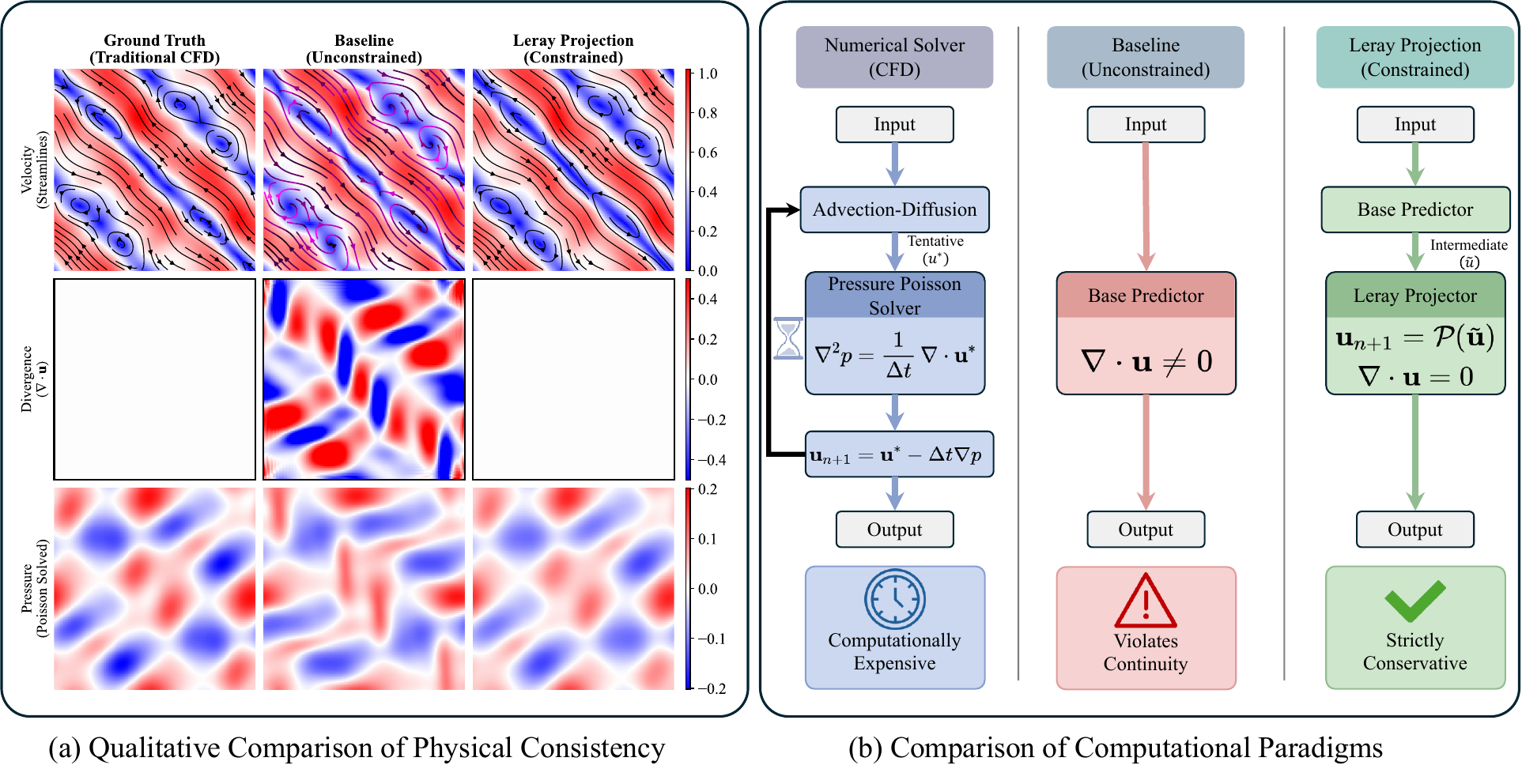} 
    \caption{
    \textbf{Motivation for the {\sc Project \& Generate} framework illustrated using numerical solvers.}
    \textbf{(a) Visual comparison.}
    Velocity streamlines (top, colored by error magnitude), divergence fields (middle), and recovered physical pressure (bottom).
    The \textit{Baseline} approach (middle), which enforces only the momentum equation, exhibits severe violations of mass conservation with nonzero divergence and noisy pressure fields.
    A \textit{Leray-projected solution} (right) enforces incompressibility via a spectral Leray projection, yielding a divergence-free velocity field ($\nabla \cdot \boldsymbol{u} = 0$) with machine precision and recovering a smooth, physically consistent pressure field.
    The result is visually indistinguishable from the \textit{Ground Truth} numerical solution (left).
    \textbf{(b) Computational paradigms.}
    Left: Classical numerical solvers (CFD) enforce incompressibility through iterative Poisson solves, which require accessing the pressure variable and are computationally expensive.
    Middle: Unconstrained end-to-end learning approaches typically violate the continuity constraint.
    Right: We leverage Leray projection to decouple dynamics from constraint enforcement, enabling hard physical constraints to be imposed via a fast, differentiable spectral operator.
    }
    \label{fig:teaser}
\end{figure*}

\section{Introduction}

Learning-based models have shown remarkable success in predicting and generating fluid flows, driven by advances in neural operators \cite{li2020fourier,lu2021learning, adrian2025data,sanz2025long} and modern generative modeling frameworks \cite{ho2020denoising,song2020score,lipmanflow,liu2023flow,chen2024probabilistic}. These methods offer substantial computational speedups over classical solvers and enable data-driven emulation of complex dynamical systems. However, most existing approaches operate in unconstrained function spaces and do not explicitly encode the physical structures that govern fluid dynamics. As a result, even models that achieve low short-term prediction error may produce velocity fields that violate fundamental physical principles, exhibiting spurious divergence, distorted energy spectra, or other nonphysical artifacts \cite{richter2022neural,hansen2023learning,krishnapriyan2021characterizing,wang2022and,shu2023physics}.

This issue is particularly pronounced for incompressible flows, where physical admissibility is determined by an exact conservation law rather than an empirical regularity. In incompressible fluid dynamics, the velocity field $\mathbf{u}$ is required
to satisfy the continuity equation
\begin{equation}
\nabla \cdot \mathbf{u} = 0,
\end{equation}
which expresses mass conservation and restricts solutions to a closed linear subspace of divergence-free vector fields. Unlike approximate physical priors, this constraint does not admit small, localized violations: any deviation necessarily alters the global flow structure. As a result, even weak divergence
errors can lead to artificial sources or sinks and qualitatively incorrect dynamics~\cite{mohan2019embedding,richter2022neural}.

From a learning perspective, however, most data-driven models are formulated and trained in unconstrained function spaces, where the divergence-free structure is not preserved by construction.
Even when training data satisfy $\nabla \cdot \mathbf{u} = 0$, approximation errors, finite model capacity, and optimization artifacts can introduce spurious divergence into predictions.
Because incompressibility defines a global subspace rather than a pointwise condition, such violations cannot be eliminated by local corrections and may accumulate over time, degrading long-term stability. 
In contrast, classical numerical solvers enforce incompressibility exactly
through projection-based schemes, typically by introducing an auxiliary pressure
variable and solving a Poisson equation at each time step. While effective, this
approach is difficult to integrate into modern learning pipelines due to its
implicit coupling and iterative nature. Motivated by this classical perspective,
we identify the Leray projection as the essential structural component that
enforces incompressibility, without explicitly introducing pressure.

A common strategy to mitigate this issue is to enforce incompressibility through soft penalties, as in physics-informed neural networks (PINNs) and related approaches \cite{raissi2019physics}. While such methods can reduce divergence in practice, they do not guarantee exact satisfaction of the constraint and often require careful tuning of penalty weights. More fundamentally, soft constraints treat incompressibility as an auxiliary optimization objective rather than as a defining property of the solution space, causing instability in challenging regimes such as high Reynolds numbers or long-term rollouts.

In this work, we adopt a different perspective and treat incompressibility as a \emph{hard structural constraint} on both model outputs and generative processes. As shown in Figure~\ref{fig:teaser}, this approach effectively eliminates the non-physical artifacts observed in baselines and recovers strict mass conservation. Our approach is based on the Leray projection arising from the Helmholtz--Hodge decomposition \cite{chorin1967numerical,chorin1968numerical}, which maps arbitrary vector fields onto the divergence-free subspace by removing the irrotational components. We integrate this projection as a modular operator that can be applied consistently across both learning and inference paradigms, without modifying the underlying model architecture. This operator-level enforcement ensures that all predicted velocity fields lie exactly in the physically admissible function space.

For regression tasks, the proposed projection guarantees incompressibility of model predictions by construction. For generative modeling, however, enforcing divergence-free structure poses additional challenges. In generative frameworks such as flow matching, not only the generated samples but also the reference measure, interpolation paths, and learned dynamics must remain compatible with the underlying physical constraint. Naively projecting samples or velocities can lead to inconsistencies between probability measures and may render the probability path ill-defined, particularly in infinite-dimensional settings where mutual absolute continuity is not guaranteed. Building on the functional flow matching perspective \cite{kerrigan2024functional}, we extend our framework to generative modeling by constructing probability paths that evolve entirely within the divergence-free subspace. To this end, we design the reference measure to be compatible with the projection operator, so that the induced probability path is well-defined on the space of incompressible velocity fields. This alignment ensures that all intermediate states and learned dynamics remain physically admissible throughout the generation.

While all models are implemented on finite-resolution grids, our methodology is
motivated by a function-space perspective, in which both regression and
generation are viewed as approximations to operators acting on spaces of
divergence-free velocity fields. From this viewpoint, the Leray projection and
the design of the reference measure are structural components of the modeling
pipeline, rather than auxiliary regularization mechanisms.
Empirical results on two-dimensional Navier--Stokes equations demonstrate that enforcing incompressibility as a hard constraint leads to improved physical fidelity, enhanced stability under long-term rollout, and superior performance.

Our main contributions are summarized as follows:
\begin{itemize}
\item We propose a unified framework that treats incompressibility as an intrinsic geometric property of the function space, enabling exact physical enforcement across both regression and generative modeling tasks.
\item We construct a divergence-free Gaussian reference measure based on stream functions and combine it with Leray-projected dynamics to formulate flow matching within the divergence-free subspace.
\item We demonstrate through extensive experiments on two-dimensional Navier--Stokes equations that the proposed approach significantly improves physical fidelity and long-term stability. 
\end{itemize}

\section{Related Work}
\paragraph{Data-Driven Neural Operators.}
Neural operators have emerged as efficient surrogates for PDE solvers by learning
mappings between infinite-dimensional function spaces.
Methods such as the Fourier Neural Operator~\cite{li2020fourier} and
DeepONet~\cite{lu2021learning} achieve substantial speedups by predicting flow
evolution in a single forward pass, and have been applied to a wide range of
fluid problems~\citep{pathak2022fourcastnet, price2023gencast, lippe2023pde, azizzadenesheli2024neural, jiang2025hierarchical, barthel2026probabilistic}.
However, deterministic operators often exhibit spectral bias, leading to
over-smoothed predictions that under-represent high-frequency turbulent
structures~\citep{jiang2023training, falasca2025probing}.
To address this limitation, recent work has explored stochastic and generative
models for fluid dynamics, including diffusion and flow-matching-based
approaches~\citep{molinaro2024generative, kohl2024turbulent}.
Despite their improved expressivity, these models typically operate in
unconstrained vector spaces and may violate fundamental physical constraints
such as incompressibility.

\paragraph{Physics-Informed Constraints in Learning.}
Physical constraints can be incorporated into learning frameworks through
soft penalties, structure-preserving models, and projection-based methods.
PINNs~\cite{raissi2019physics} and related extensions penalize PDE residuals
during training, an approach that has also been adopted in generative
models~\citep{ben2024d,shu2023physics,huang2024diffusionpde,bastek2025physicsinformed}.
Despite their flexibility, such soft constraints provide no guarantee of
exact physical validity and often lead to ill-conditioned optimization
landscapes~\cite{wang2022and}.
Related approaches in differentiable physics, such as
Hamiltonian~\cite{greydanus2019hamiltonian} and Lagrangian neural
networks~\cite{cranmer2020lagrangian}, encode physical structure directly
into the model to preserve invariants.
Projection-based methods provide another route, with some applying
corrections only at inference
time~\cite{sojitra2025method,utkarsh2025physicsconstrained,cheng2025eci}
and others incorporating differentiable projections into the model.
For incompressible flows,~\cite{jiang2019enforcing} introduced a
differentiable PDE layer for CNNs and applied it to conditional GANs
for turbulence super-resolution, while~\cite{duruisseaux2024towards} studied Fourier-space projections
for imposing hard physical constraints in neural operators,
including incompressibility.
In our flow-matching setting, we pair the Leray projection with a
divergence-free reference measure so that the interpolation paths
also satisfy the incompressibility constraint.

\section{Preliminaries}

\subsection{Notations and Functional Spaces}
Let $\mathbb{R}$ and $\mathbb{Z}$ denote the sets of real numbers and integers,
respectively, and define the unit torus $\mathbb{T}^2 := \mathbb{R}^2 / \mathbb{Z}^2$.
We denote by $\nabla$ and $\Delta$ the spatial gradient and Laplacian operators.
For a vector field $\mathbf{u} = (u_1, u_2)^\top$ on $\mathbb{T}^2$, we define the
scalar curl as $\nabla \times \mathbf{u} := \partial_1 u_2 - \partial_2 u_1$ and
the vector curl as $\nabla^\perp \psi := (\partial_2 \psi, -\partial_1 \psi)^\top$.

Let $L^2_{\mathrm{per}}(\mathbb{T}^2;\mathbb{R}^2)$ denote the Hilbert space of square-integrable periodic vector fields, and $H^r_{\mathrm{per}}$ the periodic Sobolev spaces of order $r \ge 0$. Unless otherwise specified, all norms and inner products are taken in $L^2$.

\subsection{Navier--Stokes Equations and Hodge Decomposition}

\paragraph{Incompressible Dynamics.}
We consider the 2D incompressible Navier--Stokes equations on the torus $\mathbb{T}^2$: 
\begin{equation}
\left\{
\begin{aligned}
    \partial_t \mathbf{u} + (\mathbf{u} \cdot \nabla)\mathbf{u}
&= \nu \Delta \mathbf{u} + \mathbf{f} - \nabla p, \\
\qquad \nabla \cdot \mathbf{u} &= 0 ,
\end{aligned}
\right.
\end{equation}
where $\mathbf{u}$ is the velocity field and $\nu > 0$ the viscosity.

\paragraph{Hodge Decomposition and the Subspace $\mathcal{V}$.}
To handle the incompressibility constraint $\nabla \cdot \mathbf{u} = 0$, we invoke the Hodge decomposition. Throughout this work, we restrict attention to zero-mean velocity fields, thereby excluding the harmonic (constant) component. Under this convention, $L^2_{\mathrm{per}}$ admits the orthogonal decomposition
\begin{equation}
L^2_{\text{per}} = \mathcal{V} \oplus \mathcal{\mathcal{V}^\perp},
\end{equation}
where $\mathcal{V}$ is the subspace of divergence-free vector fields with zero mean (solenoidal fields), and $\mathcal{V}^\perp$ is the subspace of gradient fields (irrotational fields):
\begin{equation}
    \begin{aligned}
        \mathcal{V} &:= \left\{ \mathbf{v} \in L^2_{\text{per}} \;\middle|\; \nabla \cdot \mathbf{v} = 0, \int_{\mathbb{T}^2} \mathbf{v} \, dx = 0 \right\}, \\
        \mathcal{\mathcal{V}^\perp} &:= \left\{ \nabla q \;\middle|\; q \in H^1_{\text{per}}(\mathbb{T}^2; \mathbb{R}) \right\}.
    \end{aligned}
\end{equation}
\section{Method}
\label{sec:methods}
In this section, focusing on 2D incompressible flows, we present a unified framework for learning incompressible fluid dynamics with hard divergence-free constraints.
By embedding incompressibility directly into the operator structure, all model states including intermediate dynamics, final outputs, and stochastic perturbations, are restricted to the divergence-free
subspace by construction. As illustrated in Figure~\ref{fig:method_framework}, this is realized through two complementary mechanisms:
(i) a Leray projection enforcing incompressibility, and
(ii) a divergence-free noise construction for generative modeling.

\begin{figure*}[t]
    \centering
    \includegraphics[width=\textwidth]{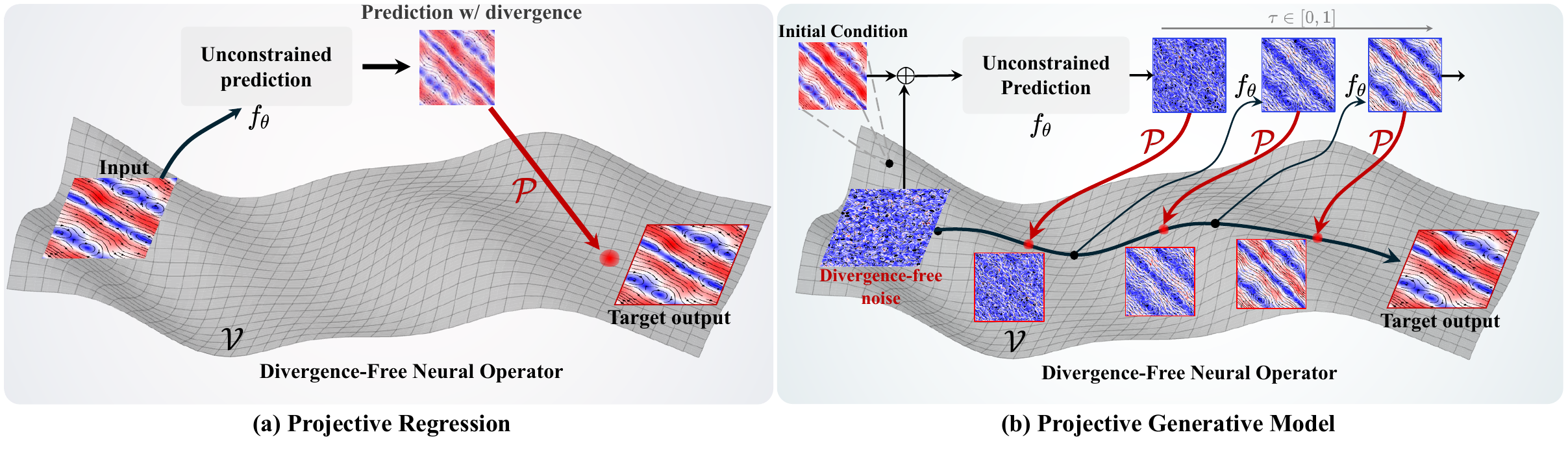} 
    \caption{\textbf{Schematic of the proposed Projective Framework.} The grey surface represents the divergence-free subspace $\mathcal{V}$ (the physical manifold).
    (a) \textbf{Projective Regression:} The framework takes an input and uses a base neural operator $f_\theta$ to produce an intermediate prediction in the ambient space (shown hovering above the manifold). This unconstrained prediction is then strictly mapped onto $\mathcal{V}$ via the spectral Leray projection $\mathcal{P}$. The composite system constitutes a \textit{Divergence-Free Neural Operator}, ensuring the final output is physically consistent.
    (b) \textbf{Projective Generative Model:} The generation process is initialized with a sample from a \textbf{divergence-free noise} distribution (bottom left) constrained to $\mathcal{V}$. During the probability flow evolution ($\tau \in [0, 1]$), the vector field predicted by $f_\theta$ is continuously projected via $\mathcal{P}$ (red arrows) back onto the manifold. This enforces hard constraints at every step of the flow matching process, ensuring the entire trajectory stays within the physically admissible subspace.}
    \label{fig:method_framework}

\end{figure*}

\subsection{Enforcing Incompressibility via Leray Projection}
\label{sec:leray}
To ensure that learned velocity fields lie strictly in the divergence-free subspace, we embed the Leray projection $\mathcal{P}$ as a structural component of the neural operator. Unlike soft-penalty methods (e.g., PINNs), this construction restricts the hypothesis space to $\mathcal{V}$, ensuring physical consistency regardless of the optimization state.

Our approach builds on the Helmholtz--Hodge decomposition. Under periodic boundary conditions, the space
$L^2_{\mathrm{per}}(\mathbb{T}^2;\mathbb{R}^2)$ admits an orthogonal decomposition into divergence-free and irrotational components.
We denote by
$\mathcal{P} : L^2_{\mathrm{per}} \to \mathcal{V}$
the associated Leray projection onto the divergence-free subspace.

For any $\mathbf{w} \in L^2_{\mathrm{per}}$, the operator $\mathcal{P}$ denotes the
$L^2$-orthogonal projection onto the divergence-free subspace $\mathcal{V}$,
which removes the gradient component and yields the closest divergence-free
field in the $L^2$ sense, equivalently characterized by
\begin{equation}
\mathcal{P}\mathbf{w}
=
\operatorname*{argmin}_{\mathbf{u} \in \mathcal{V}}
\| \mathbf{w} - \mathbf{u} \|_{L^2}.
\end{equation}
By composing the neural operator output with $\mathcal{P}$, incompressibility is enforced by construction, without auxiliary loss terms or post-hoc corrections.
\paragraph{Spectral Implementation.}
Differentiable spectral projections for enforcing incompressibility have also been studied in CNNs and neural operators~\cite{jiang2019enforcing,duruisseaux2024towards}.
Under periodic boundary conditions, the Leray projector admits an efficient Fourier-space implementation. Let $\hat{\mathbf{w}}(\mathbf{k})$ denote the Fourier coefficient of $\mathbf{w}$ at wave vector $\mathbf{k} = (k_x, k_y) \in \mathbb{Z}^2$. The
incompressibility constraint $\nabla \cdot \mathbf{u} = 0$ is then characterized by
$\mathbf{k} \cdot \hat{\mathbf{u}}(\mathbf{k}) = 0$.
For $\mathbf{k} \neq \mathbf{0}$, the Leray projection is given by
\begin{equation}
\label{eq:leray_fourier}
\widehat{\mathcal{P}\mathbf{w}}(\mathbf{k})
=
\left(
\mathbf{I}
-
\frac{\mathbf{k} \otimes \mathbf{k}}{\|\mathbf{k}\|_2^2}
\right)
\hat{\mathbf{w}}(\mathbf{k}),
\end{equation}
where $\mathbf{I}$ denotes the $2 \times 2$ identity matrix.
The zero mode is set to zero to enforce a zero-mean condition, thereby ensuring the uniqueness of the Helmholtz–Hodge decomposition under periodic boundaries (see Appendix \ref{pseudocode} for the corresponding code implementation).

\begin{algorithm}[t]
\caption{Leray Projection in Fourier Space}
\label{alg:leray_projection}
\begin{algorithmic}[1]
\REQUIRE Vector field $\mathbf{v} \in L^2_{\mathrm{per}}(\mathbb{T}^2;\mathbb{R}^2)$
\ENSURE Divergence-free field $\mathcal{P}\mathbf{v} \in \mathcal{V}$
\STATE Compute Fourier transform: $\hat{\mathbf{v}} \gets \mathcal{F}[\mathbf{v}]$
\FOR{each wave vector $\mathbf{k} \neq \mathbf{0}$}
    \STATE $\alpha \gets (\mathbf{k} \cdot \hat{\mathbf{v}}(\mathbf{k})) / \|\mathbf{k}\|_2^2$
    \STATE $\widehat{\mathcal{P}\mathbf{v}}(\mathbf{k}) \gets \hat{\mathbf{v}}(\mathbf{k}) - \alpha \mathbf{k}$
\ENDFOR
\STATE Set $\widehat{\mathcal{P}\mathbf{v}}(\mathbf{0}) \gets \mathbf{0}$
\STATE $\mathcal{P}\mathbf{v} \gets \Re(\mathcal{F}^{-1}[\widehat{\mathcal{P}\mathbf{v}}])$
\STATE \textbf{return} $\mathcal{P}\mathbf{v}$
\end{algorithmic}
\end{algorithm}

\subsection{Learning Divergence-Free Vector Fields}

We parametrize time-dependent velocity fields $v_\theta(\cdot,t)$ using neural
operators. To enforce incompressibility, the Leray projection is composed with
the model output, ensuring that all learned dynamics evolve in $\mathcal{V}$.

\paragraph{Regression.}
For deterministic prediction tasks, a neural operator
$f_\theta$ produces an unconstrained velocity field prediction
$f_\theta(\cdot,t) \in L^2_{\mathrm{per}}(\mathbb{T}^2;\mathbb{R}^2)$.
We enforce incompressibility by applying the Leray projection pointwise in time,
\begin{equation}
\hat{\mathbf{u}}(\cdot,t) = \mathcal{P} f_\theta(\cdot,t),
\end{equation}
ensuring the resulting prediction remains divergence-free by construction.

This projection-based formulation contrasts with PINNs, where incompressibility is typically imposed via soft penalty terms in the loss function. Such approaches require balancing multiple objectives (including data fidelity, divergence penalties, and boundary conditions), often necessitating problem-specific tuning of loss weights. In contrast, composing the neural operator with the Leray projection enforces incompressibility at the
level of the hypothesis space, eliminating auxiliary divergence losses and
ensuring physical admissibility by design.

\paragraph{Generative Modeling via Flow Matching.}
Standard flow matching is formulated in the ambient $L^2$ space and does not
explicitly account for incompressibility constraints. We instead restrict the learned probability flow to the divergence-free subspace $\mathcal{V}$.

Flow Matching performs regression on a target velocity field defined along an
interpolation between data samples and reference samples. We therefore enforce
incompressibility by ensuring that this regression target lies in
$\mathcal{V}$. To this end, we parametrize an unconstrained vector field
\[
v_\theta : L^2_{\mathrm{per}} \times [0,1] \to L^2_{\mathrm{per}},
\]
and enforce incompressibility by composing it with the Leray projection.

Let $\nu$ denote the data distribution supported on $\mathcal{V}$, and let $\mu_0$ be a reference distribution. Given
$\mathbf{u}_1 \sim \nu$ and $\mathbf{u}_0 \sim \mu_0$, define the linear
interpolation
\begin{equation}
\mathbf{u}_\tau = (1 - \tau)\mathbf{u}_0 + \tau \mathbf{u}_1,
\qquad \tau \sim \mathcal{U}(0,1),
\end{equation}
with target velocity $\mathbf{v}_\tau = \mathbf{u}_1 - \mathbf{u}_0$.

When both endpoints lie in $\mathcal{V}$, $\mathbf{v}_\tau$ is divergence-free due to linearity. Under this condition, the resulting
divergence-free flow matching objective is
\begin{equation}
\label{eq:div_fm_objective}
\mathcal{L}(\theta)
=
\mathbb{E}_{\tau, \mathbf{u}_0, \mathbf{u}_1 \mid c}
\left[
\bigl\|
\mathcal{P} v_\theta(\mathbf{u}_\tau, \tau; c)
-
\mathbf{v}_\tau
\bigr\|_{L^2}^2
\right].
\end{equation}

Here \(c\) denotes conditioning information, such as observed states at earlier
time instances or physical initial conditions.
This objective enforces incompressibility at the level of the learned velocity
field. However, consistency of the induced probability flow additionally requires the reference distribution $\mu_0$ to be supported on $\mathcal{V}$, posing a nontrivial technical challenge, particularly in infinite-dimensional settings. We address this issue in the following subsection.

\subsection{Divergence-Free Gaussian Noise via Curl Pushforward}
\label{sec:df-noise}

While the Leray projection guarantees that the learned velocity field is
divergence-free, it is not sufficient to ensure a well-defined flow matching
formulation. If the reference noise distribution is not supported on
$\mathcal{V}$, the regression target generally contains compressible components, whereas the model output is constrained to be divergence-free after projection, resulting in a structurally inconsistent objective and ill-posed regression problem. Moreover, probability mass may leave the physically admissible space, and the induced continuity equation may become ill-defined in infinite-dimensional settings. To ensure consistency between the regression target, the learned dynamics, and the underlying probability measures, the reference noise must itself be supported on $\mathcal{V}$, motivating the divergence-free Gaussian construction introduced below.

\paragraph{Curl-Based Gaussian Reference Measure on $\mathcal{V}$.}
We construct a divergence-free Gaussian reference measure using the stream function formulation of 2D incompressible flows.

Let $\mathcal{X} := H^1_{\mathrm{per}}(\mathbb{T}^2; \mathbb{R})
\cap \left\{ \psi \;:\; \int_{\mathbb{T}^2} \psi = 0 \right\}$
denote the space of zero-mean stream functions, and define the operator
$\nabla^\perp : \mathcal{X} \to L^2_{\mathrm{per}}(\mathbb{T}^2;\mathbb{R}^2)$ by
$\nabla^\perp \psi = (\partial_2 \psi, -\partial_1 \psi)^\top$.

\begin{lemma}
\label{lem:curl-properties}
The operator $\nabla^\perp$ is a bounded linear map whose image is exactly the
closed subspace of divergence-free, zero-mean vector fields in $L^2$.
\end{lemma}

This property allows us to lift Gaussian measures from the scalar space
$\mathcal{X}$ to the divergence-free velocity space $\mathcal{V}$.
Let $\gamma_\psi = \mathcal{N}(0, \mathcal{C}_\psi)$ be a centered Gaussian measure on $\mathcal{X}$, and define the reference measure
\[
\mu_0 := (\nabla^\perp)_{\#} \gamma_\psi .
\]
Then $\mu_0$ is a centered Gaussian measure supported on $\mathcal{V}$, with
covariance operator
$\mathcal{C}_{\mathbf{u}} = \nabla^\perp \mathcal{C}_\psi (\nabla^\perp)^*$ and
Cameron--Martin space
$\mathcal{H}_{\mu_0} = \nabla^\perp(\mathcal{H}_{\gamma_\psi})$.

\paragraph{Conditional Probability Paths.}
Conditioned on a fixed data sample $\mathbf{y} \in \mathcal{V}$, we define a
family of time-indexed random fields
$(\mathbf{u}_\tau^{\mathbf{y}})_{\tau \in [0,1]}$
via an affine interpolation in function space,
\begin{equation}
\label{eq:conditional_path_simple}
\mathbf{u}_\tau^{\mathbf{y}} = \sigma_\tau \mathbf{u}_0 + \tau \mathbf{y},
\qquad
\mathbf{u}_0 \sim \mu_0,
\end{equation}
where $\sigma_\tau = 1 - (1 - \sigma_{\min})\tau$ is a variance schedule with
$\sigma_{\min} > 0$.

For each fixed $\tau \in [0,1)$, the random field $\mathbf{u}_\tau^{\mathbf{y}}$
follows a Gaussian distribution on $\mathcal{V}$ with mean $\tau \mathbf{y}$
and covariance operator $\mathcal{C}_\tau = \sigma_\tau^2 \mathcal{C}_{\mathbf{u}}$.
In particular, all such conditional distributions share the same covariance
structure and differ only in their means.

Under the assumption that the data distribution $\nu$ is supported on the
Cameron--Martin space $\mathcal{H}_{\mu_0}$ of the reference Gaussian measure,
these conditional Gaussian distributions are mutually absolutely continuous for
$\nu$-almost every $\mathbf{y}$.

\paragraph{Existence of a Divergence-Free Flow Matching Vector Field.}
A key difficulty in infinite-dimensional flow matching is the potential
mutual singularity of probability measures, which can obstruct the
definition of a global velocity field. To ensure well-posedness, we
construct probability paths that remain absolutely continuous by design.

Let $\nu \in \mathcal{P}(\mathcal{V})$ denote the data distribution.
For each $\mathbf{y} \sim \nu$, the conditional path
$(\mu_\tau^{\mathbf{y}})_{\tau \in [0,1]}$ is defined as in
Eq.~\eqref{eq:conditional_path_simple}, with $\mathbf{u}_0 \sim \mu_0$.
The corresponding marginal path
\[
\mu_\tau := \int \mu_\tau^{\mathbf{y}} \, \mathrm{d}\nu(\mathbf{y})
\]
is therefore supported on $\mathcal{V}$.
Under the assumption that all conditional measures share the same
covariance operator and differ only by admissible Cameron--Martin shifts,
the measures $\mu_\tau^{\mathbf{y}}$ are mutually absolutely continuous.

\begin{theorem}[Existence of a Divergence-Free Flow Matching Vector Field]
\label{thm:div-free-fm}
Let $(\mu_\tau)_{\tau \in [0,1]}$ be the marginal path on $\mathcal{V}$ constructed above. Under the absolute continuity condition $\mu_\tau^{\mathbf{y}} \ll \mu_\tau$, the Radon--Nikodym derivative
\[
w_\tau(\mathbf{x}, \mathbf{y}) := \frac{\mathrm{d}\mu_\tau^{\mathbf{y}}}{\mathrm{d}\mu_\tau}(\mathbf{x})
\]
is well-defined for $\mu_\tau$-almost every $\mathbf{x}$. The vector field
\begin{equation}
\label{eq:div-free-fm}
\mathbf v_\tau(\mathbf x)
=
\int_{\mathcal V}
\mathbf v_\tau^{\mathbf y}(\mathbf x)\,
w_\tau(\mathbf x, \mathbf y)\,
\mathrm d\nu(\mathbf y)
\end{equation}
generates the path $(\mu_\tau)_{\tau \in [0,1]}$, where $\mathbf{v}_\tau^{\mathbf{y}}$ denotes the velocity field associated with the conditional path $\mu_\tau^{\mathbf{y}}$.
\end{theorem}

The proof is provided in Appendix~\ref{proof:fm}. In theory, the existence result relies on the assumption that the data distribution $\nu$ is supported on the Cameron--Martin space associated with the reference covariance operator. Verifying this condition directly is generally intractable, particularly in infinite-dimensional settings. In practice, however, our models operate on finite-resolution discretizations, where all probability measures are supported on finite-dimensional subspaces and the above absolute continuity condition is automatically satisfied. Consequently, following prior work on functional flow matching~\cite{kerrigan2024functional}, we do not explicitly enforce or verify this assumption in our experiments. A rigorous characterization of when such conditions hold in the infinite-dimensional limit is left for future work.

\begin{figure*}[t]
    \centering
    \begin{subfigure}[b]{0.48\textwidth}
        \centering
        \includegraphics[width=\linewidth]{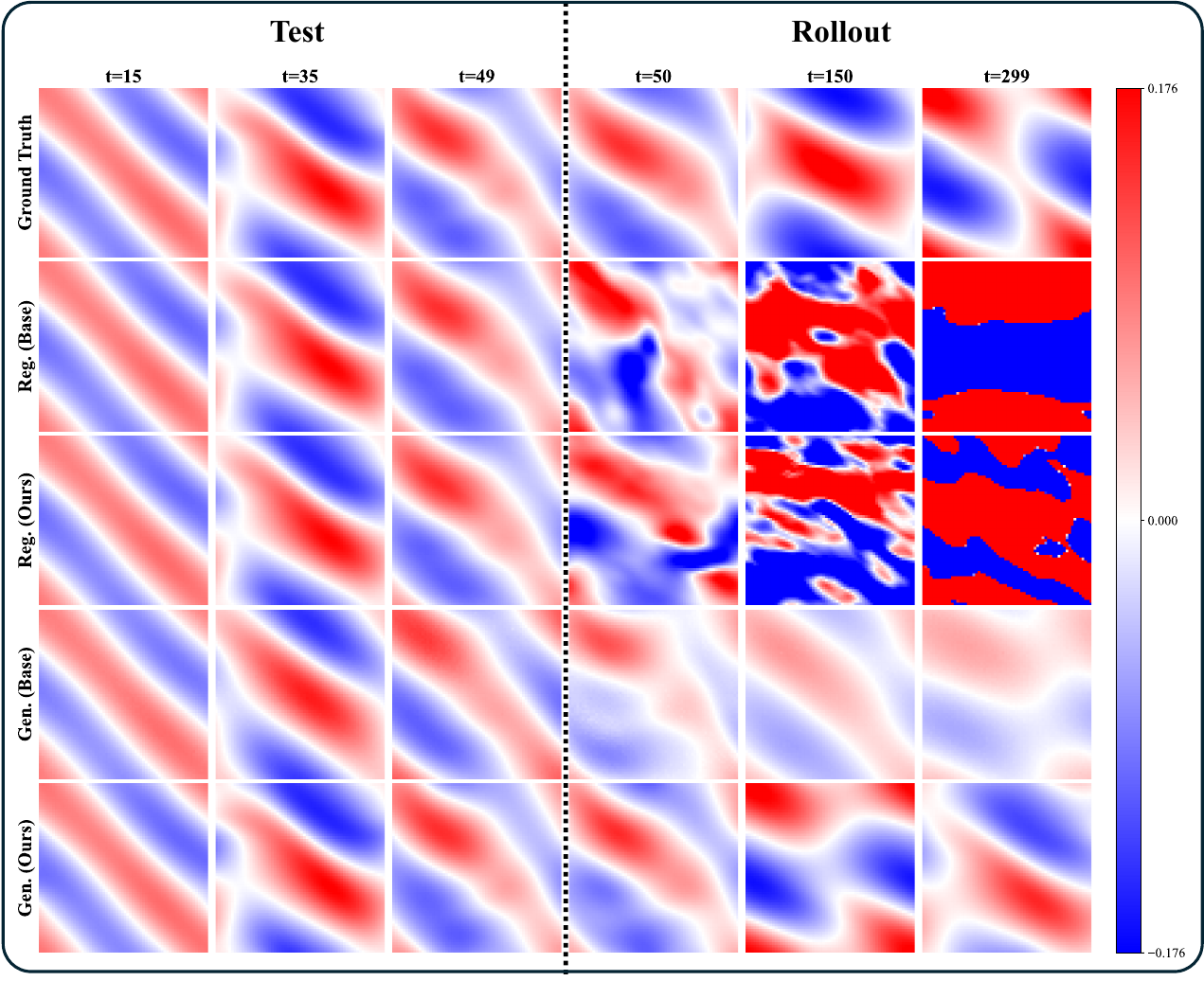}
        \caption{$u$ velocity}
        \label{fig:res_u}
    \end{subfigure}
    \hfill
    \begin{subfigure}[b]{0.48\textwidth}
        \centering
        \includegraphics[width=\linewidth]{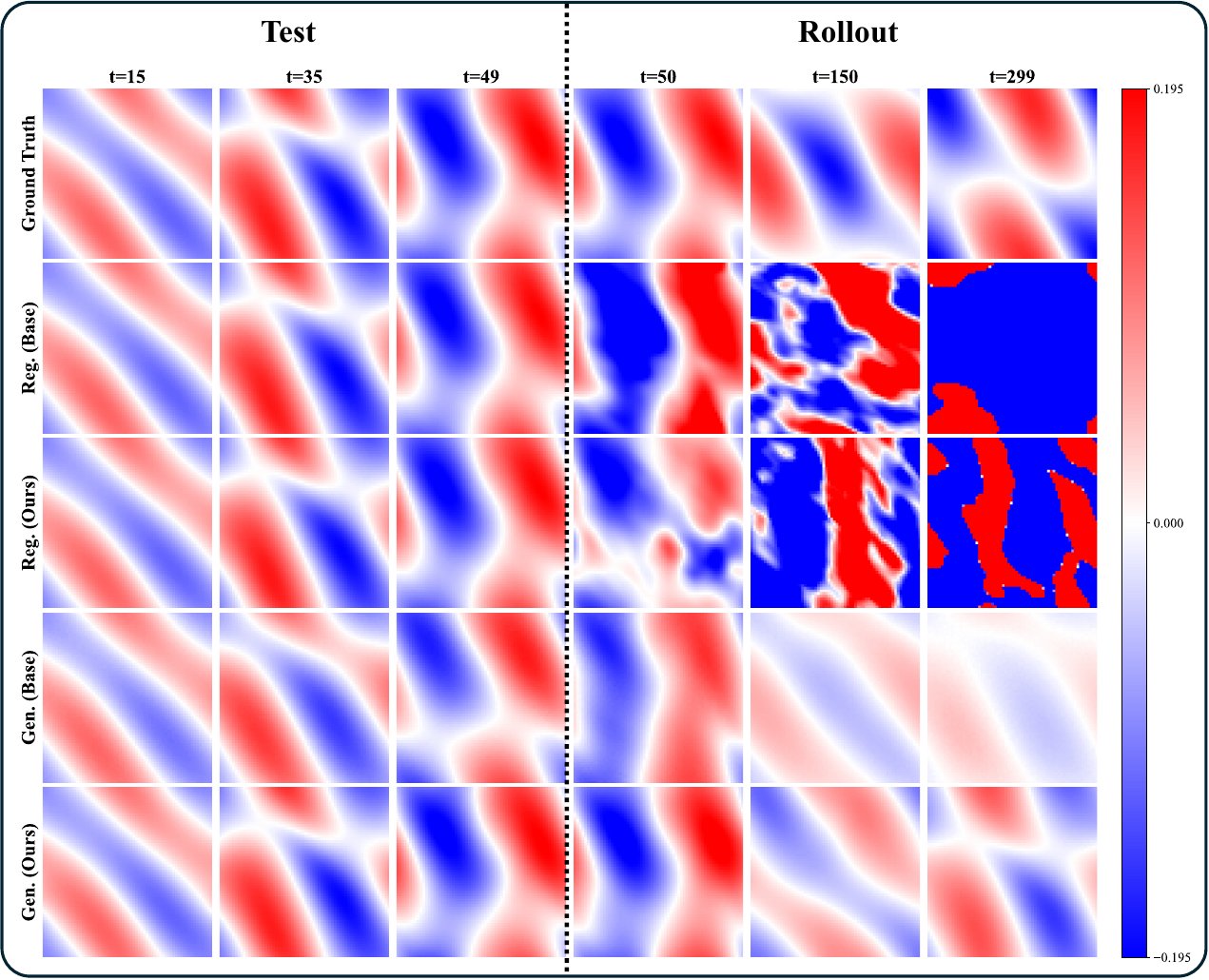}
        \caption{$v$ velocity}
        \label{fig:res_v}
    \end{subfigure}
    
    
    \begin{subfigure}[b]{0.48\textwidth}
        \centering
        \includegraphics[width=\linewidth]{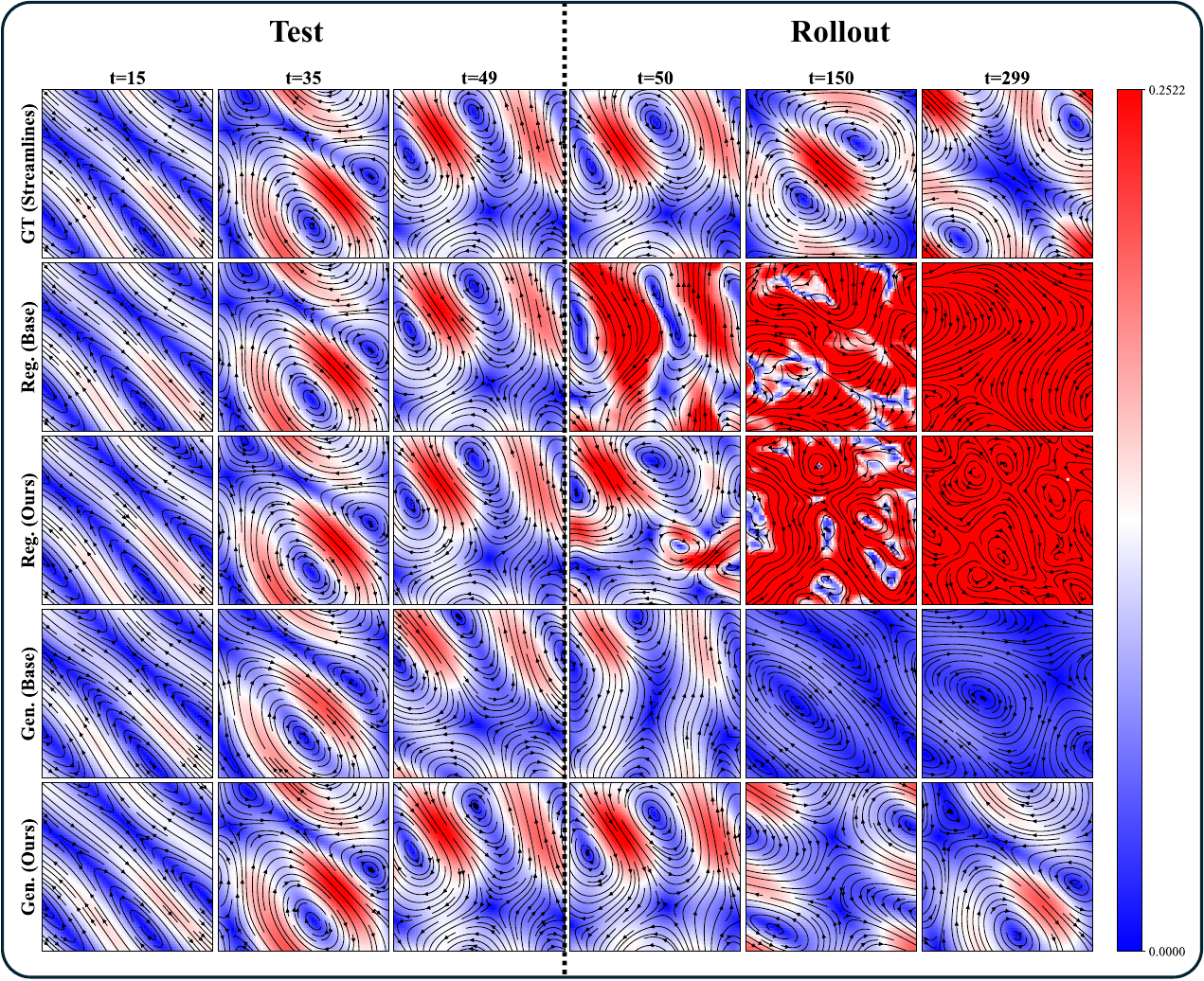}
        \caption{Streamlines}
        \label{fig:res_streamline}
    \end{subfigure}
    \hfill
    \begin{subfigure}[b]{0.48\textwidth}
        \centering
        \includegraphics[width=\linewidth]{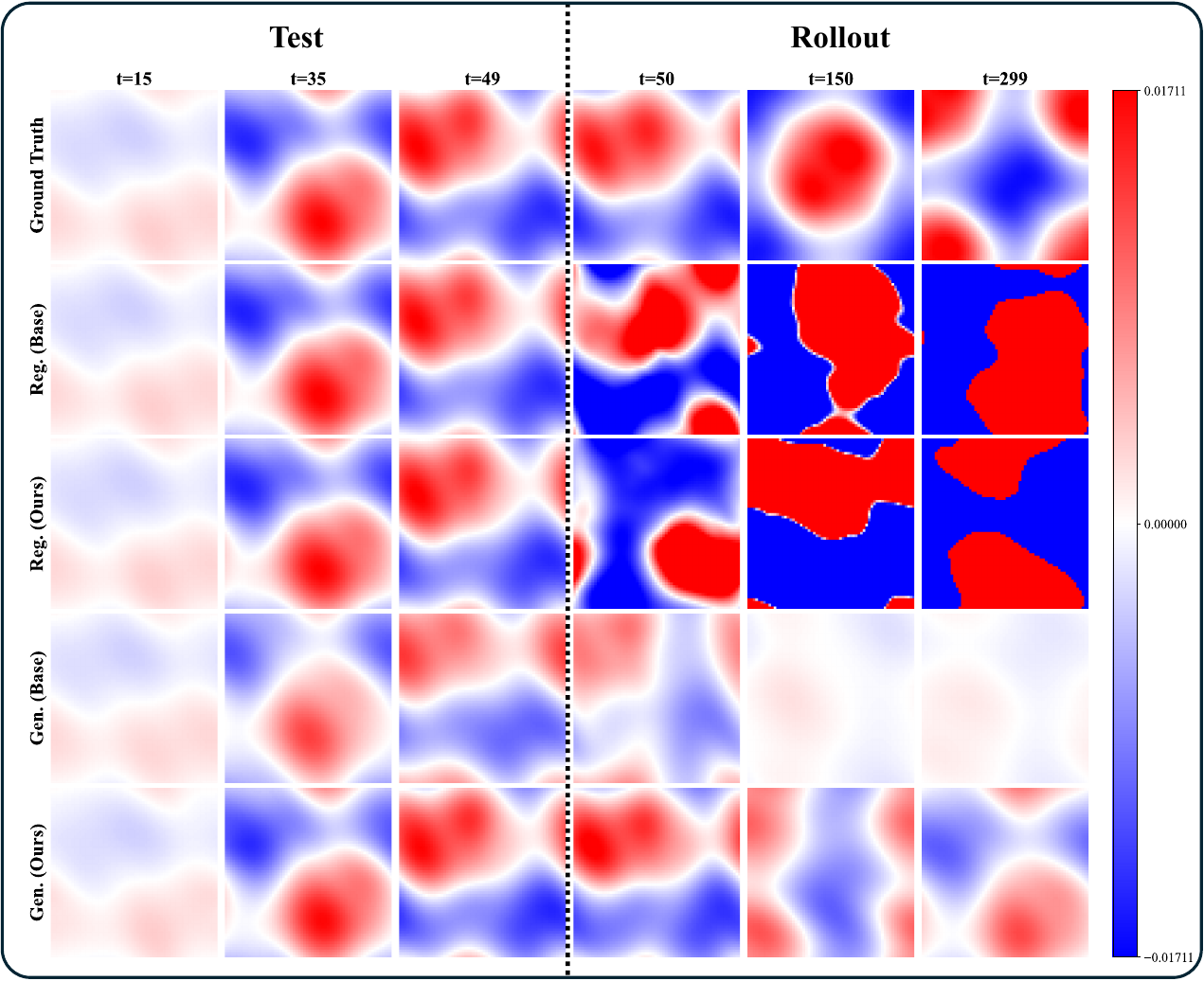}
        \caption{Reconstructed Pressure}
        \label{fig:res_p}
    \end{subfigure}
    
    \caption{\textbf{Flow Field Visualization.} 
    \textbf{Top Row:} Snapshot of velocity components ($u, v$). 
    \textbf{Bottom Row:} Streamlines and Pressure field. The pressure is reconstructed from the predicted velocity field via the pressure Poisson equation, serving as a proxy for the dynamical consistency of the flow structure.}
    \label{fig:vis_fields}
\end{figure*}

\section{Experiments}
\label{sec:experiments}

In this section, we validate the effectiveness of our proposed framework on two-dimensional incompressible turbulent flows. 
We compare our method against state-of-the-art regression and generative baselines, focusing on physical consistency (divergence, spectrum) and long-term stability.
Specifically, the evaluation centers on two primary objectives: (1) demonstrating that enforcing incompressibility as a hard constraint improves prediction accuracy and stability compared to unconstrained models; and (2) establishing that our generative framework produces diverse, high-fidelity rollouts that rigorously respect conservation laws without sacrificing sample quality.

\begin{table*}[t]
    \centering
    \renewcommand{\arraystretch}{1.25} 
    \setlength{\tabcolsep}{3pt}     
    
    \caption{\textbf{Quantitative Comparison on Re = 1000.} 
    We report \textbf{Mean $\pm$ Std} over the test set. 
    \textbf{Bold} indicates best performance. 
    \textcolor{gray}{Gray values} indicate model divergence.
    Prediction and Short-term metrics are scaled by $10^{-3}$ and $10^{-2}$ respectively.}
    \label{tab:re1000_resize}
    
    \resizebox{\textwidth}{!}{%
    \begin{tabular}{l ccc ccc ccc}
        \toprule
        \multirow{2.5}{*}{\textbf{Model}} & 
        \multicolumn{3}{c}{\textbf{Prediction} ($T_{15\text{-}49}$)} & 
        \multicolumn{3}{c}{\textbf{Short-term} ($T_{50\text{-}100}$)} & 
        \multicolumn{3}{c}{\textbf{Long-term} ($T_{101\text{-}300}$)} \\
        \cmidrule(lr){2-4} \cmidrule(lr){5-7} \cmidrule(lr){8-10}
         & U-MSE & V-MSE & Div & U-MSE & V-MSE & Div & U-MSE & V-MSE & Div \\
        \midrule
        
        \multicolumn{10}{l}{\textit{\textbf{Regression Models}}} \\
        \hspace{1em}Base 
            & $4.0 \pm 3.6$ & $5.6 \pm 5.4$ & $625.0 \pm 576.0$ 
            & $16.6 \pm 17.5$ & $17.5 \pm 15.5$ & $21.7 \pm 18.6$ 
            & \textcolor{gray}{$4.3\text{e}3 \pm 5.1\text{e}3$} & \textcolor{gray}{$4.4\text{e}3 \pm 5.2\text{e}3$} & \textcolor{gray}{$1.5\text{e}5 \pm 1.9\text{e}5$} \\
        \hspace{1em}Ours 
            & $\mathbf{3.4 \pm 3.0}$ & $\mathbf{3.2 \pm 2.9}$ & $\mathbf{0.1 \pm 0.1}$ 
            & $\mathbf{5.9 \pm 4.3}$ & $\mathbf{6.9 \pm 5.4}$ & $\mathbf{0.0 \pm 0.0}$ 
            & \textcolor{gray}{$227.0 \pm 335.0$} & \textcolor{gray}{$210.0 \pm 328.0$} & \textcolor{gray}{$118.0 \pm 160.0$} \\
            
        \addlinespace[0.6em] 
        
        \multicolumn{10}{l}{\textit{\textbf{Generative Models}}} \\
        \hspace{1em}Base 
            & $4.3 \pm 3.6$ & $4.6 \pm 4.3$ & $1.7 \pm 1.6$ 
            & $1.0 \pm 0.7$ & $1.2 \pm 1.4$ & $0.6 \pm 2.5$ 
            & \textcolor{gray}{Inf} & \textcolor{gray}{Inf} & \textcolor{gray}{Inf} \\
            
        \hspace{1em}Ours 
            & $\mathbf{1.0 \pm 1.2}$ & $\mathbf{1.1 \pm 1.3}$ & $\mathbf{4.0\text{e-}4 \pm 2.0\text{e-}4}$ 
            & $\mathbf{0.5 \pm 0.3}$ & $\mathbf{0.6 \pm 0.4}$ & $\mathbf{2.0\text{e-}5 \pm 1.0\text{e-}5}$ 
            & $\mathbf{7.0\text{e-}3 \pm 4.0\text{e-}3}$ & $\mathbf{7.0\text{e-}3 \pm 4.0\text{e-}3}$ & $\mathbf{1.6\text{e-}7 \pm 4.1\text{e-}8}$ \\
        
        \bottomrule
    \end{tabular}
    } 
\end{table*}

\begin{figure*}[t]
    \centering
    \includegraphics[width=\linewidth]{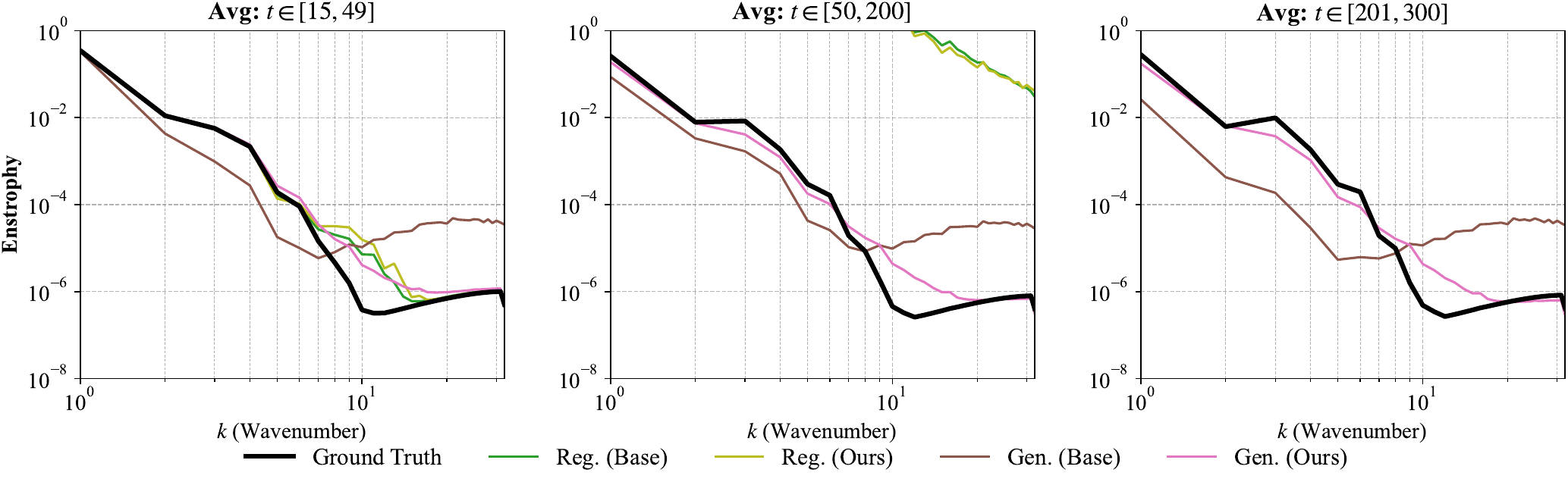}
    \caption{\textbf{Enstrophy Spectrum Analysis.} Comparison of energy spectra. Our method accurately recovers the inertial range scalings, preserving the correct energy cascade across scales.}
    \label{fig:res_spectrum}
\end{figure*}

\subsection{Experimental Setup}

\paragraph{Dataset and Implementation Details.}

Experiments focus on the 2D incompressible Navier--Stokes equations with periodic boundary conditions at $\text{Re} = 1000$. The dataset, generated via a pseudo-spectral solver, contains 10,000 training trajectories ($64 \times 64$, 50 timesteps). Following benchmarks \cite{li2020fourier, cheng2025eci}, the task involves predicting 35-frames from a 15-frame history, with long-term rollouts (300 steps) to assess stability. We evaluate performance using mean squared error (MSE) and divergence ($|\nabla \cdot \mathbf{u}|^2$). Details are provided in Appendices~\ref{app:implementation} and \ref{app:dataset}.

Benchmarking utilizes the FNO backbone across regression and generative tasks. We compare  unconstrained baselines (\textbf{Reg. / Gen. Base}) against our proposed methods (\textbf{Reg. / Gen. Ours}), which enforce physical constraints via Leray projection and curl-based priors (Figure~\ref{fig:method_framework}). The use of spectral projection to enforce incompressibility is also shared by prior studies on CNNs and neural operators~\cite{jiang2019enforcing,duruisseaux2024towards}. In our generative experiments, this constraint is imposed on the flow-matching dynamics together with a divergence-free reference distribution. Flow matching paths are sampled using the \texttt{dopri5} solver. For generative evaluation, metrics are reported as the ensemble mean and standard deviation over 20 independent realizations.

\subsection{Results and Analysis}

\paragraph{Quantitative Performance.}
Table~\ref{tab:re1000_resize} summarizes the quantitative results across prediction (training horizon), short-term rollout, and long-term extrapolation.
In the regression setting, \textbf{Reg. (Base)} achieves reasonable initial accuracy but degrades drastically during extrapolation, eventually diverging due to rapid divergence error accumulation.
In contrast, \textbf{Reg. (Ours)} maintains structural stability throughout the rollout by enforcing exact incompressibility up to discretization error, preventing the growth of non-physical modes.
A similar trend is observed in the generative comparison. \textbf{Gen. (Base)} produces plausible short-term samples but fails catastrophically in the long run (``Inf'' indicates numerical overflow).
Our method, \textbf{Gen. (Ours)}, achieves the best performance across all metrics. Notably, it outperforms even the regression models in MSE. This can be attributed to the model's ability to capture the correct statistical distribution of the flow, rather than converging to a smoothed mean prediction that dampens physically relevant fluctuations.

\paragraph{Physical Consistency and Spectral Fidelity.}
The ``Div'' column in Table~\ref{tab:re1000_resize} highlights the critical role of the hard constraint. The baselines exhibit significant divergence violations ($O(10^2)$ to $O(10^5)$), whereas our methods maintain divergence errors near machine precision (e.g., $2 \times 10^{-7}$ for Gen. Ours).
This structural advantage is further evidenced by the pressure reconstruction (Figure~\ref{fig:res_p}). Since pressure is governed by the Poisson equation $-\Delta p = \nabla \cdot (\mathbf{u} \cdot \nabla \mathbf{u})$, its accurate recovery depends on the correct computation of velocity gradients, which is only possible if the field is smooth and divergence-free. 
Furthermore, the streamline visualization (Figure~\ref{fig:res_streamline}) demonstrates that our method generates coherent vortex structures without the non-physical sources or sinks often observed in unconstrained baselines.

Figure~\ref{fig:res_spectrum} presents the enstrophy spectrum analysis. The energy cascade in turbulence follows a specific power law. Unconstrained models tend to deviate from this physical law, either by dissipating energy too quickly (spectral damping) or by accumulating spurious energy at high frequencies (spectral blocking).
Our method closely matches the ground truth spectrum, indicating that the predicted flows not only satisfy kinematic constraints but also preserve the correct multi-scale energy distribution of the turbulent system.

\paragraph{Regression vs. Generative Modeling}
We observe that \textbf{Gen. (Ours)} yields superior long-term stability ($\text{MSE} \approx 0.007$) compared to \textbf{Reg. (Ours)} ($\text{MSE} \approx 210$).
This significant gap suggests that for chaotic dynamical systems, learning the probability flow of the trajectory (Generative) is more robust than minimizing pointwise error (Regression), provided that the generative process is constrained to the physical manifold. By confining the probability path to the divergence-free subspace, our framework maximizes the benefits of generative modeling while ensuring physical validity.

\section{Conclusion}
\label{sec:conclusion}
In this work, we presented a unified framework that enforces incompressibility
as an intrinsic constraint within neural operator architectures.
The framework combines spectral Leray projection with a curl-based Gaussian
reference measure, keeping the flow-matching interpolation and velocity
targets within the divergence-free subspace. Our evaluation on 2D Navier–Stokes equations confirms that embedding physical constraints directly into the operator structure enhances the stability and physical fidelity of the simulations. The results demonstrate that the proposed framework achieves exact incompressibility and accurately captures the statistical characteristics of the flow. 

While focusing on periodic domains, the principles established here provide a foundation for future extensions to complex boundary conditions and 3D fluid systems.

\section*{Impact Statement}
This paper presents work whose goal is to advance the field of Machine
Learning. There are many potential societal consequences of our work, none
which we feel must be specifically highlighted here.

\bibliography{example_paper}
\bibliographystyle{icml2026}

\newpage
\appendix
\onecolumn

\section{Incompressibility and Projection-Based Methods}
\label{app:incompressibility_methods}

In this section, we elaborate on the motivation behind adopting the Spectral Leray Projection. We contrast our approach with classical pressure-correction schemes and soft-penalty methods, specifically addressing the computational bottlenecks of traditional solvers within deep learning graphs and the theoretical inconsistencies they introduce in generative modeling.

\textbf{Classical Pressure-Correction and Its Limitations in Deep Learning.} 
In classical Computational Fluid Dynamics (CFD), particularly in fractional-step methods, incompressibility ($\nabla \cdot \mathbf{u} = 0$) is strictly enforced by solving a pressure Poisson equation:
\begin{equation}
\nabla^2 p = \frac{\nabla \cdot \mathbf{u}^*}{\Delta t},
\end{equation}
where $\mathbf{u}^*$ is a tentative velocity field. While this method effectively handles complex boundaries \cite{chorin1968numerical}, it relies on iterative linear solvers (e.g., Conjugate Gradient) that are computationally intensive. Integrating such iterative solvers directly into a neural network training loop presents fundamental optimization challenges. Training a neural operator requires differentiating through the entire model; embedding an iterative solver necessitates implicit differentiation or unrolling solver steps during backpropagation. This process is not only computationally expensive but also numerically unstable, often leading to vanishing or exploding gradients that hinder convergence.

\textbf{Manifold Consistency in Generative Modeling.}
Beyond computational costs, relying on post-hoc Poisson corrections introduces theoretical inconsistencies in generative frameworks like Flow Matching. The goal of these models is to learn a vector field that generates a continuous probability path from noise to data. If a model operates in an unconstrained space and applies a Poisson correction only at discrete steps, the underlying probability flow becomes ill-defined. The vector field attempting to bridge the Gaussian noise and the incompressible data would continuously drift off the divergence-free subspace $\mathcal{V}$, resulting in a ``broken" generative process where the learned dynamics contradict physical constraints \cite{kerrigan2024functional, lipmanflow}.

\textbf{Methodological Advantages.}
To address these issues, our framework integrates the \textbf{Spectral Leray Projection} directly into the architecture. This approach offers distinct advantages:
\begin{itemize}
    \item \textit{Exact Differentiability and Stability:} Unlike iterative solvers, the Spectral Leray Projection is a closed-form, linear operator based on Fast Fourier Transforms (FFT). It enables exact, stable, and efficient backpropagation, allowing the network to learn the projection mechanism as an intrinsic geometric constraint.
    
    \item \textit{Decoupling from Pressure:} This approach eliminates the need to explicitly model the pressure field $p$, simplifying the learning task by focusing purely on the kinematic velocity field.
    
    \item \textit{Hard Constraints with Machine Precision:} Unlike soft-penalty methods (e.g., PINNs \cite{raissi2019physics}) that treat divergence as an optimization objective, which often results in ``spectral creep'' and high-frequency errors \cite{jiang2023training}, our method enforces $\nabla \cdot \mathbf{u} = 0$ to machine precision by construction.
    
    \item \textit{Generalizability:} While implemented spectrally for efficiency on periodic domains, the principle of differentiable projection is generalizable to complex geometries via other differentiable bases or sparse solvers in future work.
\end{itemize}

\begin{table}[t]
\centering
\caption{\textbf{Summary of notation.}
We distinguish between data samples, probability measures, and vector fields
used in divergence-free flow matching.}
\label{tab:notation}
\begin{tabular}{lll}
\toprule
\textbf{Symbol} & \textbf{Category} & \textbf{Meaning} \\
\midrule
$\mathbf y \in \mathcal V$
& data sample
& divergence-free velocity field from the dataset \\

$\nu$
& data distribution
& probability distribution of data samples on $\mathcal V$ \\

$\mu_0$
& reference distribution
& divergence-free Gaussian noise on $\mathcal V$ \\

$\mathbf u_0 \sim \mu_0$
& noise sample
& reference divergence-free random field \\

$\mathbf u_\tau^{\mathbf y}$
& interpolated sample
& noisy interpolation between $\mathbf u_0$ and data sample $\mathbf y$ \\

$\mu_\tau^{\mathbf y}$
& conditional distribution
& distribution induced by $\mathbf u_\tau^{\mathbf y}$ \\

$\mu_\tau$
& marginal distribution
& mixture of $\mu_\tau^{\mathbf y}$ over $\mathbf y \sim \nu$ \\

$\mathbf v_\tau^{\mathbf y}$
& conditional vector field
& velocity field generating the conditional path \\

$\mathbf v_\tau$
& vector field
& velocity field generating the marginal probability path \\

$v_\theta(\cdot,\tau)$
& neural operator
& learned (unconstrained) velocity field \\

$\mathcal P$
& projection operator
& Leray projection onto the divergence-free subspace $\mathcal V$ \\
\bottomrule
\end{tabular}
\end{table}

\section{Gaussian Measures in Hilbert Spaces}

 The push-forward measure is characterized by:$$(F_{\#}\gamma)(B) = \gamma\left(\{x \in \mathcal{H} : F(x) \in B\}\right)$$for every Borel set $B \in \mathcal{B}(\mathbb{R})$.

The existence and uniqueness of $m$ and $C$ are guaranteed by the Riesz Representation Theorem, which ensures that the bounded linear and bilinear forms appearing on the right-hand sides of the above equations correspond to unique elements and operators in $\mathcal{H}$.

Specifically, the mean $m$ satisfies
\begin{equation}
\langle m, h \rangle = \int_{\mathcal{H}} \langle x, h \rangle \, \gamma(\mathrm{d}x), \quad \forall h \in \mathcal{H},
\end{equation}
and the covariance operator $C$, which is symmetric, non-negative, and trace-class, is defined by
\begin{equation}
\langle Ch, k \rangle = \int_{\mathcal{H}} \langle x-m, h \rangle \langle x-m, k \rangle \, \gamma(\mathrm{d}x), \quad \forall h, k \in \mathcal{H}.
\end{equation}

The regularity of samples from $\gamma$ is governed by the \textbf{Cameron--Martin space} $\mathcal{H}_\gamma$.
Since $C$ is self-adjoint, non-negative and trace-class, it admits a unique square root $C^{1/2}$.
We define
$$
\mathcal{H}_\gamma := \operatorname{Im}(C^{1/2}),
$$
equipped with the inner product
$$\langle h_1, h_2 \rangle_{\mathcal{H}_\gamma}=
\langle C^{-1/2} h_1, C^{-1/2} h_2 \rangle_{\mathcal{H}},
$$
where $C^{-1/2}$ denotes the pseudo-inverse of $C^{1/2}$.

\section{Proof of Lemma \ref{lem:curl-properties}}
\begin{lemma}[Properties of the Curl Push-forward]
The operator $\nabla^\perp$ is a bounded linear map whose image is exactly the
closed subspace of divergence-free, zero-mean vector fields in $L^2$.
\end{lemma}

\begin{proof}
We verify the linearity, boundedness, and the properties of the image sequentially.\paragraph{1. Linearity.}The linearity of $\nabla^\perp$ follows directly from the linearity of the weak differentiation operators. We verify this by checking additivity and homogeneity explicitly.

\textit{Additivity:} Let $\psi_1, \psi_2 \in \mathcal{X}$ be arbitrary stream functions. By the distributive property of partial derivatives, we have:
\begin{equation}
\begin{aligned}
\nabla^\perp (\psi_1 + \psi_2)&= \begin{pmatrix} \partial_2 (\psi_1 + \psi_2) \ -\partial_1 (\psi_1 + \psi_2) \end{pmatrix}= \begin{pmatrix} \partial_2 \psi_1 + \partial_2 \psi_2 \ -\partial_1 \psi_1 - \partial_1 \psi_2 \end{pmatrix} \\
&= \begin{pmatrix} \partial_2 \psi_1 \ -\partial_1 \psi_1 \end{pmatrix} + \begin{pmatrix} \partial_2 \psi_2 \ -\partial_1 \psi_2 \end{pmatrix}= \nabla^\perp \psi_1 + \nabla^\perp \psi_2.
\end{aligned}
\end{equation}

\textit{Homogeneity:} Let $c \in \mathbb{R}$ be an arbitrary scalar and $\psi \in \mathcal{X}$. By the commutativity of differentiation with scalar multiplication:
\begin{equation}
\begin{aligned}
\nabla^\perp (c \psi)&= \begin{pmatrix} \partial_2 (c \psi) \ -\partial_1 (c \psi) \end{pmatrix}= \begin{pmatrix} c \partial_2 \psi \ -c \partial_1 \psi \end{pmatrix} \&= c \begin{pmatrix} \partial_2 \psi \ -\partial_1 \psi \end{pmatrix}= c \nabla^\perp \psi.
\end{aligned}
\end{equation}
Thus, $\nabla^\perp$ is a linear operator.
\paragraph{2. Boundedness and Closed Image.}
For any $\psi \in \mathcal{X} \subset H^1_{\mathrm{per}}$, the $L^2$-norm of the vector field is:
\begin{equation}
\|\nabla^\perp \psi\|_{L^2}^2 = \|\partial_2 \psi\|_{L^2}^2 + \|-\partial_1 \psi\|_{L^2}^2 = \|\nabla \psi\|_{L^2}^2 \leq \|\psi\|_{H^1}^2.
\end{equation}
This implies that $\nabla^\perp$ is a bounded operator with operator norm at most $1$. Furthermore, since $\mathcal{X}$ consists of zero-mean functions, the Poincaré inequality holds: $\|\psi\|_{L^2} \leq C \|\nabla \psi\|_{L^2}$. Consequently, the map is an isomorphism onto its image (specifically, it is bounded from below), which implies that the image is a closed subspace in $L^2_{\mathrm{per}}(\mathbb{T}^2; \mathbb{R}^2)$.
\paragraph{3. Divergence-Free and Zero-Mean Property.}
The resulting vector field $\mathbf{u} = (u_1, u_2)^\top = (\partial_2 \psi, -\partial_1 \psi)^\top$ is divergence-free in the distributional sense. For smooth approximations, Schwarz's theorem gives:
\begin{equation}
\nabla \cdot \mathbf{u} = \partial_1 u_1 + \partial_2 u_2 = \partial_1(\partial_2 \psi) + \partial_2(-\partial_1 \psi) = \partial_{12}\psi - \partial_{21}\psi = 0.
\end{equation}
Additionally, by the periodicity of the domain $\mathbb{T}^2$, the integral of any partial derivative of a periodic function vanishes (e.g., $\int_{\mathbb{T}^2} \partial_2 \psi = 0$). Thus, the vector field has zero mean.
\end{proof}

\section{Proof of Theorem \ref{thm:div-free-fm}}
\label{proof:fm}
\begin{theorem}[Existence of a Divergence-Free Flow Matching Vector Field]
\label{thm:div-free-fm-appendix}
Let $(\mu_\tau)_{\tau \in [0,1]}$ be the marginal path on $\mathcal{V}$ constructed above. Under the absolute continuity condition $\mu_\tau^{\mathbf{y}} \ll \mu_\tau$, the Radon--Nikodym derivative
\[
w_\tau(\mathbf{x}, \mathbf{y}) := \frac{\mathrm{d}\mu_\tau^{\mathbf{y}}}{\mathrm{d}\mu_\tau}(\mathbf{x})
\]
is well-defined for $\mu_\tau$-almost every $\mathbf{x}$. The vector field
\begin{equation}
\label{eq:div-free-fm-appendix}
\mathbf v_\tau(\mathbf x)
=
\int_{\mathcal V}
\mathbf v_\tau^{\mathbf y}(\mathbf x)\,
w_\tau(\mathbf x, \mathbf y)\,
\mathrm d\nu(\mathbf y)
\end{equation}
generates the path $(\mu_\tau)_{\tau \in [0,1]}$, where $\mathbf{v}_\tau^{\mathbf{y}}$ denotes the velocity field associated with the conditional path $\mu_\tau^{\mathbf{y}}$.
\end{theorem}

\begin{proof}

We follow the argument of Theorem~1 in \cite{kerrigan2024functional}.

Let $\varphi \in C_c^\infty(\mathcal{V} \times [0,1])$ be an arbitrary smooth test
function. We aim to show that $(\mu_\tau)_{\tau \in [0,1]}$ and the vector field
$\mathbf{v}_\tau$ defined in \eqref{eq:div-free-fm-appendix} satisfy the continuity equation
in weak form, namely
\begin{equation}
\label{eq:ce-weak}
\int_0^1 \int_{\mathcal{V}} \partial_\tau \varphi(\mathbf{x},\tau)\,
\mathrm{d}\mu_\tau(\mathbf{x})\,\mathrm{d}\tau
=
- \int_0^1 \int_{\mathcal{V}}
\langle \mathbf{v}_\tau(\mathbf{x}), \nabla \varphi(\mathbf{x},\tau) \rangle\,
\mathrm{d}\mu_\tau(\mathbf{x})\,\mathrm{d}\tau .
\end{equation}

By definition of the marginal measure $\mu_\tau = \int \mu_\tau^{\mathbf{y}} \,
\mathrm{d}\nu(\mathbf{y})$ and by Fubini--Tonelli, the left-hand side can be
written as
\begin{equation}
\label{eq:l.h.s}
\int_0^1 \int_{\mathcal{V}} \partial_t \varphi(\mathbf{x},\tau)\,
\mathrm{d}\mu_t(\mathbf{x})\,\mathrm{d}\tau
=
\int_0^1 \int_{\mathcal{V}} \int_{\mathcal{V}}
\partial_t \varphi(\mathbf{x},\tau)\,
\mathrm{d}\mu_t^{\mathbf{y}}(\mathbf{x})\,\mathrm{d}\nu(\mathbf{y})\,\mathrm{d}\tau .
\end{equation}

Since, for $\nu$-almost every $\mathbf{y}$, the conditional path
$(\mu_\tau^{\mathbf{y}})_{\tau \in [0,1]}$ is generated by the velocity field
$\mathbf{v}_\tau^{\mathbf{y}}$ through the continuity equation, we obtain
\begin{equation}
\label{eq:conditional-ce}
=
- \int_0^1 \int_{\mathcal{V}} \int_{\mathcal{V}}
\langle \mathbf{v}_\tau^{\mathbf{y}}(\mathbf{x}),
\nabla \varphi(\mathbf{x},\tau) \rangle\,
\mathrm{d}\mu_\tau^{\mathbf{y}}(\mathbf{x})\,
\mathrm{d}\nu(\mathbf{y})\,\mathrm{d}\tau .
\end{equation}

Under the absolute continuity assumption $\mu_\tau^{\mathbf{y}} \ll \mu_t$, we may
change measure and write
\begin{equation}
\label{eq:change-measure}
=
- \int_0^1 \int_{\mathcal{V}} \int_{\mathcal{V}}
\langle \mathbf{v}_\tau^{\mathbf{y}}(\mathbf{x}),
\nabla \varphi(\mathbf{x},\tau) \rangle\,
w_t(\mathbf{x}, \mathbf{y})\,
\mathrm{d}\mu_\tau(\mathbf{x})\,
\mathrm{d}\nu(\mathbf{y})\,\mathrm{d}\tau ,
\end{equation}
where $w_\tau(\mathbf{x}, \mathbf{y}) = \mathrm{d}\mu_\tau^{\mathbf{y}} /
\mathrm{d}\mu_\tau (\mathbf{x})$.

Since Bochner integrals commute with inner products, another application of
Fubini--Tonelli yields
\begin{equation}
\label{eq:average}
=
- \int_0^1 \int_{\mathcal{V}}
\left\langle
\int_{\mathcal{V}} \mathbf{v}_\tau^{\mathbf{y}}(\mathbf{x})\,
w_\tau(\mathbf{x}, \mathbf{y})\,\mathrm{d}\nu(\mathbf{y}),
\nabla \varphi(\mathbf{x},\tau)
\right\rangle
\mathrm{d}\mu_\tau(\mathbf{x})\,\mathrm{d}\tau .
\end{equation}

By definition of $\mathbf{v}_\tau$ in \eqref{eq:div-free-fm-appendix}, this coincides with the
right-hand side of \eqref{eq:ce-weak}. Therefore, $(\mu_\tau)_{\tau \in [0,1]}$ and
$\mathbf{v}_\tau$ solve the continuity equation in weak form, and $\mathbf{v}_\tau$
generates the marginal probability path $(\mu_\tau)_{\tau \in [0,1]}$.

Finally, since each conditional velocity field $\mathbf{v}_\tau^{\mathbf{y}}$ takes
values in the divergence-free subspace $\mathcal{V}$ and $\mathcal{V}$ is linear,
the averaged field $\mathbf{v}_\tau$ also belongs to $\mathcal{V}$.
\end{proof}

\section{Continuity Equation and Probability Flows}
\label{app:continuity}

Probability paths $(\mu_\tau)_{\tau\in[0,1]}$ on $\mathcal V$ are generated by
time-dependent vector fields $(\mathbf v_\tau)_{\tau\in[0,1]}$ through the
continuity equation
\begin{equation}
\label{eq:continuity}
\partial_\tau \mu_\tau + \mathrm{div}(\mathbf v_\tau \mu_\tau) = 0 .
\end{equation}

Equation~\eqref{eq:continuity} is interpreted in the weak sense: for all test
functions $\varphi : \mathcal V \times [0,1] \to \mathbb R$,
\begin{equation}
\label{eq:weak_continuity}
\int_0^1 \int_{\mathcal V}
\Bigl(
\partial_\tau \varphi(\mathbf x,\tau)
+
\langle \mathbf v_\tau(\mathbf x), \nabla_{\mathbf x} \varphi(\mathbf x,\tau) \rangle
\Bigr)
\, \mathrm d\mu_\tau(\mathbf x)\, \mathrm d\tau
= 0 .
\end{equation}

\section{Implementation Details and Algorithmic Realization}
\label{pseudocode}

\begin{algorithm}[H]
\centering
\caption{Gaussian Divergence-Free Noise (stream function).}
\begin{minipage}{0.99\linewidth}
\begin{lstlisting}[style=torchstyle]
from torchvision.transforms.functional import gaussian_blur

def sample_div_free_noise(B, H, W, T, sigma, device):
    # dims: (B, H, W, T) spatial-temporal dimensions
    # sigma: smoothing standard deviation
    # output: divergence-free velocity field u
    
    # 1. Sample scalar stream function
    psi = torch.randn(B, H, W, T, device=device)
    psi = gaussian_blur(psi, sigma)  # smooth in spatial dimensions

    # 2. Spatial gradients
    dpsi_dy, dpsi_dx = torch.gradient(psi, dim=(1, 2))

    # 3. Curl construction (2D incompressible)
    u = torch.stack([dpsi_dy, -dpsi_dx], dim=-1)

    return u
\end{lstlisting}
\end{minipage}
\label{alg:divfree_noise_code}
\end{algorithm}

\begin{algorithm}[htbp]
\centering
\caption{Spectral Leray Projection (Divergence-Free Velocity).}
\begin{minipage}{0.99\linewidth}
\begin{lstlisting}[style=torchstyle]
def leray_projection(uv_raw):
    # uv_raw: Tensor of shape (B, 2, H, W, T)
    u, v = uv_raw[:, 0], uv_raw[:, 1]

    u_hat = fft2(u)
    v_hat = fft2(v)

    # Wave numbers
    H, W = u.shape[-2], u.shape[-1]
    kx, ky = make_wavenumbers(H, W)

    k2 = kx**2 + ky**2
    proj = (kx * u_hat + ky * v_hat) / k2
    proj = torch.where(k2 == 0, torch.zeros_like(proj), proj)

    u_hat = u_hat - kx * proj
    v_hat = v_hat - ky * proj

    u = ifft2(u_hat)
    v = ifft2(v_hat)
    return torch.stack([u, v], dim=1)
\end{lstlisting}
\end{minipage}
\label{alg:leray_projection_code}
\end{algorithm}

\section{Experimental Implementation Details}
\label{app:implementation}

We evaluate our proposed framework on the task of turbulent flow prediction, formulated as a sequence-to-sequence learning problem following standard benchmarks \cite{li2020fourier}. The model observes a historical velocity context of 15 frames ($\mathbf{u}_{0:14}$) to predict the subsequent 35 frames ($\mathbf{u}_{15:49}$). 

For the \textbf{regression task}, the model learns a direct deterministic mapping from the history to the future sequence. 
For the \textbf{generative modeling task}, the historical context $\mathbf{u}_{0:14}$ serves as the condition $c$; the model generates the target sequence $\mathbf{u}_{15:49}$ by transforming an initial divergence-free noise distribution conditioned on $c$. 

To strictly assess long-term stability and physical conservation beyond the standard training horizon, we additionally perform autoregressive rollouts (extrapolation) up to $T=300$ steps during evaluation.

\paragraph{Model Architectures and Baselines.}
To strictly isolate the impact of our proposed geometric constraints, we maintain identical backbone architectures across all comparisons. We utilize the Fourier Neural Operator (FNO) \cite{li2020fourier} as the core parameterized mapping. The baseline models (\textbf{Reg. Base} and \textbf{Gen. Base}) map directly to the ambient space $L^2_{\text{per}}(\mathbb{T}^2; \mathbb{R}^2)$ without structural constraints. In contrast, our proposed methods (\textbf{Reg. Ours} and \textbf{Gen. Ours}) integrate the differentiable spectral Leray projection as a final layer. This design ensures that all model outputs, covering both deterministic predictions and generative samples, lie exactly on the divergence-free manifold $\nabla \cdot \mathbf{u} = 0$.

\paragraph{Flow Matching Configuration.}
For generative modeling, we employ the Flow Matching framework \cite{lipmanflow} to learn the probability flow of the turbulent trajectory. The baseline uses a standard Gaussian prior in the ambient space, which is incompatible with the incompressible subspace. Our method constructs a specific divergence-free Gaussian process via a stream-function pushforward, ensuring that the entire probability path is subspace-consistent. Sampling is performed using the \texttt{dopri5} (Dormand-Prince) adaptive ODE solver with strict absolute and relative tolerances set to $10^{-5}$. This high-precision integration ensures that valid samples are not corrupted by numerical drift. We use a Neural ODE formulation with 50 function evaluations (NFE) per step to balance computational cost and sample fidelity.

\section{Data Generation and Numerical Methods}
\label{app:dataset}

We detail the rigorous procedure used to generate the high-fidelity turbulent flow dataset. The data is produced by solving the 2D Navier--Stokes equations in the vorticity-streamfunction formulation using a pseudo-spectral method.

\subsection{Governing Equations and Solver}
The underlying dynamics are governed by the vorticity transport equation with a viscous damping term and an external forcing component:
\begin{equation}
    \partial_t \omega(\mathbf{x}, t) + (\mathbf{u} \cdot \nabla) \omega(\mathbf{x}, t) = \nu \Delta \omega(\mathbf{x}, t) + f(\mathbf{x}),
\end{equation}
where $\omega = \nabla \times \mathbf{u}$ is the vorticity, $\nu = 10^{-3}$ is the kinematic viscosity (Reynolds number $\text{Re} = 1000$), and $f$ is a deterministic forcing term.
The system is evolved in the spectral domain. We apply the Crank--Nicolson scheme for the linear viscous term and an explicit forward step for the non-linear advection and forcing terms. The discrete update rule in Fourier space ($\hat{\omega}$) with time step $\Delta t = 10^{-3}$ is given by:
\begin{equation}
    (1 + 0.5 \nu k^2 \Delta t) \hat{\omega}_{t+1} = (1 - 0.5 \nu k^2 \Delta t) \hat{\omega}_t + \Delta t \left( \hat{f} - \widehat{\mathbf{u} \cdot \nabla \omega} \right)_t,
\end{equation}
where $k^2 = 4\pi^2(k_x^2 + k_y^2)$. Dealiasing is performed using the $2/3$ rule to suppress spectral blocking artifacts.

\subsection{Velocity Extraction and Incompressibility}
Since our machine learning models operate directly on velocity fields, we strictly recover $\mathbf{u}$ from the evolved vorticity $\omega$ at each storage step. This process ensures the data is divergence-free by construction.
First, the stream function $\psi$ is solved via the Poisson equation in the spectral domain:
\begin{equation}
    -\Delta \psi = \omega \quad \Longrightarrow \quad \hat{\psi}(\mathbf{k}) = \frac{\hat{\omega}(\mathbf{k})}{4\pi^2 |\mathbf{k}|^2}, \quad \text{for } \mathbf{k} \neq \mathbf{0}.
\end{equation}
The velocity components $\mathbf{u} = (u, v)$ are then derived as:
\begin{equation}
    u = \frac{\partial \psi}{\partial y}, \quad v = - \frac{\partial \psi}{\partial x}.
\end{equation}
Numerically, this corresponds to $\hat{u} = i 2\pi k_y \hat{\psi}$ and $\hat{v} = - i 2\pi k_x \hat{\psi}$. This exact spectral relationship guarantees that $\nabla \cdot \mathbf{u} = \partial_x u + \partial_y v = 0$ is satisfied to machine precision in the generated dataset.

\subsection{Specific Dataset Configurations}
The dataset is generated on a spatial grid of resolution $64 \times 64$ over a unit domain $\Omega = [0, 1]^2$.
\begin{itemize}
    \item \textbf{Initialization}: Initial vorticity fields $\omega_0$ are sampled from a Gaussian Random Field (GRF) $\omega_0 \sim \mathcal{N}(0, C)$. The covariance operator $C$ is defined diagonally in the Fourier domain with a power-law decay, ensuring smoothness of the initial condition. Specifically, the Fourier coefficients $\hat{\omega}_0(\mathbf{k})$ are independent Gaussian variables with variance:
    \begin{equation}
        \mathbb{E}[|\hat{\omega}_0(\mathbf{k})|^2] \propto \left( 4\pi^2 |\mathbf{k}|^2 + \tau^2 \right)^{-\alpha},
    \end{equation}
    where we set $\alpha = 2.5$ and $\tau = 7.0$. This parameter choice balances the energy spectrum to produce physically realistic turbulent structures at $t=0$.
    \item \textbf{Forcing}: We effectively use a Kolmogorov-style forcing defined as $f(x,y) = 0.1\sqrt{2} \sin(2\pi(x+y) + \phi_i)$, where the phase $\phi_i$ varies across trajectories to induce diverse flow patterns.
    \item \textbf{Temporal Resolution}: The solver runs with an internal step of $\delta t = 10^{-3}$. We record snapshots every $\Delta T = 1.0$ physical time units (every 1000 solver steps).
    \item \textbf{Trajectory Length}: Each training trajectory contains 50 snapshots ($T_{total} = 50$). For testing, we generate extended trajectories of 300 snapshots ($T_{test} = 300$) to evaluate long-term stability.
    \item \textbf{Scale}: The final dataset comprises 10,000 trajectories for training and 100 for testing/validation.
\end{itemize}

\section{Generative Model Variance Analysis}
\label{app:variance}

This section provides a detailed variance analysis of the generative models to explicitly quantify the stability and consistency of the generated ensembles. 
Unlike deterministic regression approaches such as the standard Fourier Neural Operator \cite{li2020fourier}, generative models like Flow Matching \cite{lipmanflow, liu2023flow} produce a distribution of trajectories for a given initial condition.
Capturing the variability of this distribution is essential for assessing whether the model learns the true physical statistics of the turbulent flow or merely memorizes a mean path.

To evaluate this, we generated $R=20$ independent realizations for each test window in the dataset.
We measured performance using the Mean Squared Error (MSE) for velocity components ($u, v$) and the divergence error (Div), aggregating these metrics over three distinct temporal stages:
(1) \textbf{Prediction} ($T=15\dots 50$), corresponding to the training horizon;
(2) \textbf{Short-term Extrapolation} ($T=51\dots 150$); and
(3) \textbf{Long-term Extrapolation} ($T=151\dots 300$).
The statistics reported in Table~\ref{tab:gen_variance_detail} represent the global mean and standard deviation computed over the pooled population of all realizations across all test samples.
For the baseline Conditional Flow Matching model, we note that the reported statistics include trajectories that numerically diverged, effectively highlighting the catastrophic failure modes inherent in unconstrained generative modeling \cite{kerrigan2024functional}.

\paragraph{Instability of Unconstrained Baselines.}
The results in Table~\ref{tab:gen_variance_detail} reveal a critical instability in the standard Flow Matching baseline.
While the baseline performs reasonably well within the training distribution (Prediction stage), it exhibits catastrophic divergence during the Short-term Extrapolation phase ($T=51\dots 150$), with errors exploding to astronomical values ($\text{MSE} \sim 10^{28}$).
This demonstrates that without explicit physical constraints, the generated trajectories rapidly drift off the valid data manifold, leading to numerical blow-up.
In strong contrast, our proposed Divergence-Free Flow Matching maintains numerical stability and physical consistency throughout the entire rollout.
Even in the extended extrapolation regime ($T=151\dots 300$), our method maintains a low MSE ($\approx 0.007$) and machine-precision divergence ($O(10^{-7}$)), confirming that enforcing the divergence-free constraint via the Leray projection is essential for robust long-term generation of turbulent flows.

\begin{table*}[t]
    \centering
    \renewcommand{\arraystretch}{1.3} 
    \setlength{\tabcolsep}{6pt}  
    
    \caption{\textbf{Quantitative Comparison on Re = 1000 (Generative Models).} 
    We report the \textbf{Mean $\pm$ Std} computed across all windows and realizations ($20\times$).
    \textbf{Bold} indicates best performance.
    Note the astronomical error values for the Base model in Short-term Extrapolation, indicating model blow-up.
    }
    \label{tab:gen_variance_detail}
    
    \resizebox{\textwidth}{!}{%
    \begin{tabular}{l l c c c c}
        \toprule
        \textbf{Model} & \textbf{Stage} & \textbf{Total MSE} & \textbf{U-MSE} & \textbf{V-MSE} & \textbf{Div-MSE} \\
        \midrule
        
        \multirow{3}{*}{\textbf{Gen. (Base)}} 
         & Prediction ($T_{15\text{-}50}$) & $0.0046 \pm 0.0035$ & $0.0044 \pm 0.0033$ & $0.0047 \pm 0.0038$ & $0.0075 \pm 0.0120$ \\
         & Short-term Ext. ($T_{51\text{-}150}$) & $2.40\text{e}28 \pm \text{NaN}$ & $9.22\text{e}30 \pm \text{NaN}$ & $2.60\text{e}28 \pm \text{NaN}$ & $2.64\text{e}26 \pm \text{NaN}$ \\
         & Long-term Ext. ($T_{151\text{-}\text{End}}$) & $0.0076 \pm 0.0009$ & $0.0079 \pm 0.0009$ & $0.0072 \pm 0.0012$ & $0.0024 \pm 0.0046$ \\
         
        \midrule
        
        \multirow{3}{*}{\textbf{Gen. (Ours)}} 
         & Prediction ($T_{15\text{-}50}$) & $\mathbf{0.0011 \pm 0.0010}$ & $\mathbf{0.0011 \pm 0.0009}$ & $\mathbf{0.0011 \pm 0.0011}$ & $\mathbf{9.3\text{e-}7 \pm 1.2\text{e-}7}$ \\
         & Short-term Ext. ($T_{51\text{-}150}$) & $\mathbf{0.0073 \pm 0.0024}$ & $\mathbf{0.0071 \pm 0.0025}$ & $\mathbf{0.0074 \pm 0.0023}$ & $\mathbf{7.5\text{e-}7 \pm 4.0\text{e-}8}$ \\
         & Long-term Ext. ($T_{151\text{-}\text{End}}$) & $\mathbf{0.0074 \pm 0.0013}$ & $\mathbf{0.0073 \pm 0.0013}$ & $\mathbf{0.0076 \pm 0.0012}$ & $\mathbf{7.4\text{e-}7 \pm 2.7\text{e-}8}$ \\
        
        \bottomrule
    \end{tabular}
    }
\end{table*}

\section{Extended Visual Analysis}
\label{app:extended_vis}

In this section, we present a comprehensive fine-grained visual and spectral assessment of the performance of the models. We inspect primitive velocity variables, derived physical quantities including vorticity and pressure, and spectral error distributions across three distinct temporal regimes. The evaluation strictly compares our proposed divergence-free framework against the unconstrained baseline to highlight the benefits of enforcing hard physical constraints.

All models are trained to observe a historical context of 15 frames ($\mathbf{u}_{0:14}$) to predict the future evolution. For the extrapolation regimes where $T > 49$, we employ a recursive rolling window strategy. In this setup, the predictions from the model are fed back as input for the next step, extending the rollout up to $T=300$ to rigorously assess long-term stability and conservation properties.

\subsection{Regime I: Standard Prediction Horizon ($T_{15\text{-}49}$)}
\label{app:vis_pred}

This regime serves to validate the ability of the model to reconstruct fine-scale turbulent features within the training distribution. Figure~\ref{fig:pred_fields} presents a multi-panel visualization of the flow state which juxtaposes kinematic fidelity with physical consistency checks.

\textbf{Kinematic Reconstruction.} 
The top row, showing velocity components $u$ and $v$, demonstrates that our constrained model faithfully captures the instantaneous velocity distribution. The contour lines are sharp and the extrema magnitudes match the Ground Truth. This indicates that the model successfully learns the fine-grained advection dynamics without suffering from the over-smoothing often seen in regression-based approaches.

\textbf{Conservation and Pressure Recovery.} 
The middle row highlights the critical structural advantage of our framework. The unconstrained baseline exhibits significant ``checkerboard'' artifacts in the divergence field. This is a characteristic symptom of spectral leakage where high-frequency errors alias into the grid scale. This violation has a catastrophic downstream effect on the pressure field $p$. Since physical pressure is governed by the Poisson equation $-\Delta p = \nabla \cdot (\mathbf{u} \cdot \nabla \mathbf{u})$, the nonzero divergence in the baseline acts as a corrupted source term. Consequently, the recovered pressure degenerates into non-physical white noise. In contrast, our method enforces $\nabla \cdot \mathbf{u} \approx 0$ to machine precision by construction. This allows for the recovery of smooth and structurally coherent pressure fields that accurately reflect the internal stress of the fluid.

\textbf{Topological Realism and Spectral Fidelity.} 
The bottom row, displaying streamlines and vorticity $\omega$, confirms the topological realism of our approach. Our method recovers filamentary vorticity structures and coherent vortex loops effectively. This prevents the spectral blocking observed in simpler baselines where energy piles up at the cutoff frequency. This observation is quantitatively corroborated by the Enstrophy spectrum in Figure~\ref{fig:pred_mse} (Top). Our method matches the inertial range energy cascade where $E(k) \propto k^{-3}$ up to the Nyquist frequency, whereas the baseline suffers from spectral damping.

\subsection{Regime II: Short-term Extrapolation ($T_{50\text{-}100}$)}
\label{app:vis_short}

This regime marks the onset of out-of-distribution recursive prediction where prediction errors typically accumulate.

\textbf{Drift and Spectral Bias.} 
Figure~\ref{fig:short_fields} illustrates the flow state at $T=100$. The baseline model begins to drift off the physical manifold. The vorticity field $\omega$ loses its sharp gradients and becomes amorphous and overly smooth. This represents a common failure mode in MSE-trained neural operators known as spectral bias. In this mode, the model preferentially fits low-frequency components while discarding the high-frequency turbulent details that are essential for accurate evolution.

\textbf{Mechanism of Divergence Explosion.} 
Crucially, the middle row reveals the mechanism of failure for the baseline. Divergence errors do not simply persist but amplify through the recurrent process. In the absence of a correction step, small non-zero divergence acts as a spurious source term in the transport equations. This leads to a numerical ``blow-up'' where the pressure field collapses into high-magnitude noise as shown in Fig.~\ref{fig:short_fields}d. This indicates that the trajectory has left the valid solution manifold.

\textbf{Stabilization via Projection.} 
In contrast, our framework maintains tight structural coherence. By strictly enforcing the solenoidal constraint at every timestep, the spectral Leray projection acts as an inductive bias that stabilizes the recurrent dynamics. The generated flows remain on the attractive manifold of the Navier-Stokes equations. This prevents the ``spectral creep'' observed in unconstrained models and ensures the simulation remains physically plausible even outside the training horizon.

\subsection{Regime III: Long-term Extrapolation ($T_{101\text{-}300}$)}
\label{app:vis_long}

This final regime evaluates the ergodicity and asymptotic stability of the generative process. We assess whether the model can remain on the physical attractor indefinitely.

\textbf{Catastrophic Decoherence of Baseline.} 
As shown in Figure~\ref{fig:long_fields}, the baseline model completely loses physical integrity by $T=300$. The velocity field degrades into smooth and wave-like patterns that bear no resemblance to chaotic turbulence. The divergence errors saturate at a high level which renders both the pressure in the middle row and vorticity in the bottom row physically meaningless. The energy spectrum in Figure~\ref{fig:long_mse} (Top) confirms this with a severe drop-off in high-frequency energy. This indicates the model has ceased to simulate turbulence effectively.

\textbf{Invariant Measure Sampling.} 
Conversely, our proposed generative method continues to generate sharp and rich turbulent features indefinitely. Crucially, the statistics remain stationary. The Enstrophy spectrum in Figure~\ref{fig:long_mse} (Top) overlaps almost perfectly with the ground truth. This demonstrates that our model is effectively sampling from the invariant measure of the system. It implies that the framework does not merely memorize trajectories but has effectively learned the underlying conservation laws. This ensures that the probability mass stays within the valid configuration space and prevents leakages that would otherwise destroy the long-term statistics.

\begin{figure*}[ht]
    \centering
    \begin{subfigure}[b]{0.48\textwidth}
        \includegraphics[width=\linewidth]{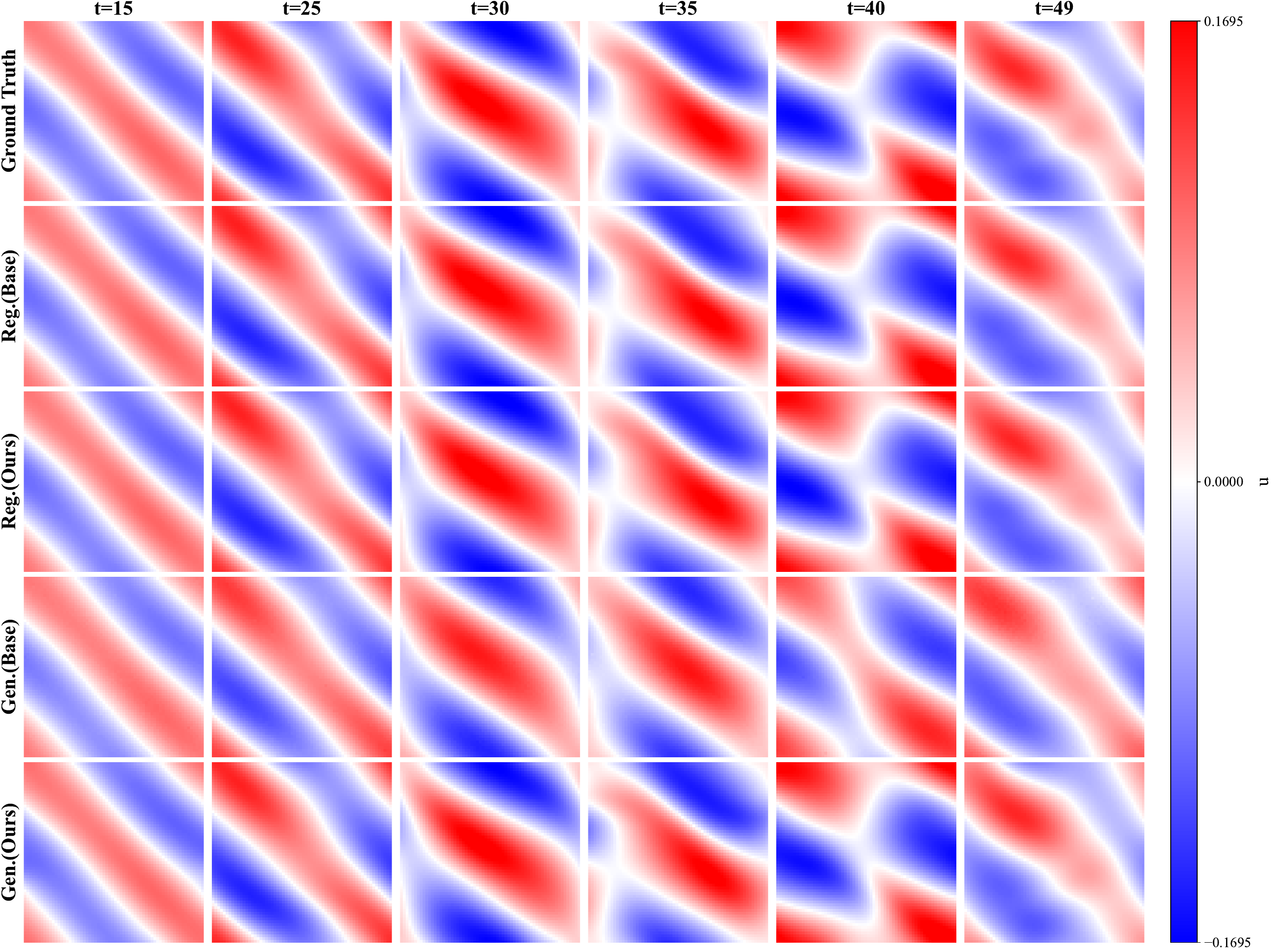} \caption{Zonal Velocity $u$}
    \end{subfigure}
    \hfill
    \begin{subfigure}[b]{0.48\textwidth}
        \includegraphics[width=\linewidth]{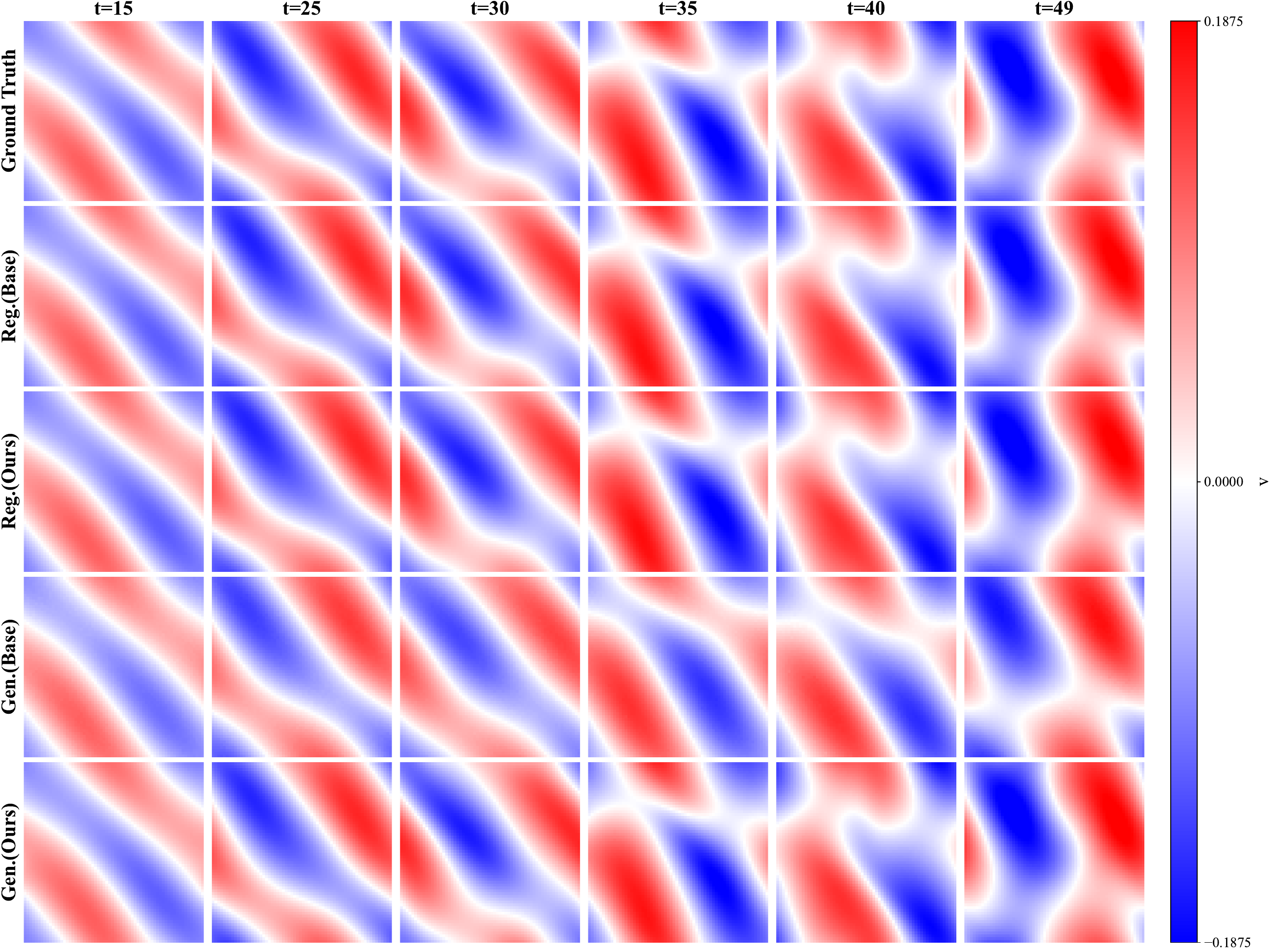} \caption{Meridional Velocity $v$}
    \end{subfigure}

    \begin{subfigure}[b]{0.48\textwidth}
        \includegraphics[width=\linewidth]{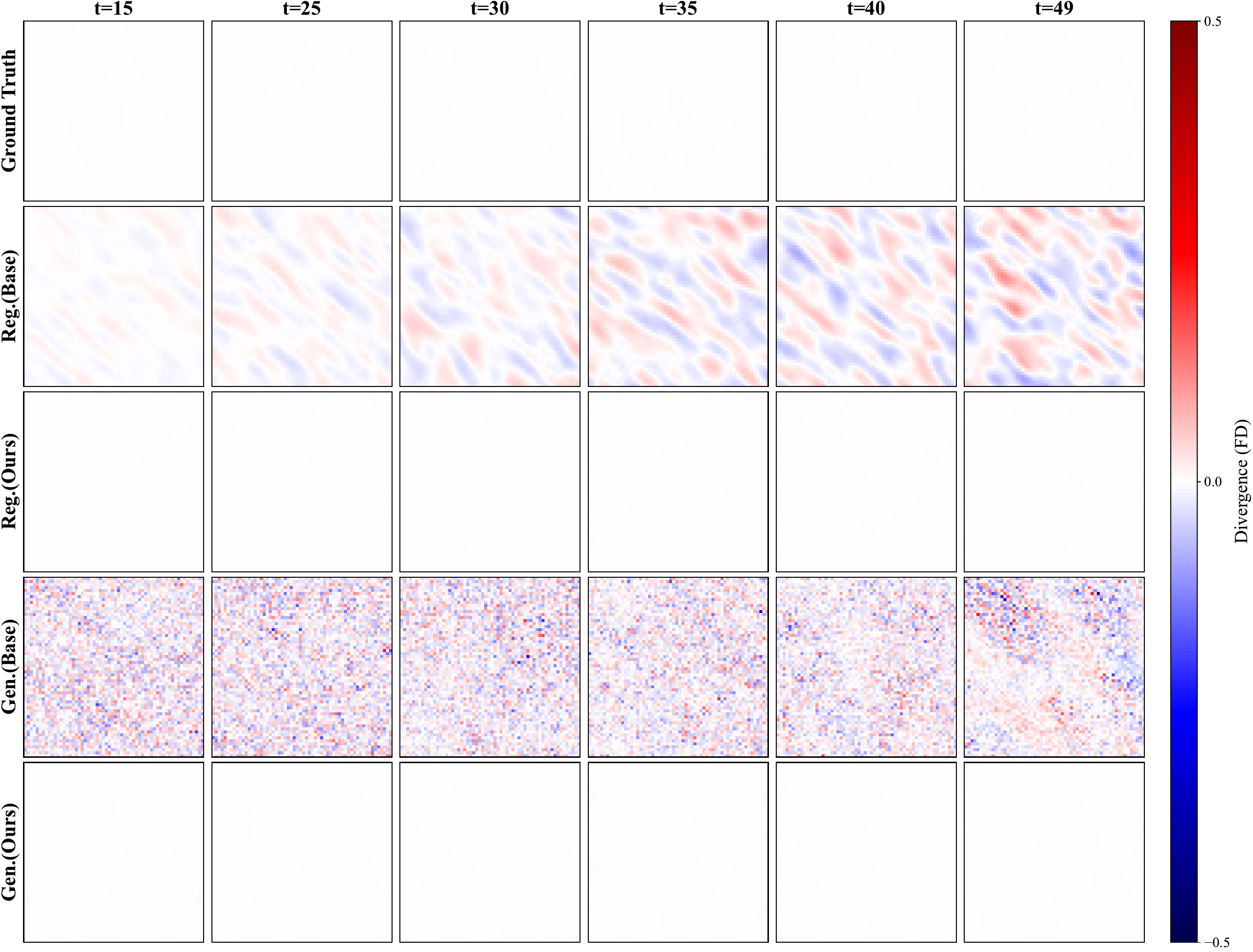} \caption{Divergence Error}
    \end{subfigure}
    \hfill
    \begin{subfigure}[b]{0.48\textwidth}
        \includegraphics[width=\linewidth]{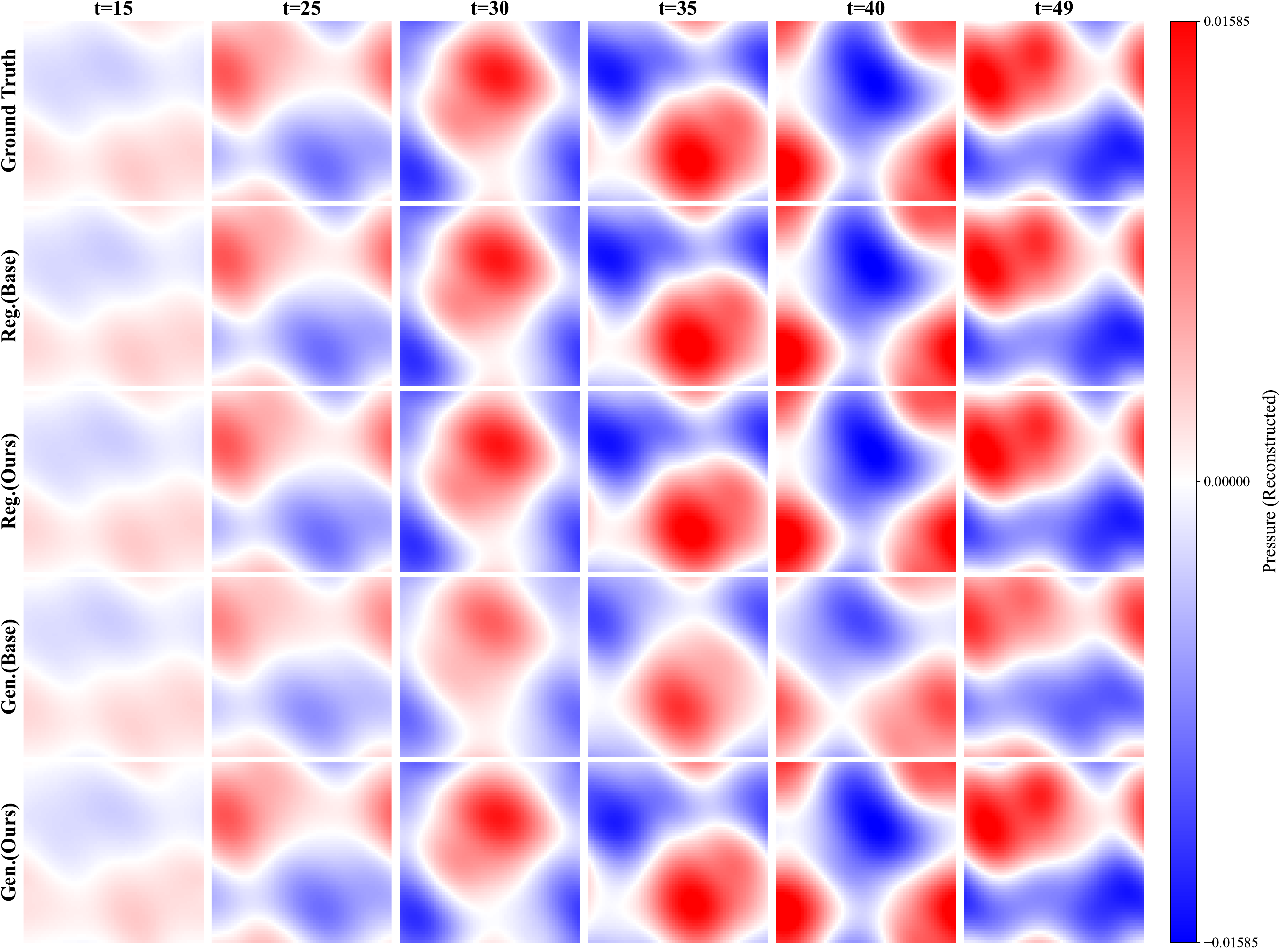} \caption{Pressure}
    \end{subfigure}

    \begin{subfigure}[b]{0.48\textwidth}
        \includegraphics[width=\linewidth]{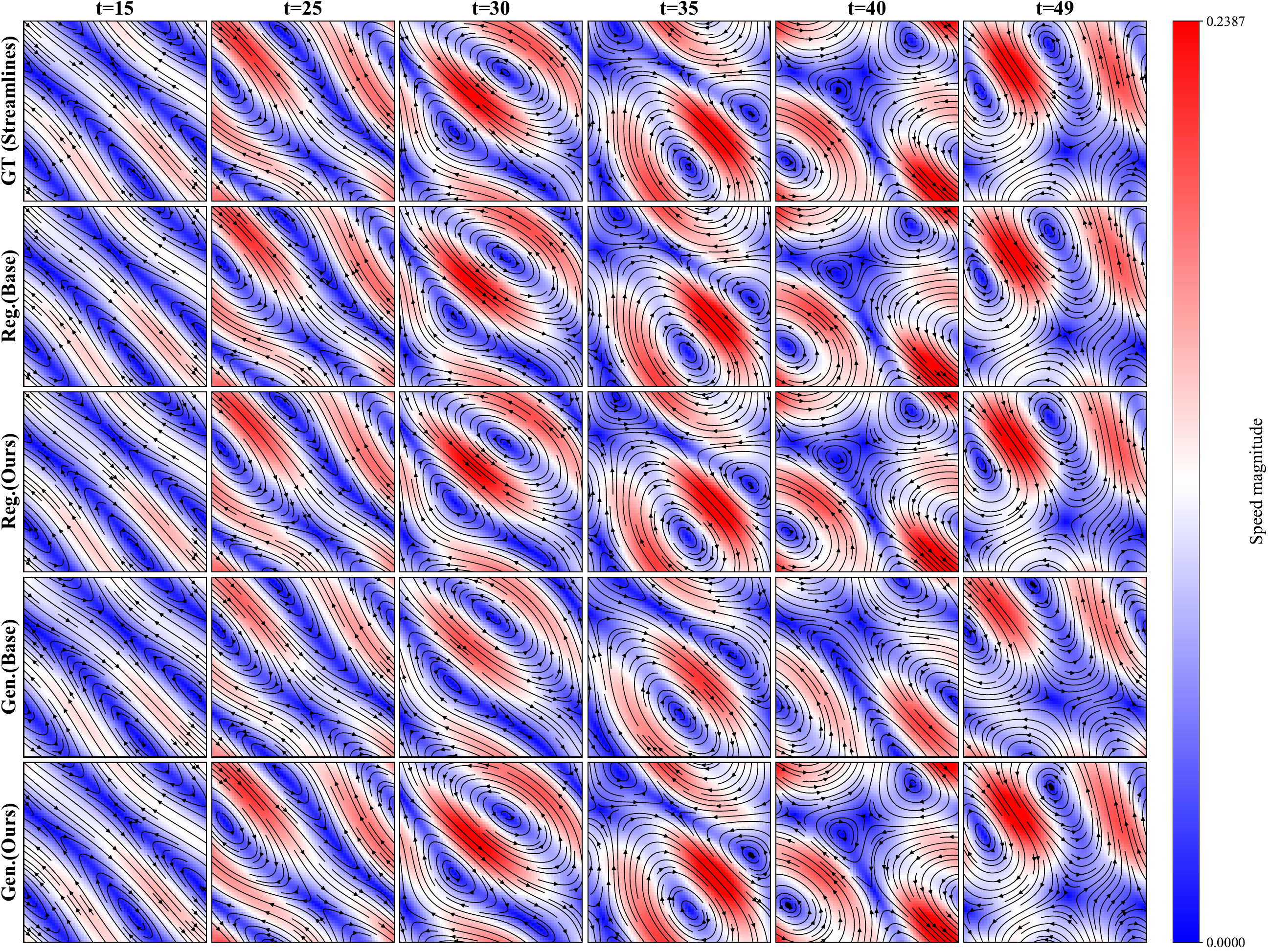} \caption{Streamlines}
    \end{subfigure}
    \hfill
    \begin{subfigure}[b]{0.48\textwidth}
        \includegraphics[width=\linewidth]{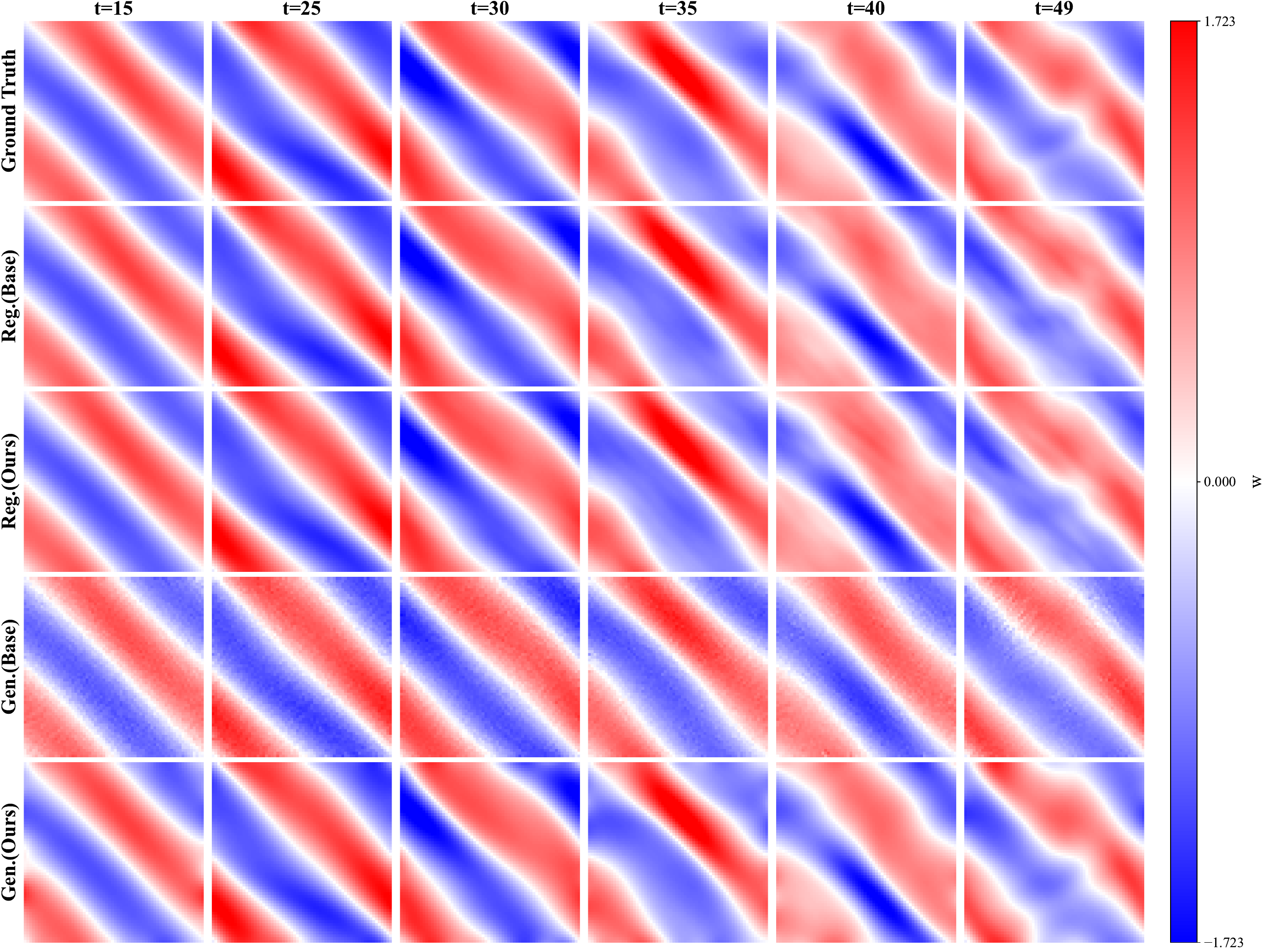} \caption{Vorticity $\omega$}
    \end{subfigure}
    
    \caption{\textbf{Physical Field Reconstruction (Prediction).} A unified view of the kinematic and topological state.
    Top: Velocity fields.
    Middle: Conservation metrics showing baseline failure.
    Bottom: Topology.}
    \label{fig:pred_fields}
\end{figure*}

\begin{figure*}[ht]
    \centering
    \begin{subfigure}[b]{0.92\textwidth}
        \centering
        \includegraphics[width=\linewidth]{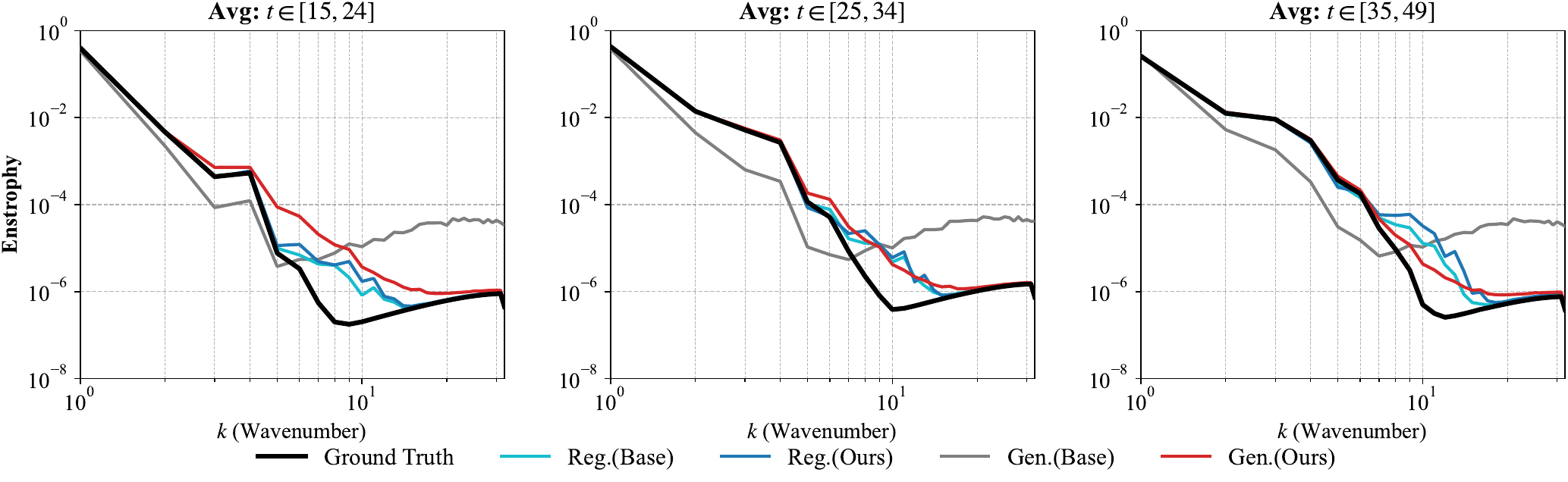}
        \caption{Enstrophy Spectrum}
    \end{subfigure}

    \begin{subfigure}[b]{0.46\textwidth}
        \includegraphics[width=\linewidth]{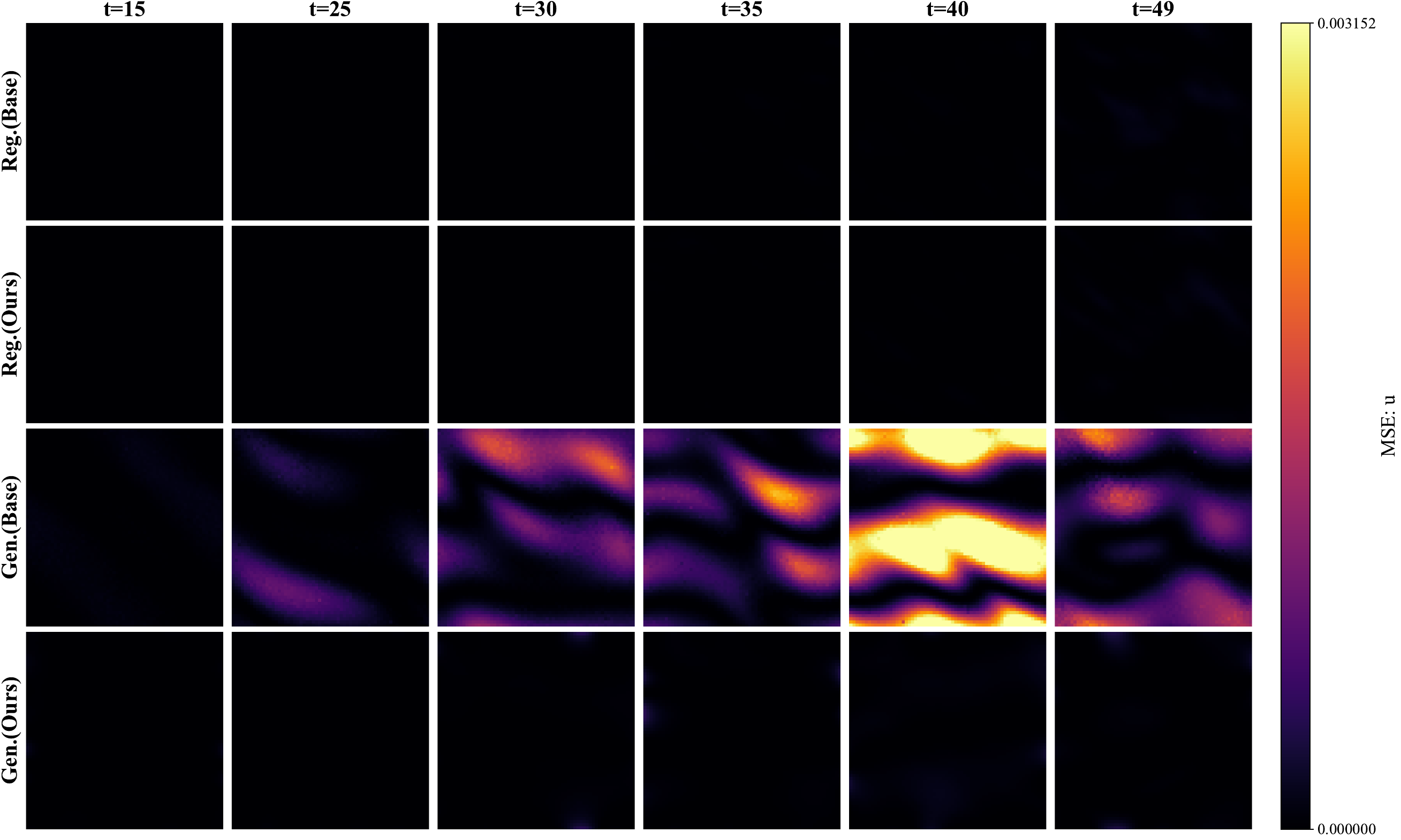} \caption{MSE $u$}
    \end{subfigure}
    \hfill
    \begin{subfigure}[b]{0.46\textwidth}
        \includegraphics[width=\linewidth]{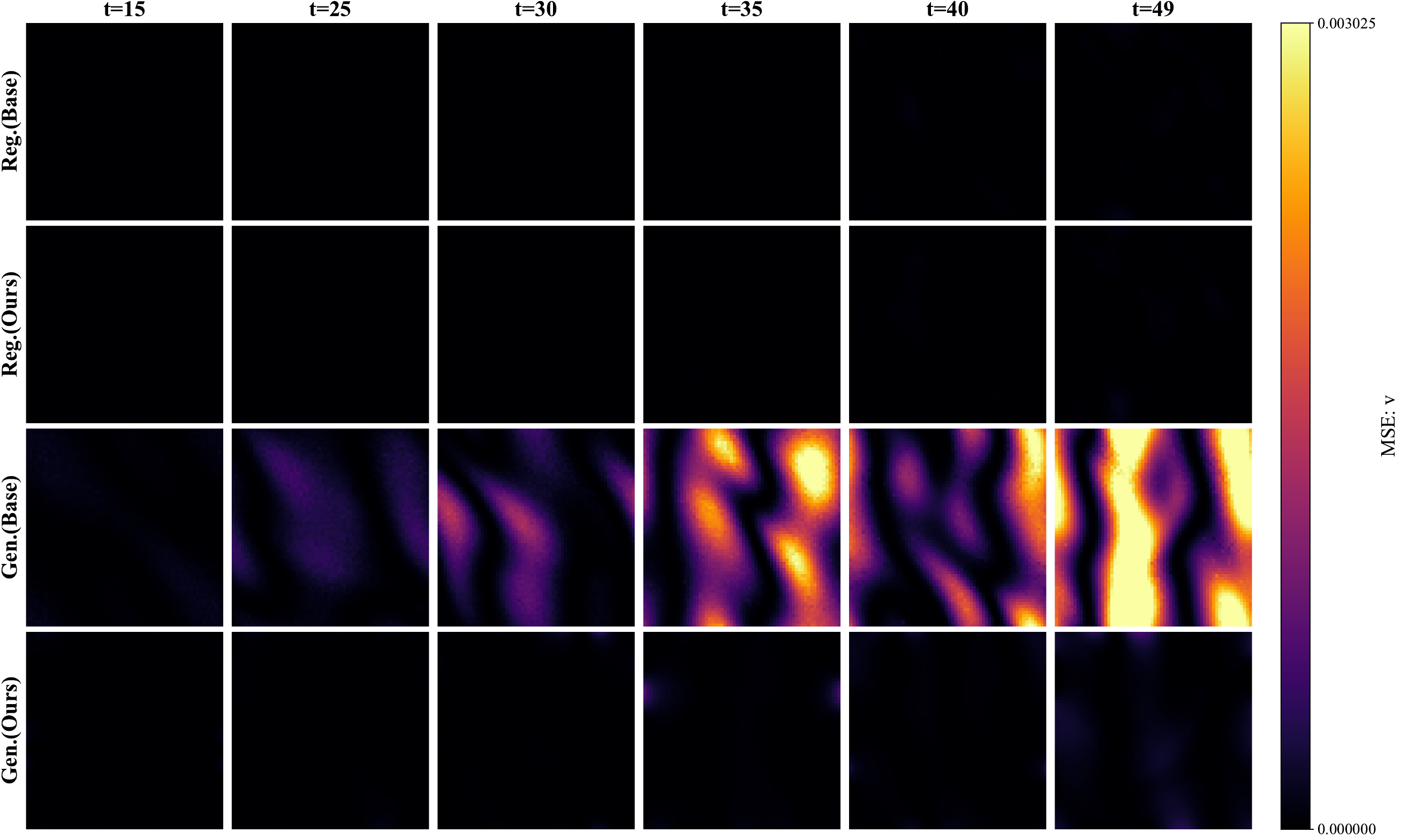} \caption{MSE $v$}
    \end{subfigure}

    \begin{subfigure}[b]{0.46\textwidth}
        \includegraphics[width=\linewidth]{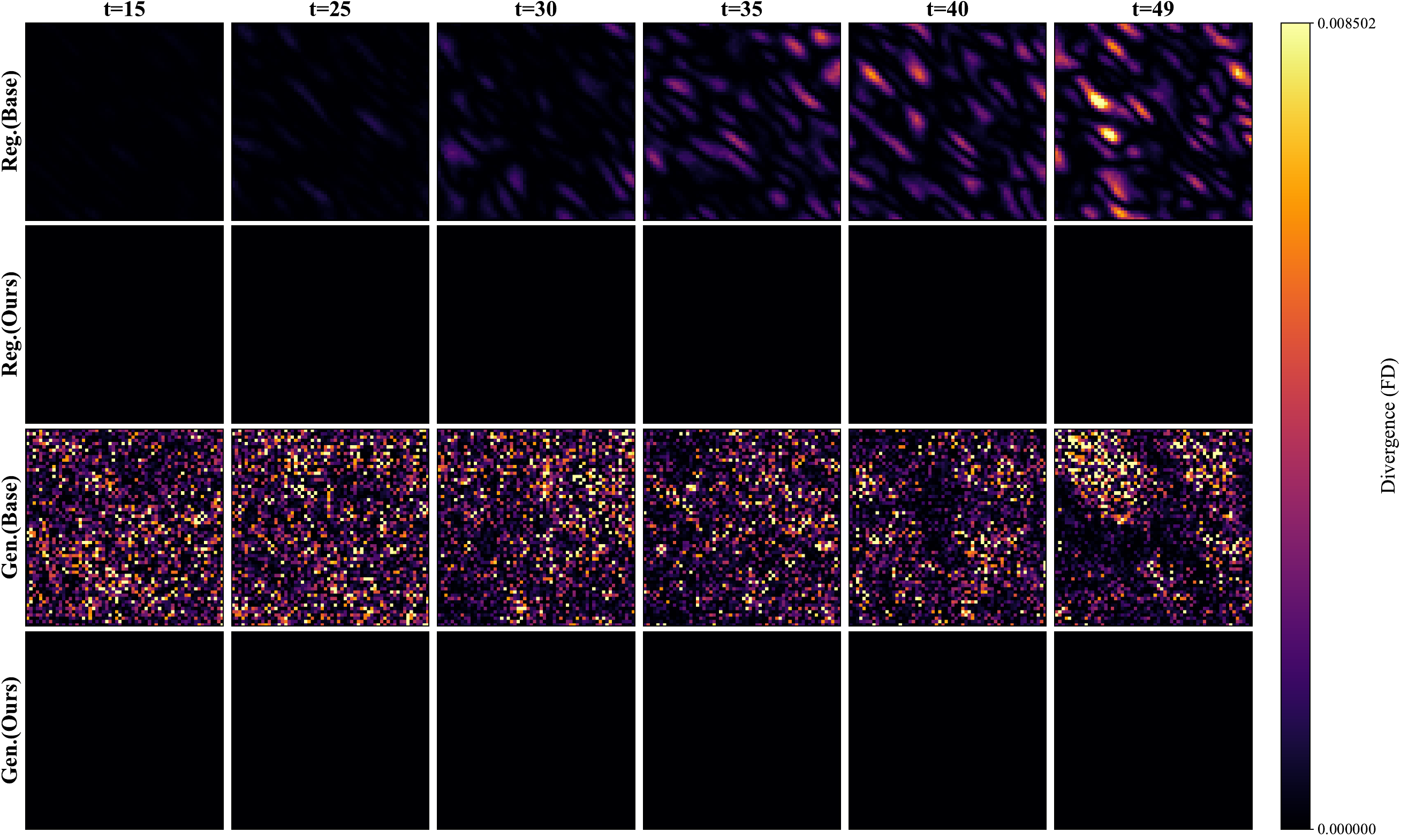} \caption{MSE Divergence}
    \end{subfigure}
    \hfill
    \begin{subfigure}[b]{0.46\textwidth}
        \includegraphics[width=\linewidth]{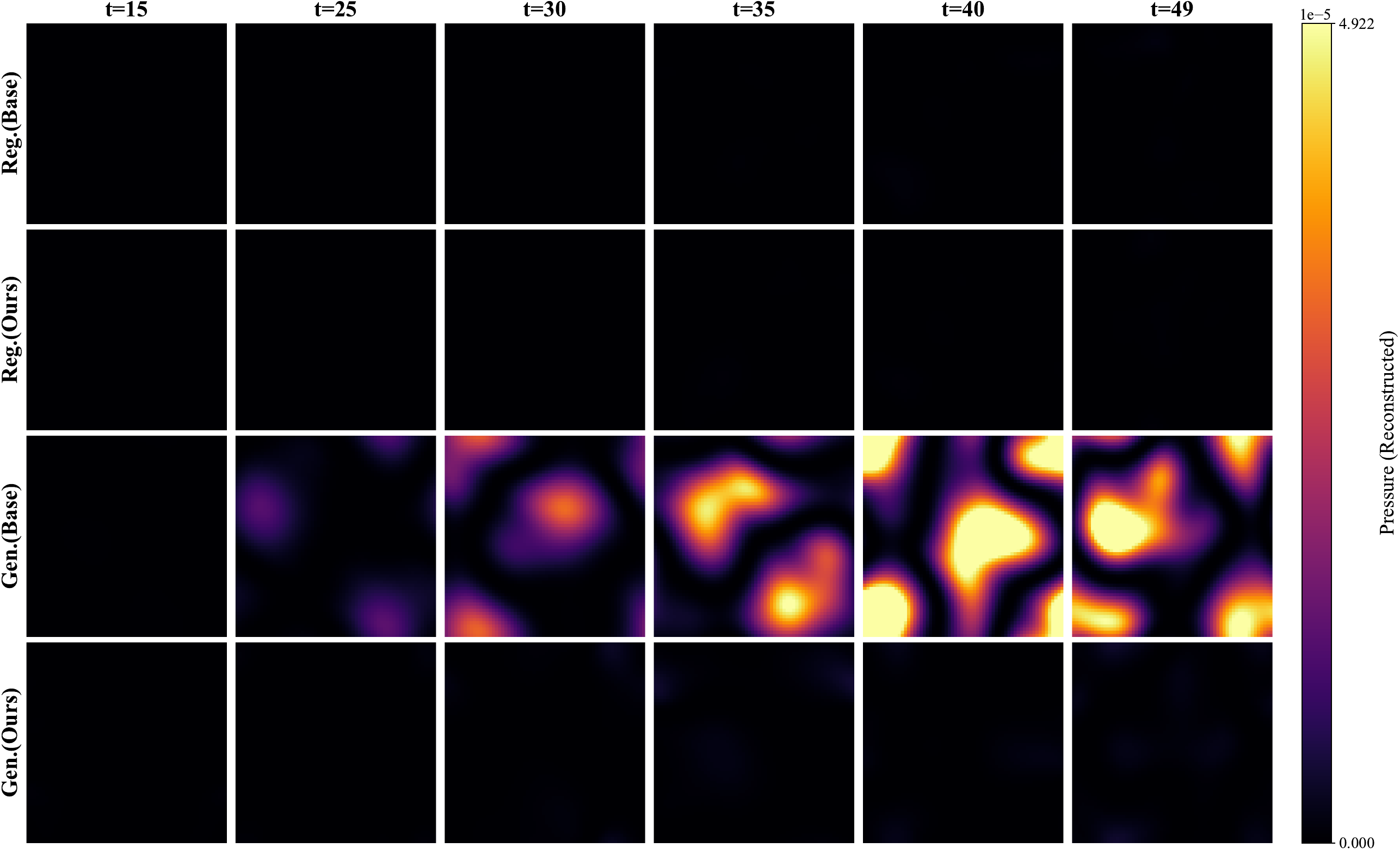} \caption{MSE Pressure}
    \end{subfigure}

    \begin{subfigure}[b]{0.46\textwidth}
        \includegraphics[width=\linewidth]{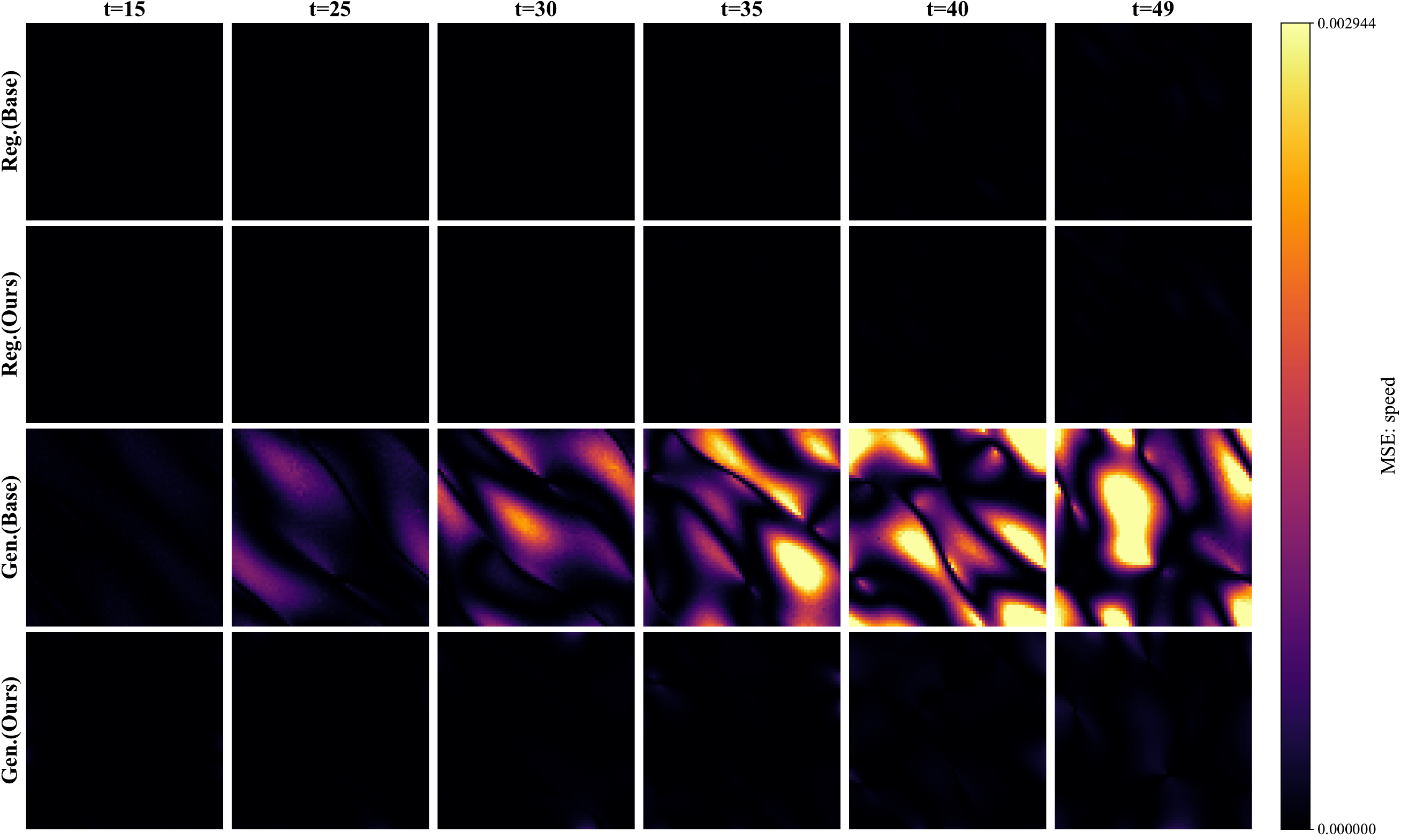} \caption{MSE Speed}
    \end{subfigure}
    \hfill
    \begin{subfigure}[b]{0.46\textwidth}
        \includegraphics[width=\linewidth]{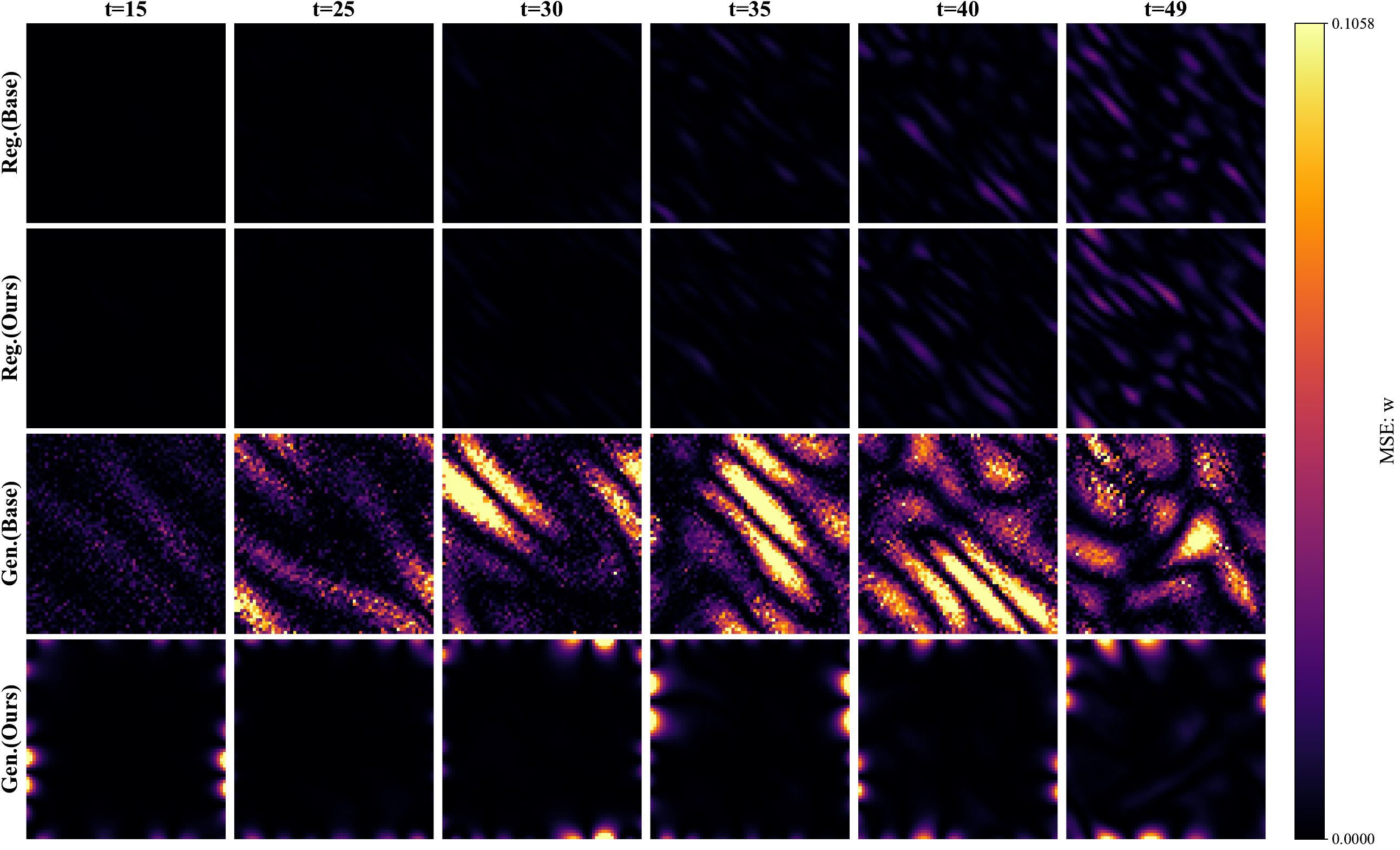} \caption{MSE $\omega$}
    \end{subfigure}
    
    \caption{\textbf{Spectral Fidelity and Error Distributions (Prediction).} 
    Top: Enstrophy Spectrum showing partial alignment with ground truth.
    Rows 2-4: Squared error distributions for primitive and derived variables.}
    \label{fig:pred_mse}
\end{figure*}

\begin{figure*}[ht]
    \centering
    \begin{subfigure}[b]{0.48\textwidth}
        \includegraphics[width=\linewidth]{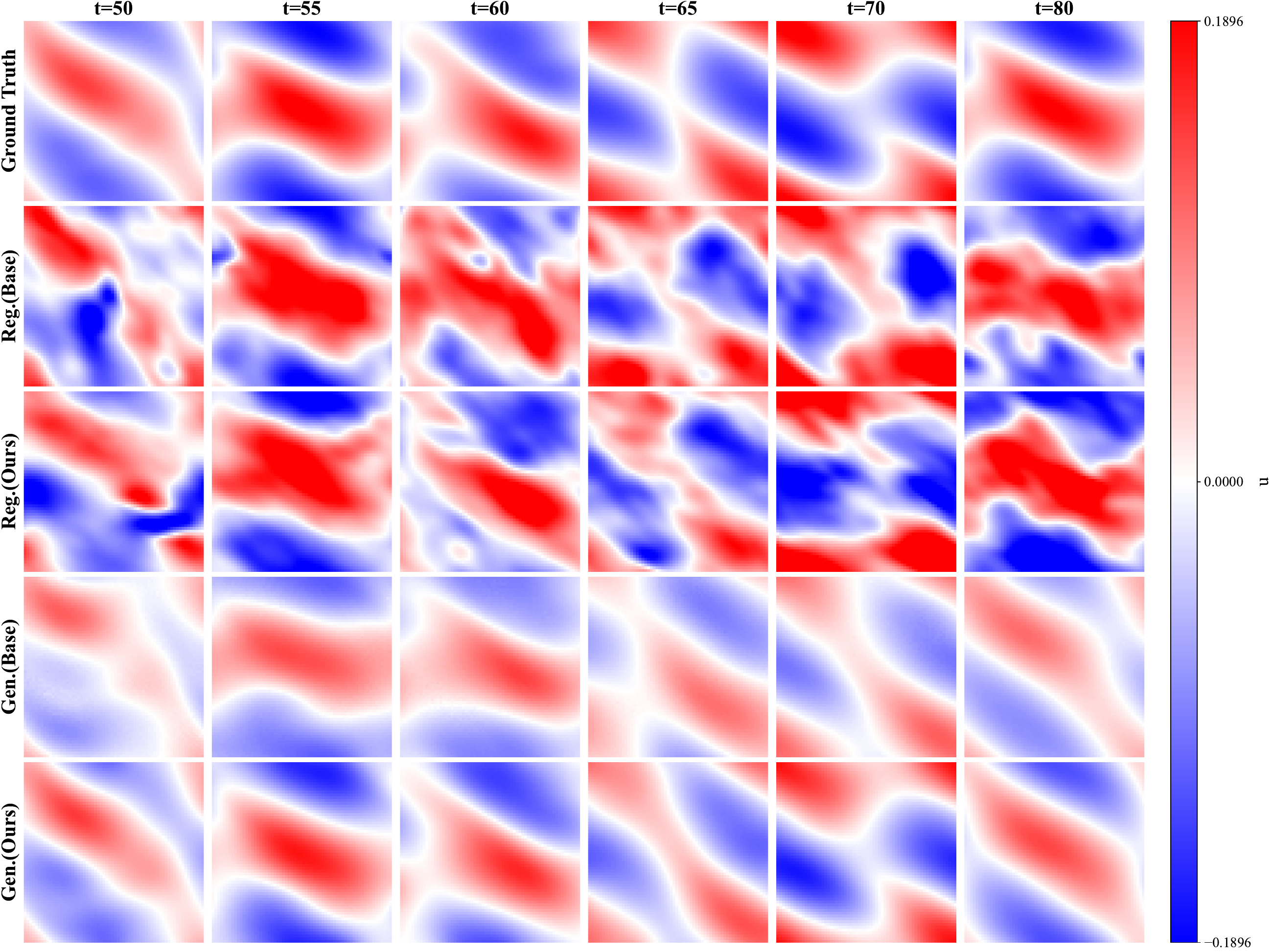} \caption{Zonal Velocity $u$}
    \end{subfigure}
    \hfill
    \begin{subfigure}[b]{0.48\textwidth}
        \includegraphics[width=\linewidth]{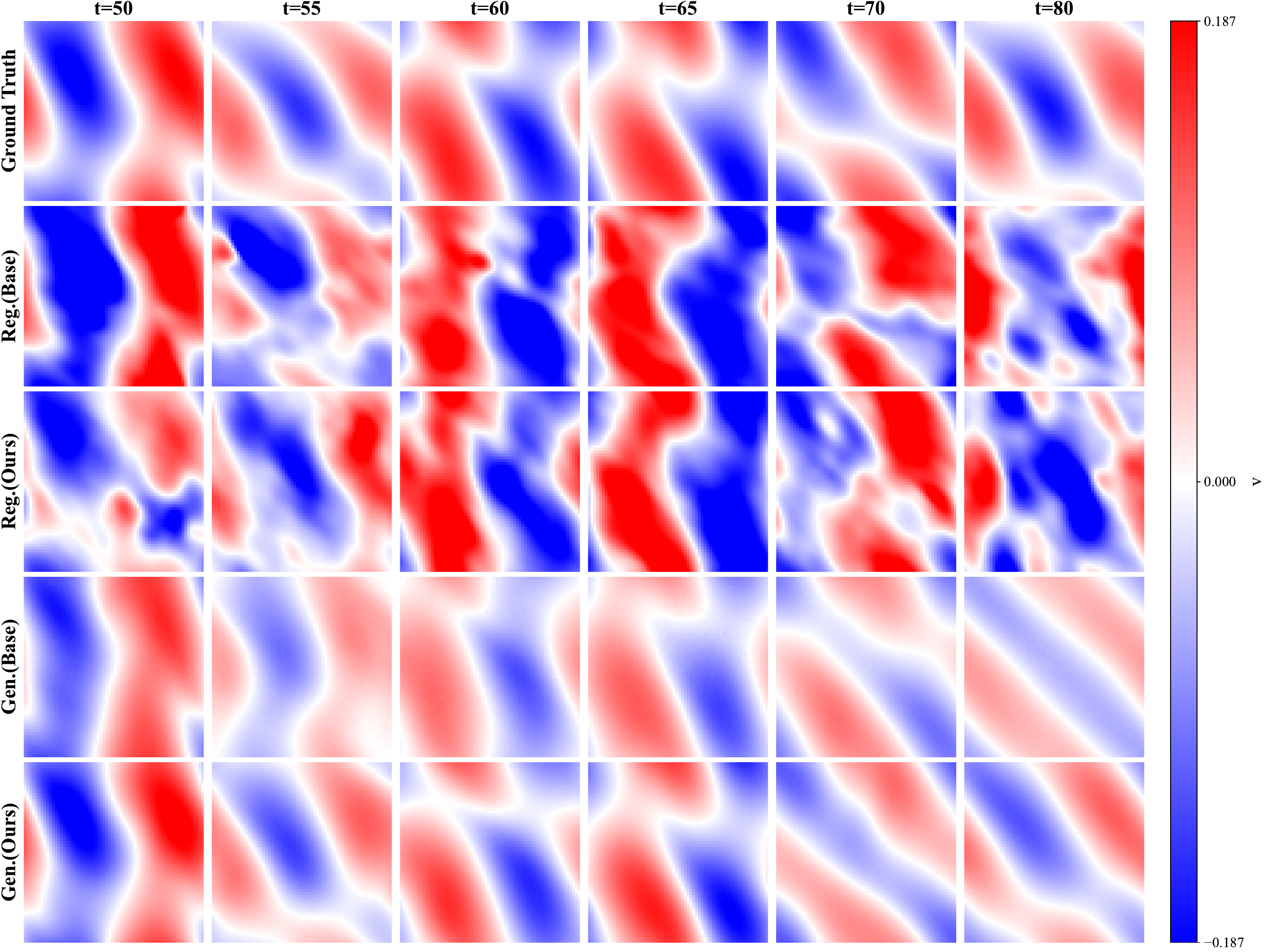} \caption{Meridional Velocity $v$}
    \end{subfigure}

    \begin{subfigure}[b]{0.48\textwidth}
        \includegraphics[width=\linewidth]{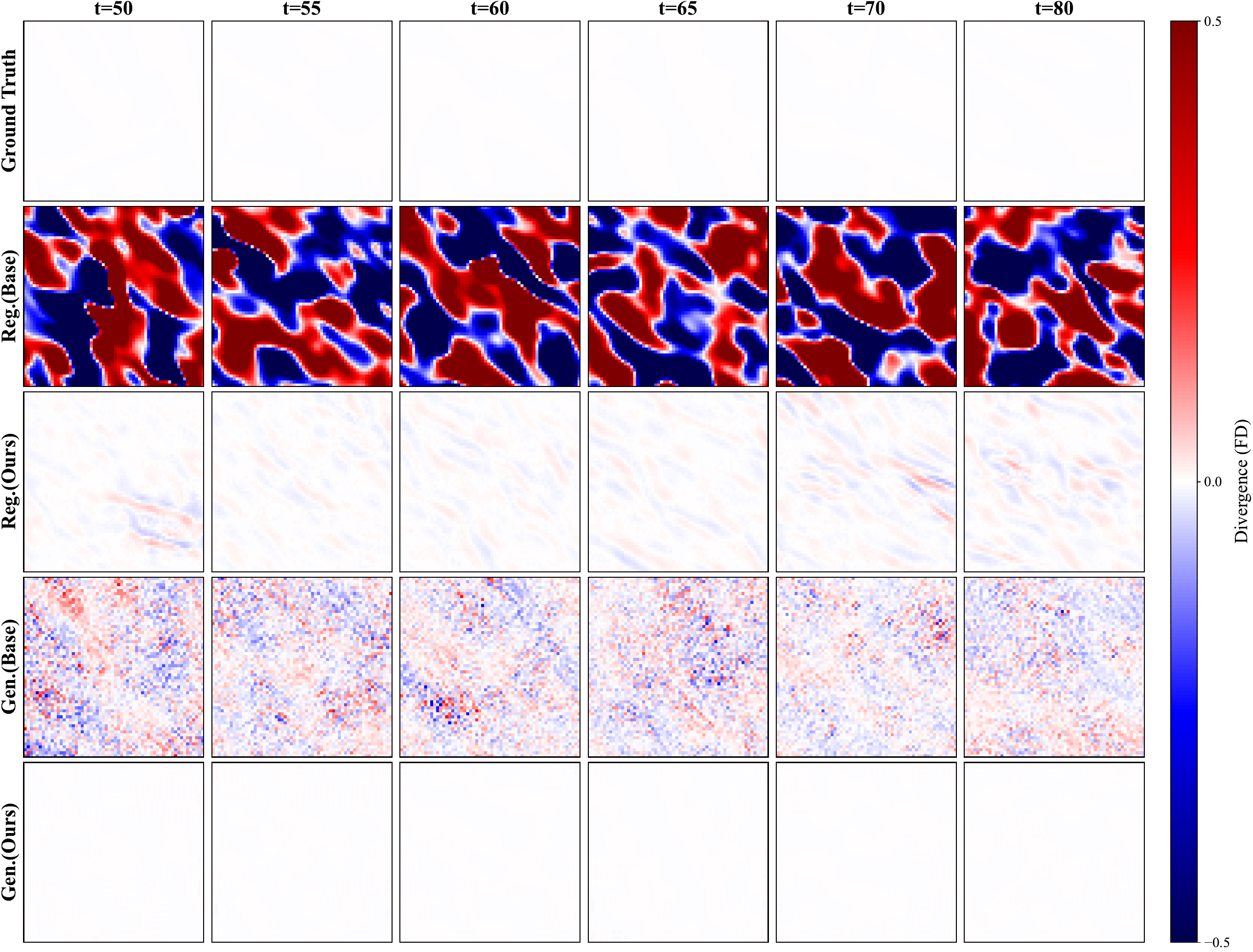} \caption{Divergence Accumulation}
    \end{subfigure}
    \hfill
    \begin{subfigure}[b]{0.48\textwidth}
        \includegraphics[width=\linewidth]{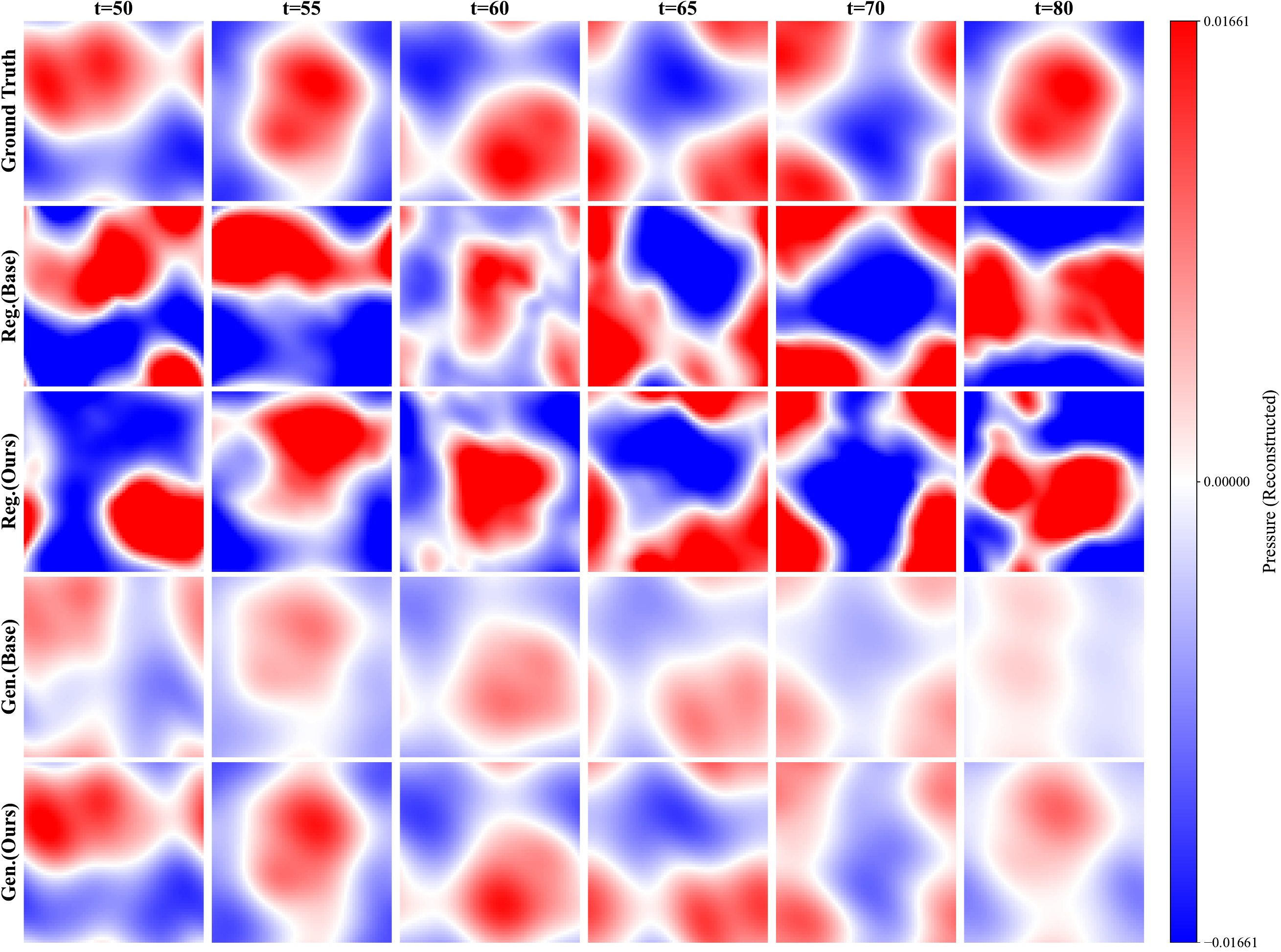} \caption{Pressure Collapse}
    \end{subfigure}

    \begin{subfigure}[b]{0.48\textwidth}
        \includegraphics[width=\linewidth]{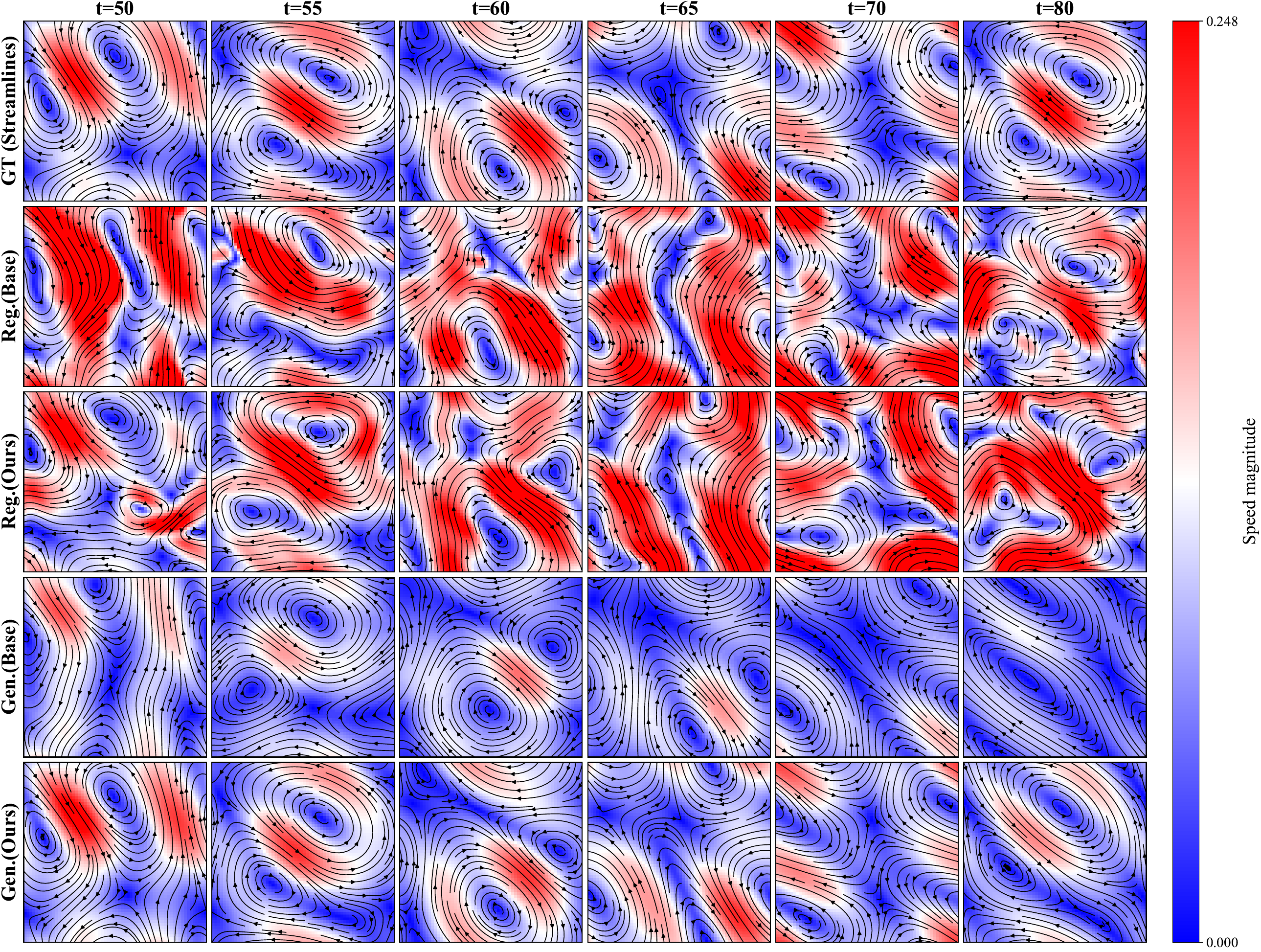} \caption{Streamlines}
    \end{subfigure}
    \hfill
    \begin{subfigure}[b]{0.48\textwidth}
        \includegraphics[width=\linewidth]{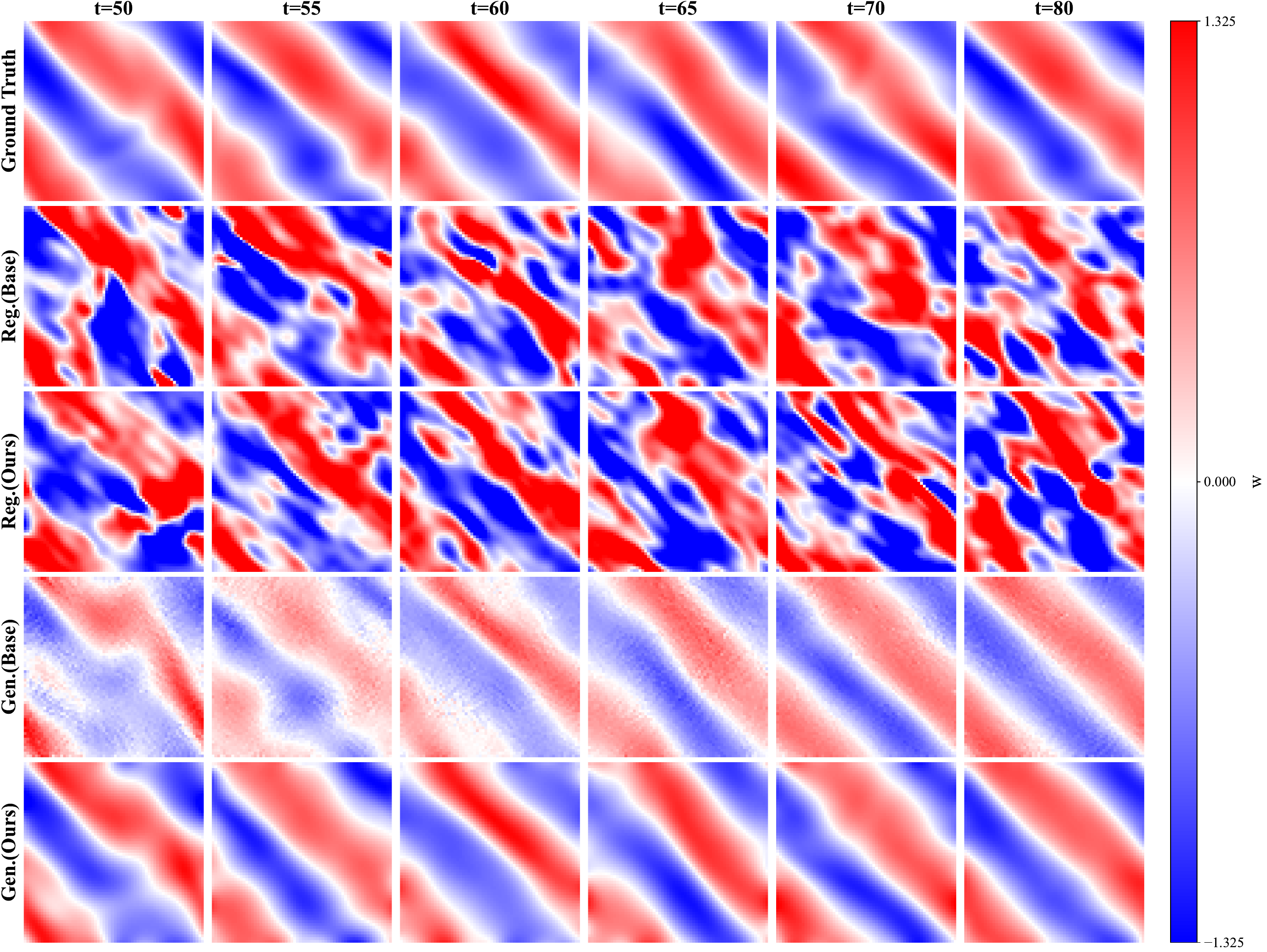} \caption{Vorticity $\omega$}
    \end{subfigure}
    
    \caption{\textbf{Short-term Stability (Extrapolation).} Snapshots at $T=100$.
    Top: Velocity drift in baseline.
    Middle: Conservation failure.
    Bottom: Topological degradation.}
    \label{fig:short_fields}
\end{figure*}

\begin{figure*}[ht]
    \centering
    \begin{subfigure}[b]{0.92\textwidth}
        \centering
        \includegraphics[width=\linewidth]{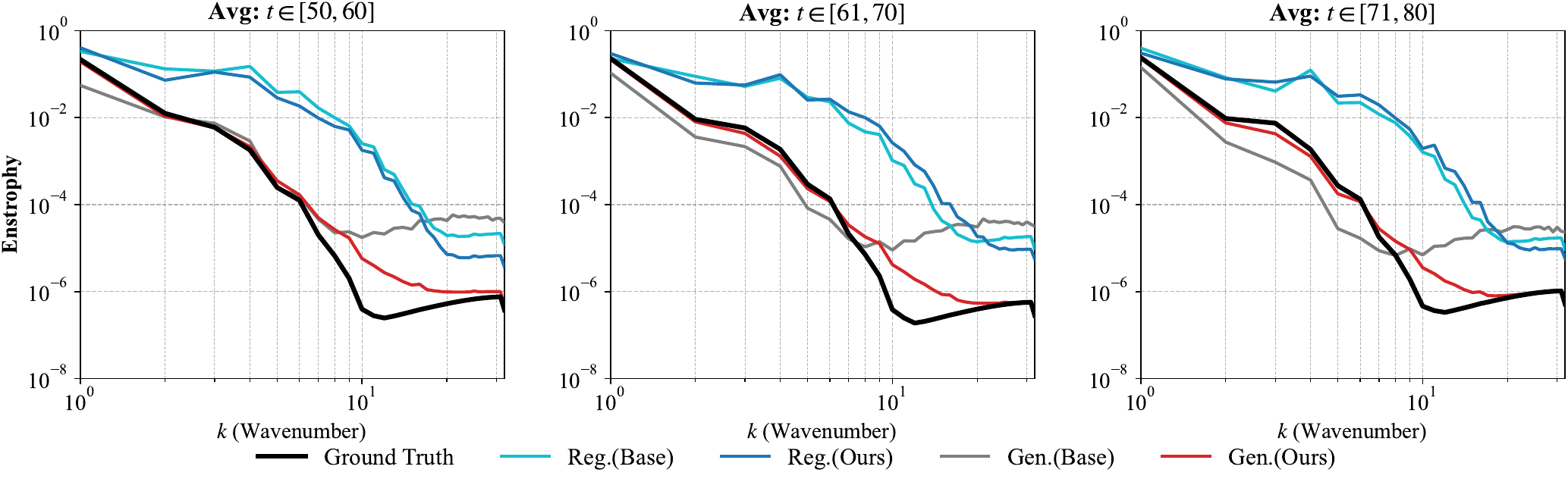}
        \caption{Enstrophy Spectrum}
    \end{subfigure}

    \begin{subfigure}[b]{0.46\textwidth}
        \includegraphics[width=\linewidth]{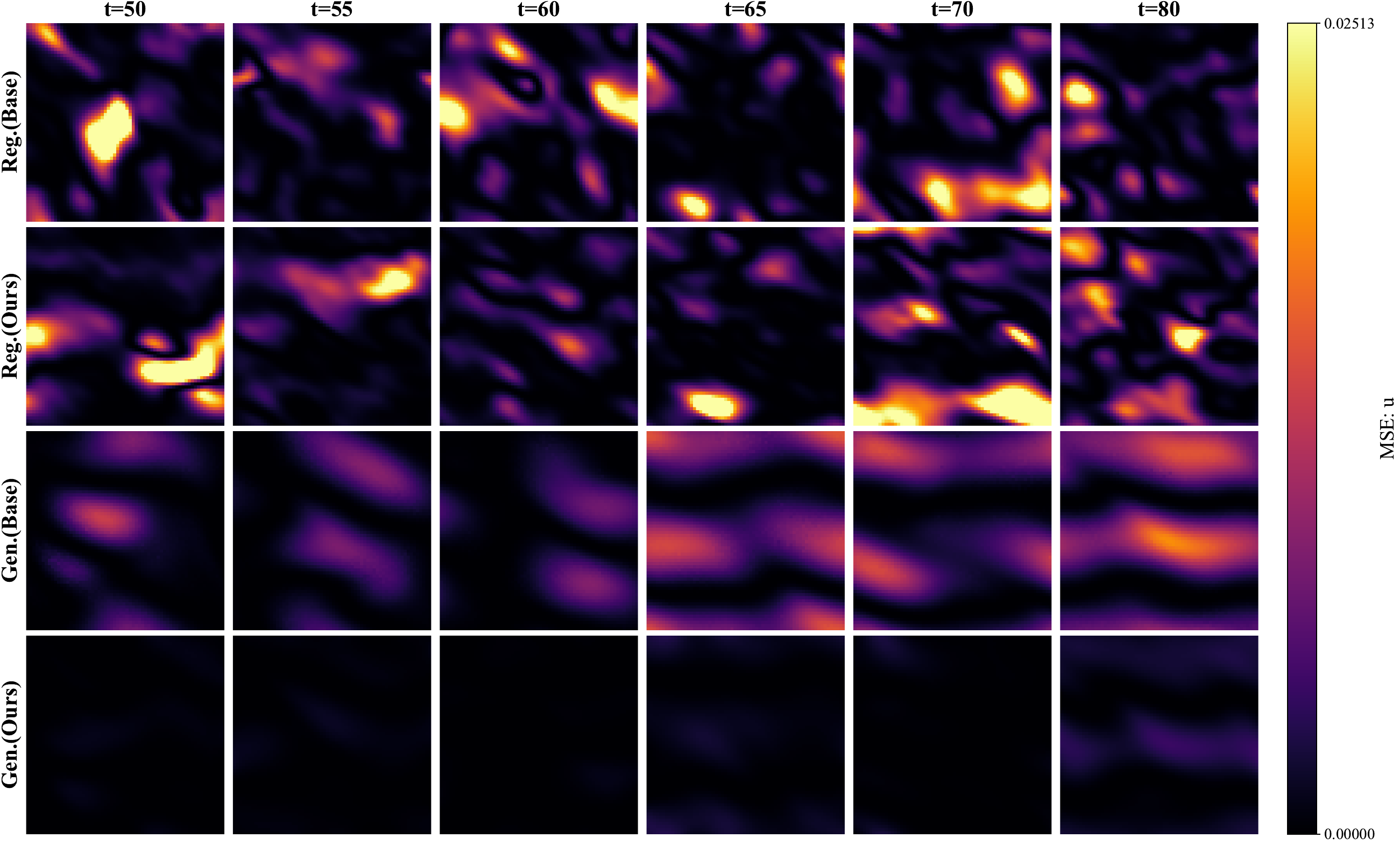} \caption{MSE $u$}
    \end{subfigure}
    \hfill
    \begin{subfigure}[b]{0.46\textwidth}
        \includegraphics[width=\linewidth]{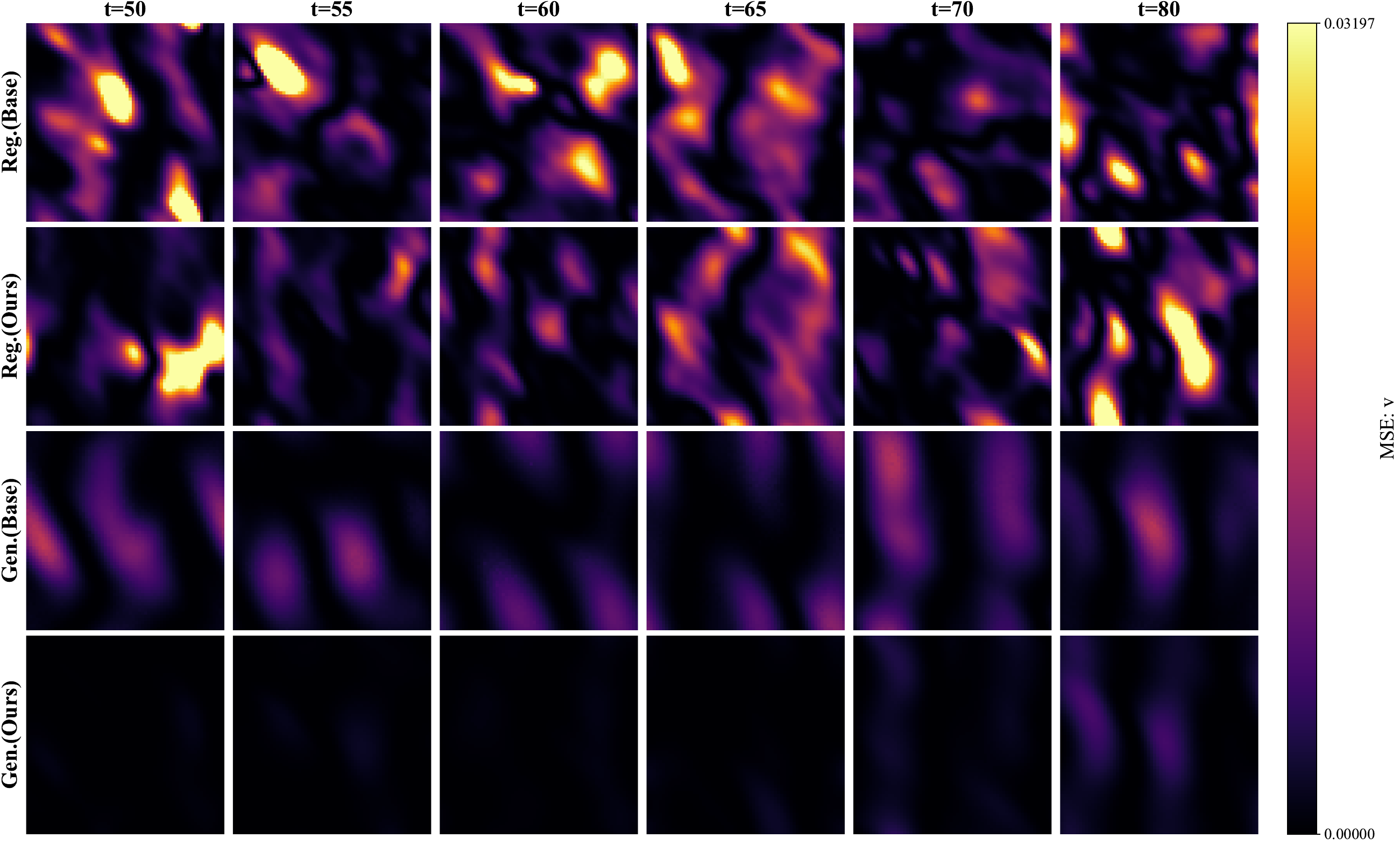} \caption{MSE $v$}
    \end{subfigure}

    \begin{subfigure}[b]{0.46\textwidth}
        \includegraphics[width=\linewidth]{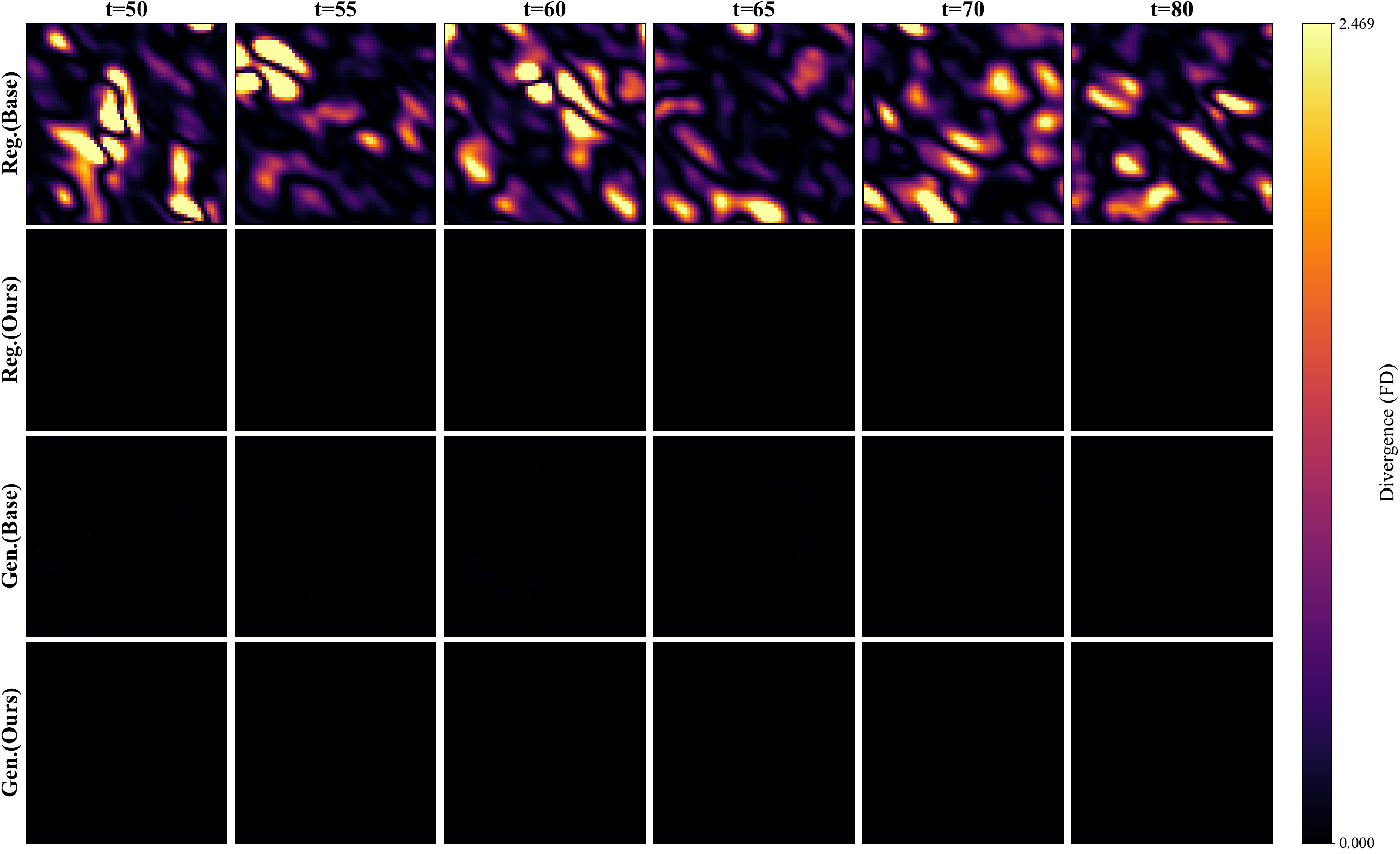} \caption{MSE Divergence}
    \end{subfigure}
    \hfill
    \begin{subfigure}[b]{0.46\textwidth}
        \includegraphics[width=\linewidth]{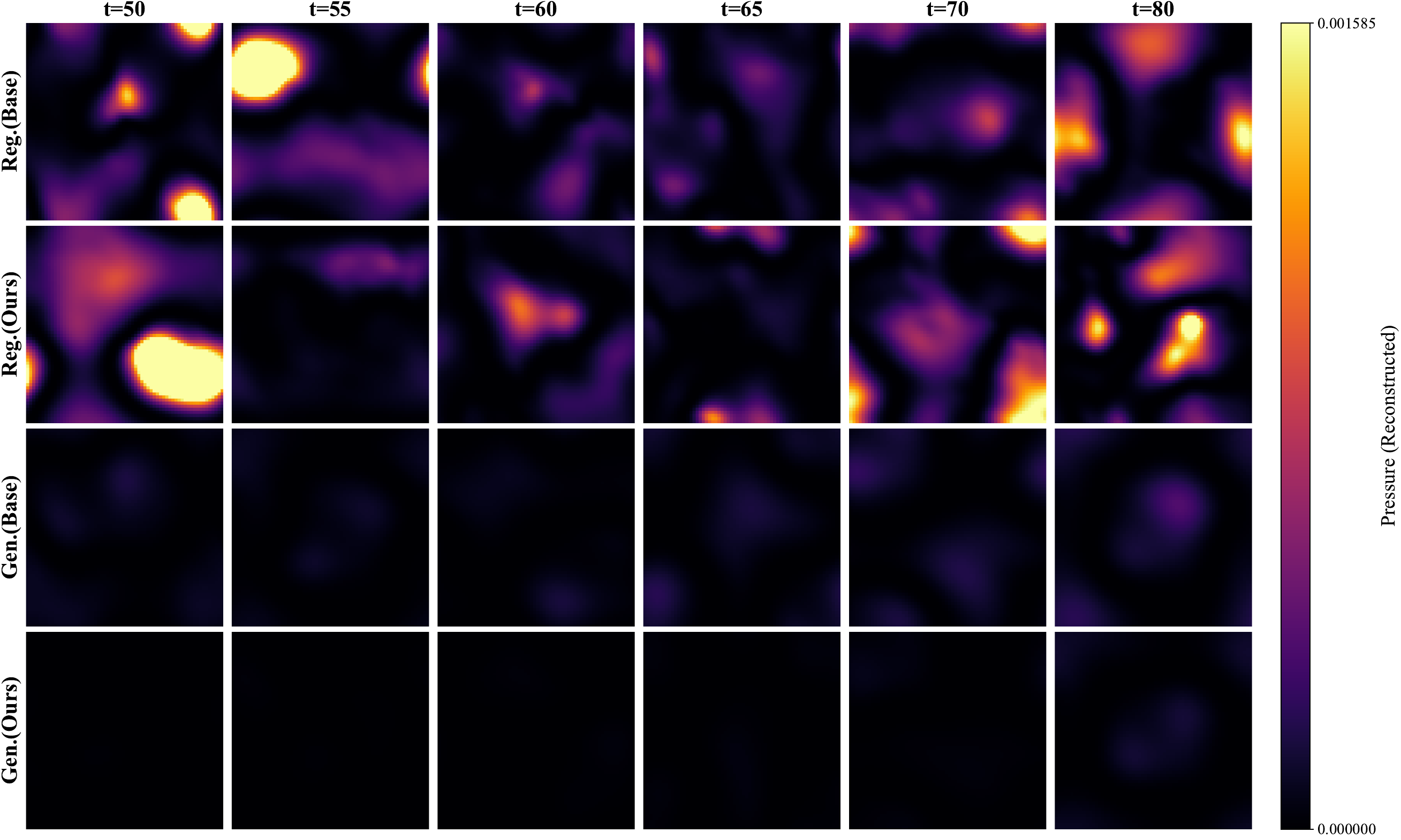} \caption{MSE Pressure}
    \end{subfigure}

    \begin{subfigure}[b]{0.46\textwidth}
        \includegraphics[width=\linewidth]{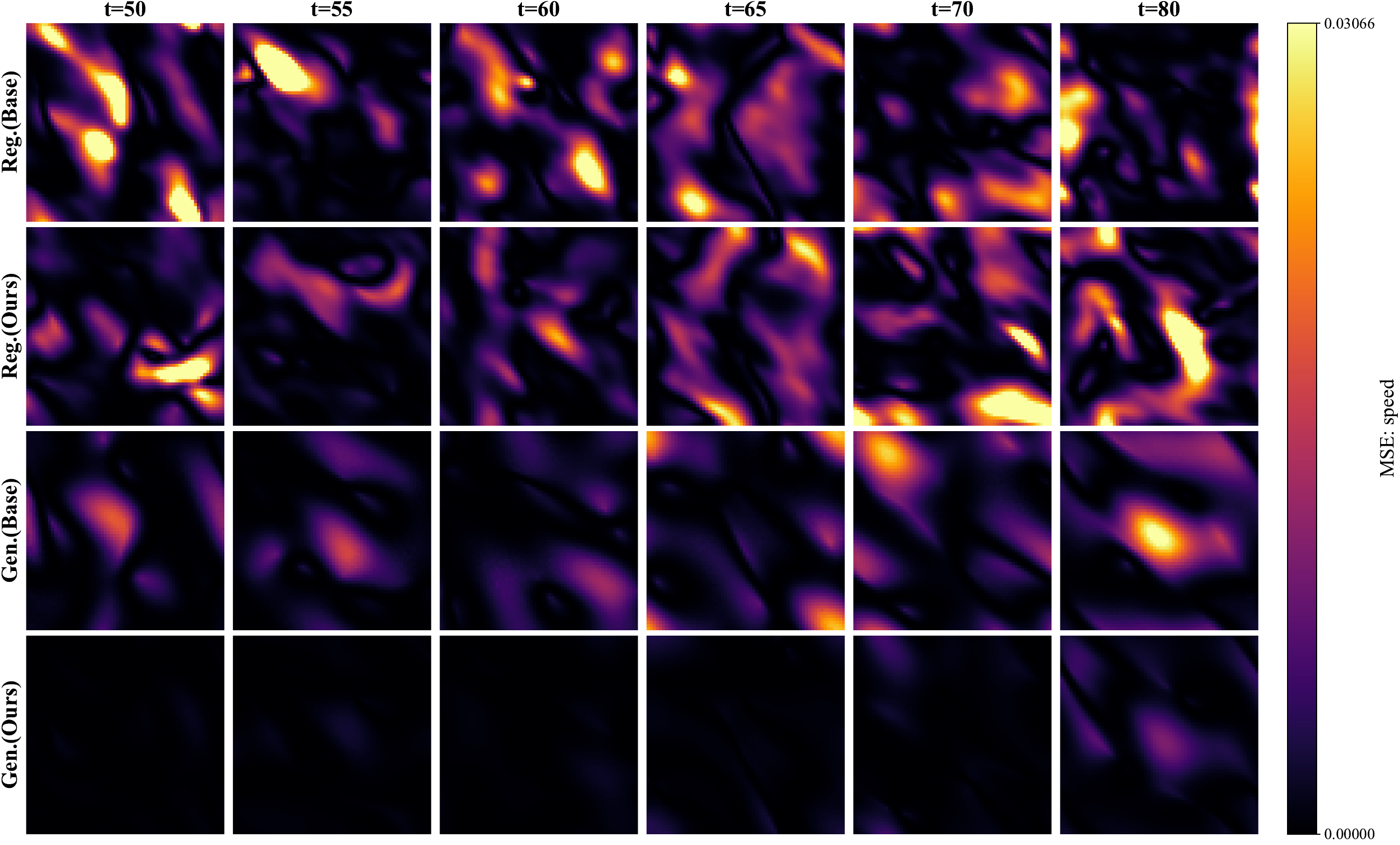} \caption{MSE Speed}
    \end{subfigure}
    \hfill
    \begin{subfigure}[b]{0.46\textwidth}
        \includegraphics[width=\linewidth]{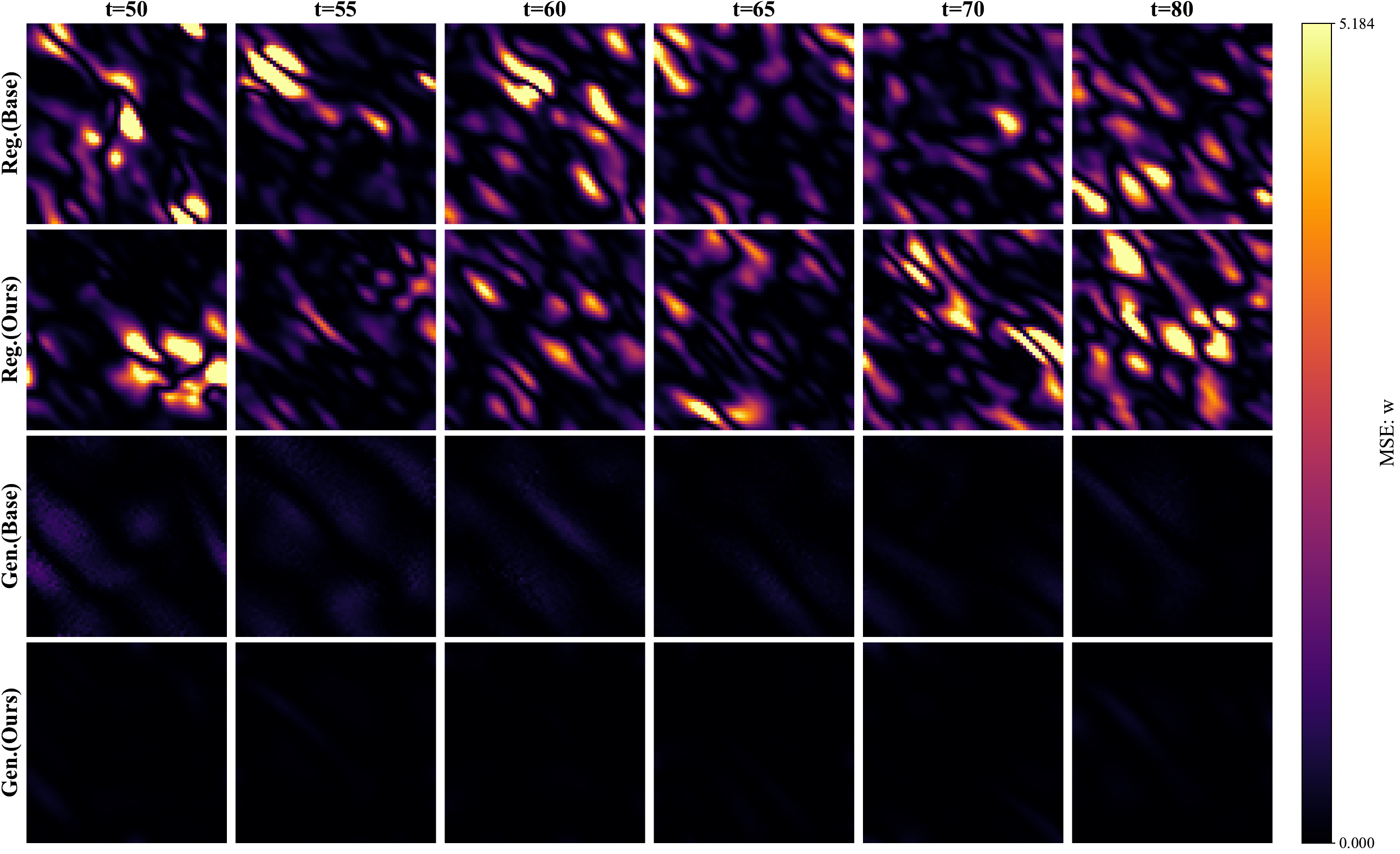} \caption{MSE $\omega$}
    \end{subfigure}
    
    \caption{\textbf{Stationarity and Error Growth (Short-term).} 
    Top: Enstrophy spectrum remains stable.
    Rows 2-4: Significant divergence errors appear in the baseline.}
    \label{fig:short_mse}
\end{figure*}

\begin{figure*}[ht]
    \centering
    \begin{subfigure}[b]{0.48\textwidth}
        \includegraphics[width=\linewidth]{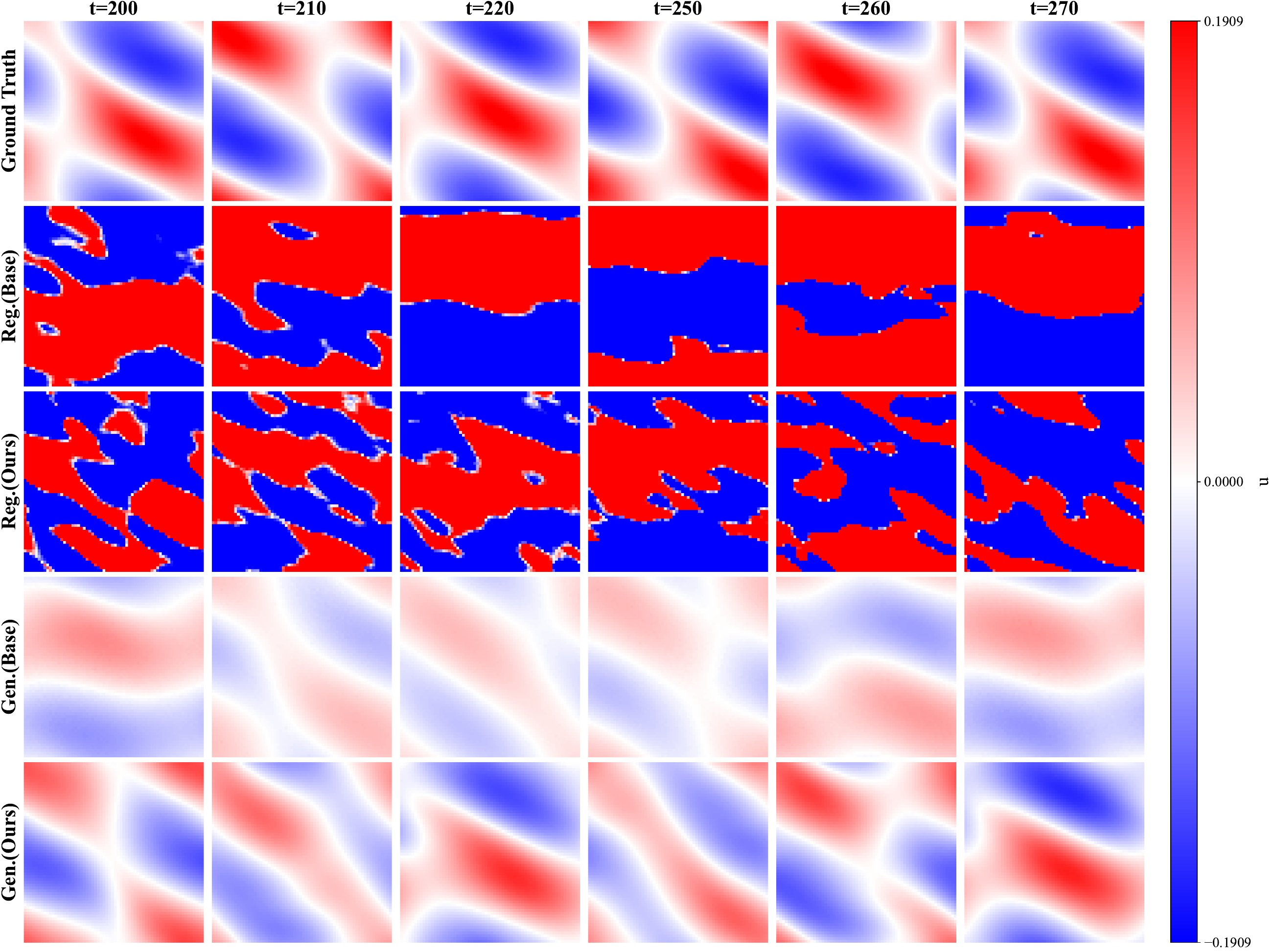} \caption{Zonal Velocity $u$}
    \end{subfigure}
    \hfill
    \begin{subfigure}[b]{0.48\textwidth}
        \includegraphics[width=\linewidth]{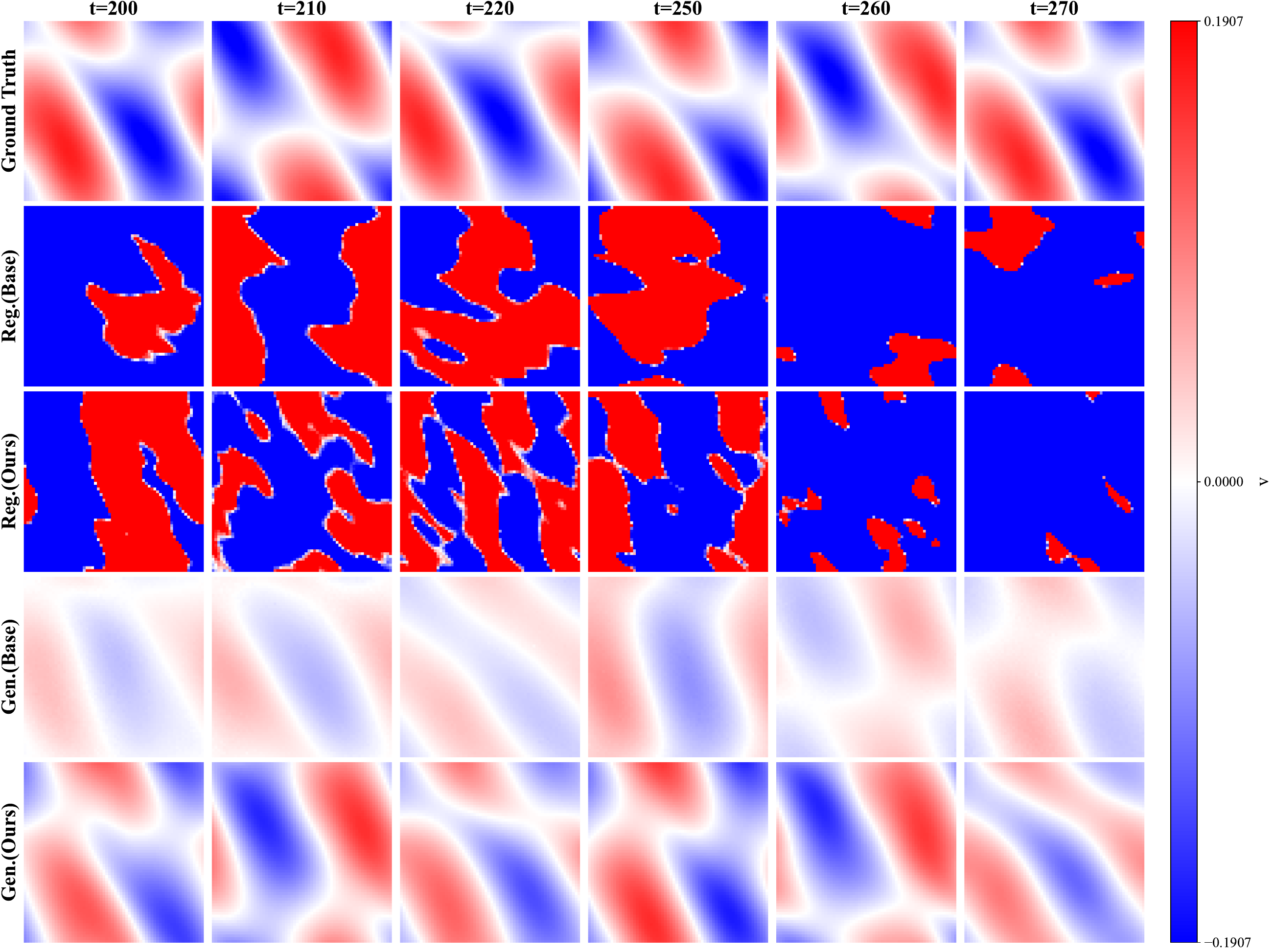} \caption{Meridional Velocity $v$}
    \end{subfigure}

    \begin{subfigure}[b]{0.48\textwidth}
        \includegraphics[width=\linewidth]{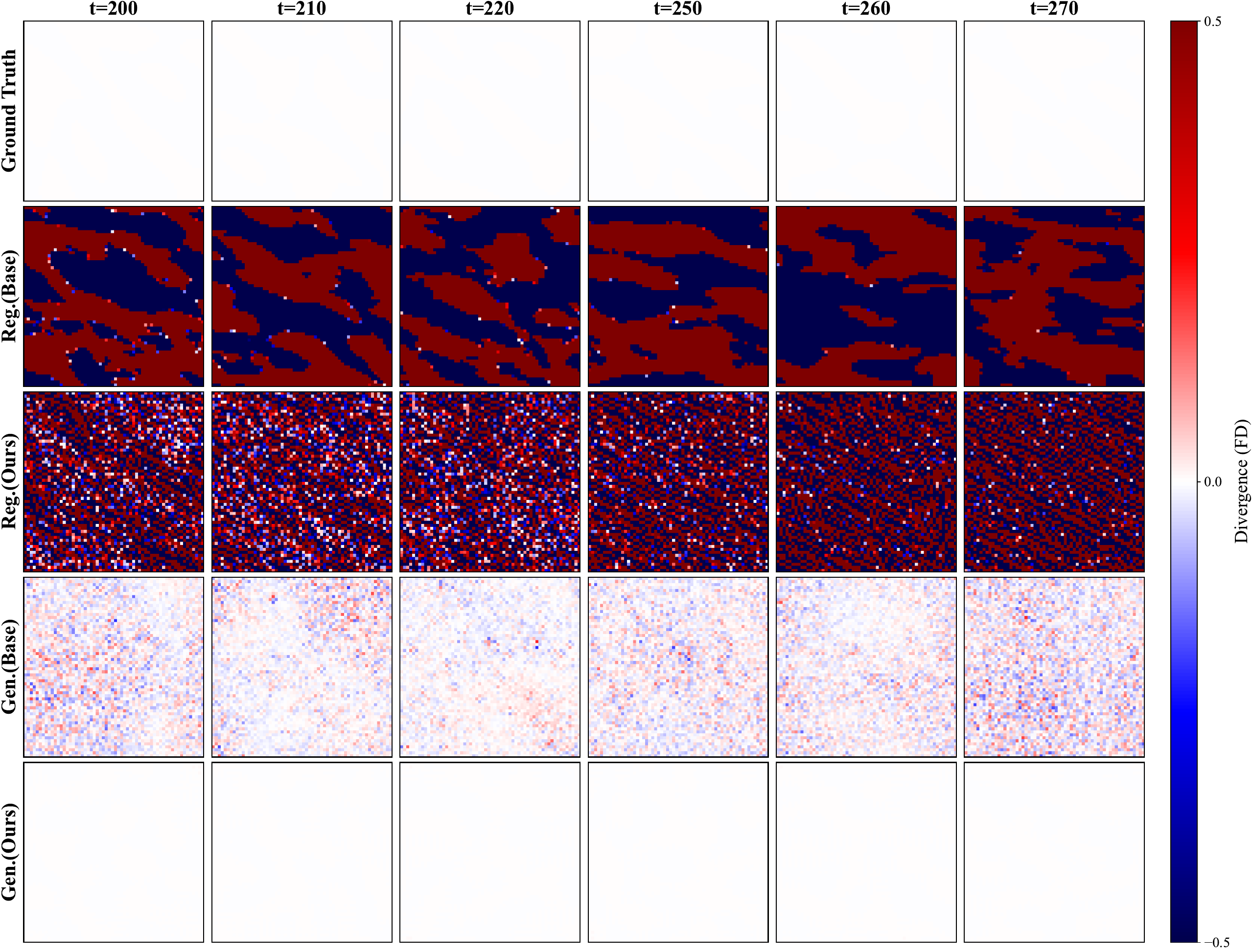} \caption{Divergence Accumulation}
    \end{subfigure}
    \hfill
    \begin{subfigure}[b]{0.48\textwidth}
        \includegraphics[width=\linewidth]{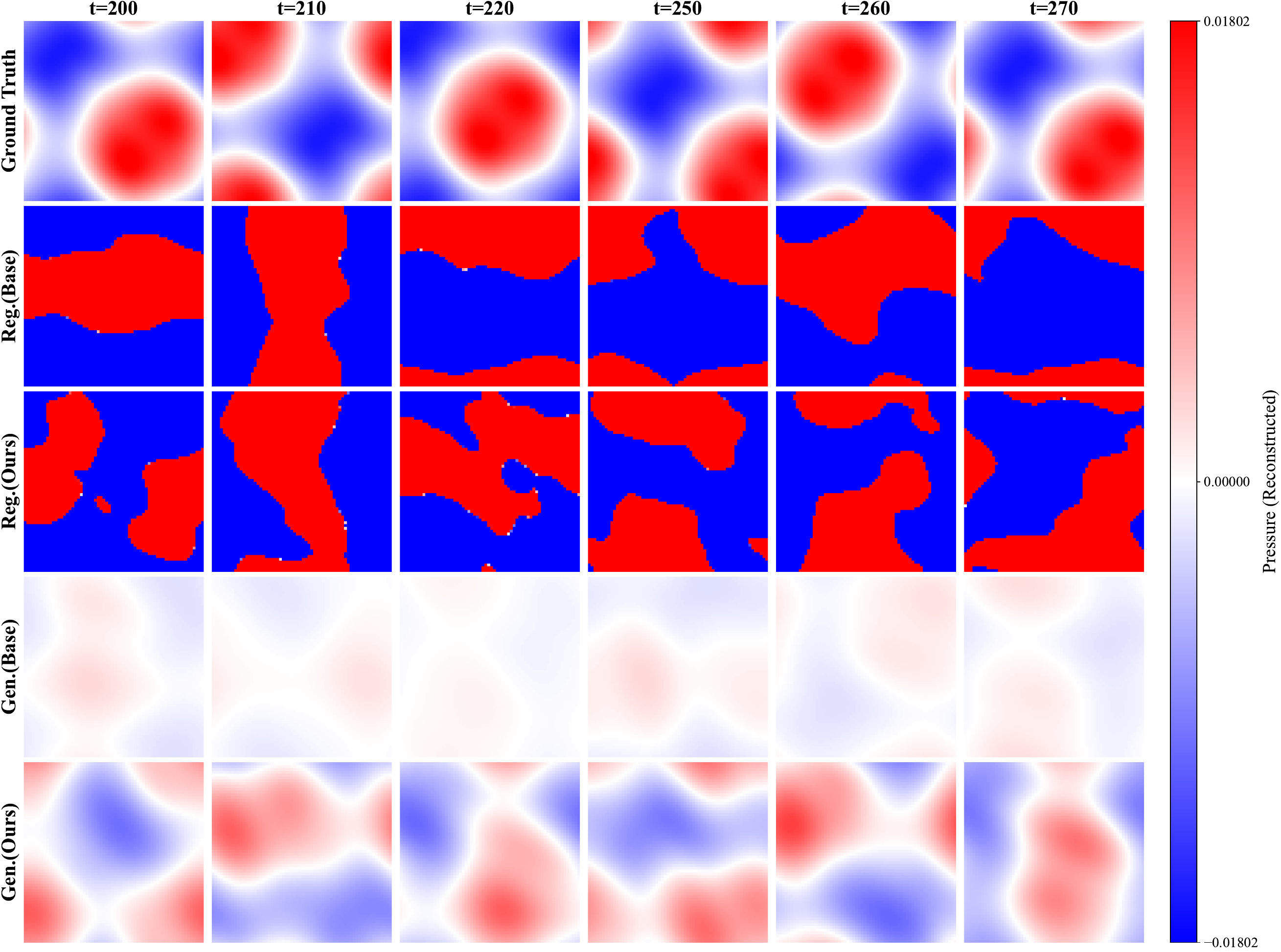} \caption{Pressure Collapse}
    \end{subfigure}

    \begin{subfigure}[b]{0.48\textwidth}
        \includegraphics[width=\linewidth]{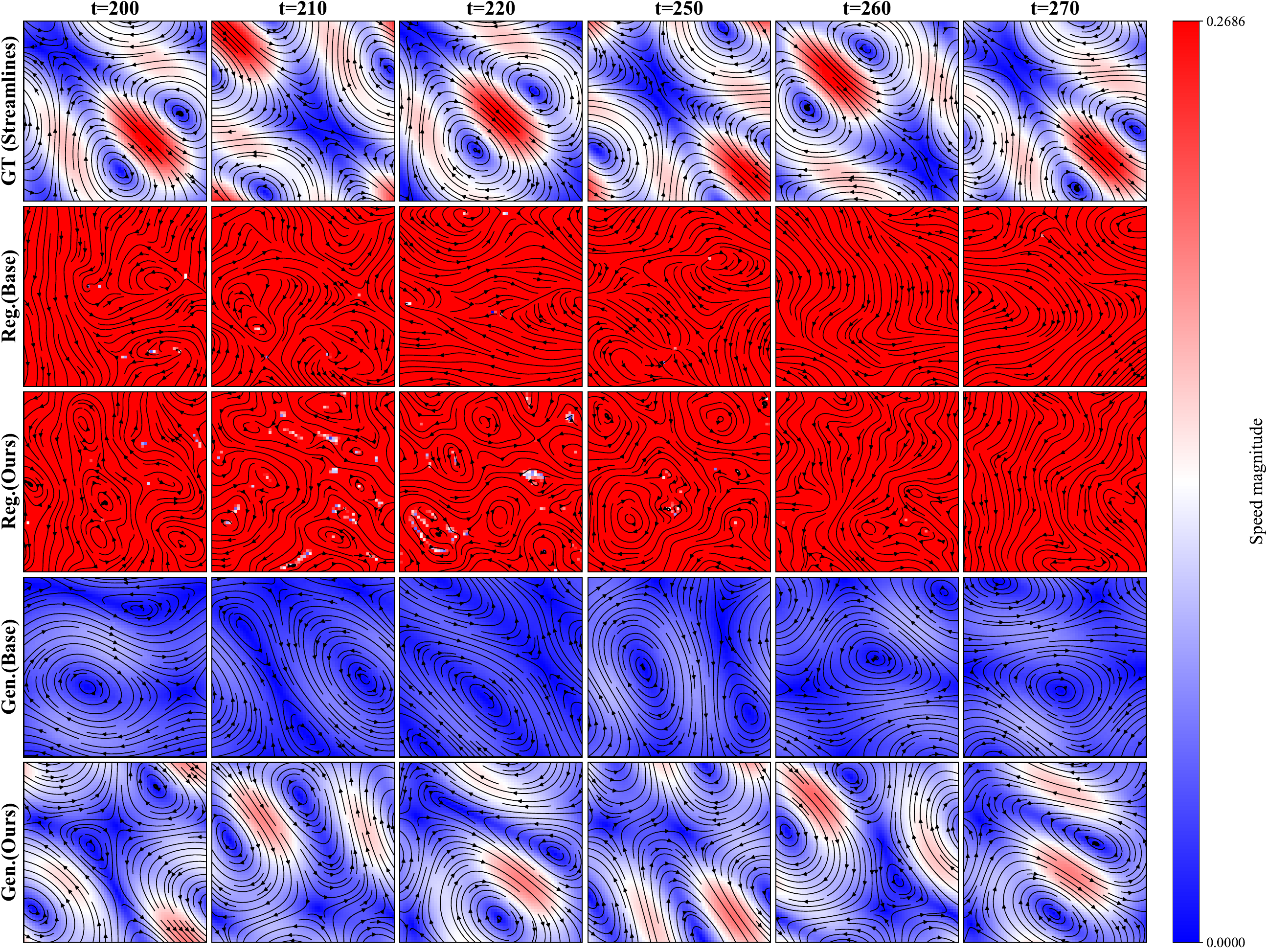} \caption{Streamlines}
    \end{subfigure}
    \hfill
    \begin{subfigure}[b]{0.48\textwidth}
        \includegraphics[width=\linewidth]{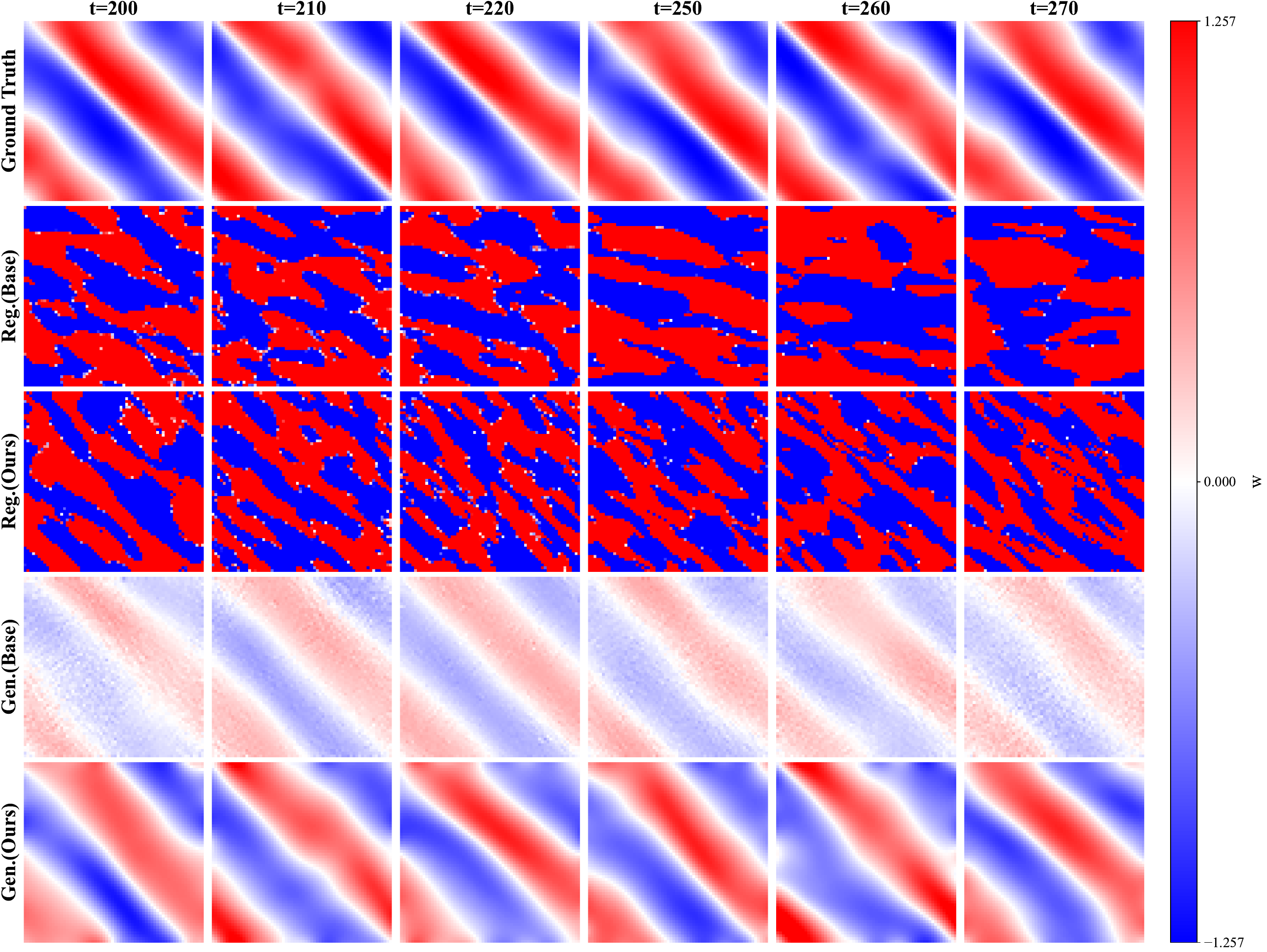} \caption{Vorticity $\omega$}
    \end{subfigure}
    
    \caption{\textbf{Long-term Stability (Extrapolation).} Snapshots at $T=300$.
    Top: Complete decoherence in baseline.
    Middle: Continued physical validity in ours.
    Bottom: Stationarity of topological features.}
    \label{fig:long_fields}
\end{figure*}

\begin{figure*}[ht]
    \centering
    \begin{subfigure}[b]{0.92\textwidth}
        \centering
        \includegraphics[width=\linewidth]{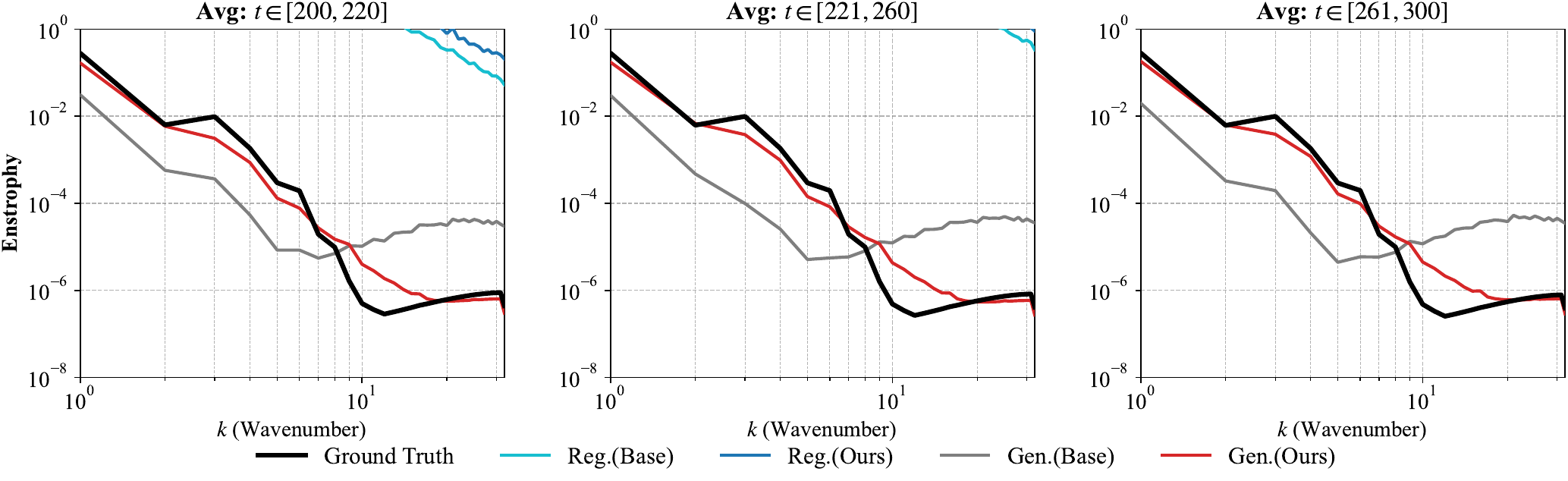}
        \caption{Enstrophy Spectrum}
    \end{subfigure}

    \begin{subfigure}[b]{0.46\textwidth}
        \includegraphics[width=\linewidth]{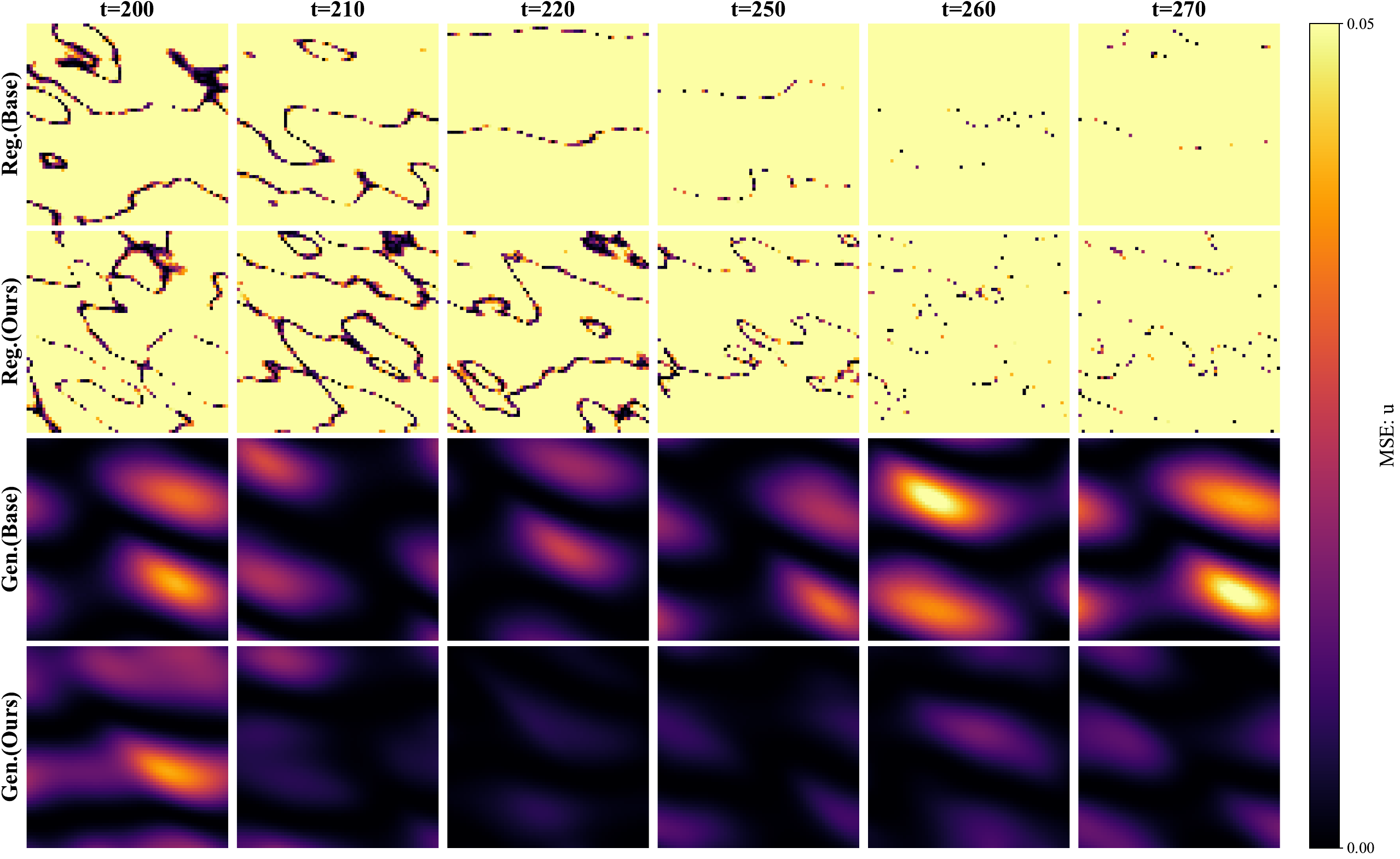} \caption{MSE $u$}
    \end{subfigure}
    \hfill
    \begin{subfigure}[b]{0.46\textwidth}
        \includegraphics[width=\linewidth]{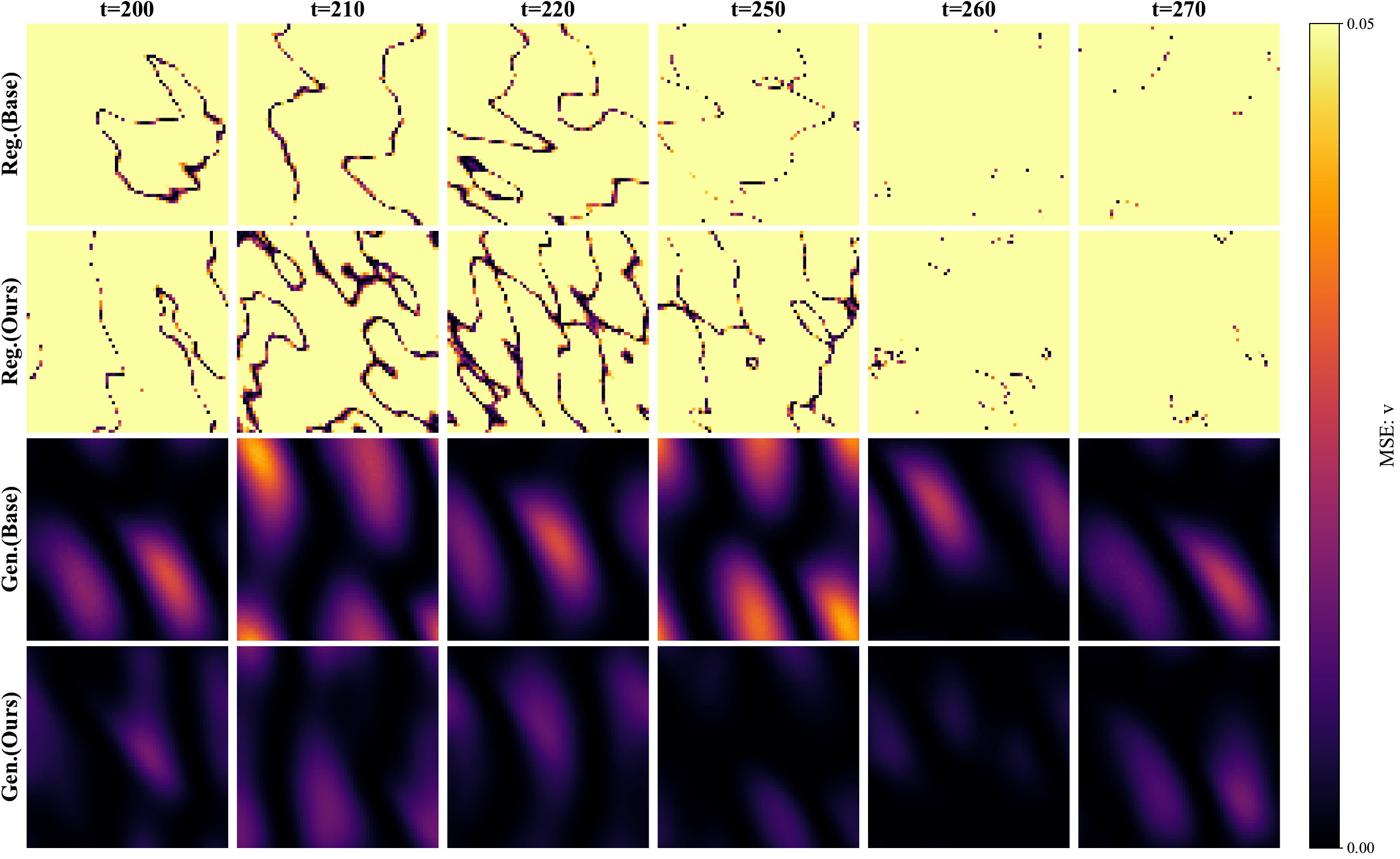} \caption{MSE $v$}
    \end{subfigure}

    \begin{subfigure}[b]{0.46\textwidth}
        \includegraphics[width=\linewidth]{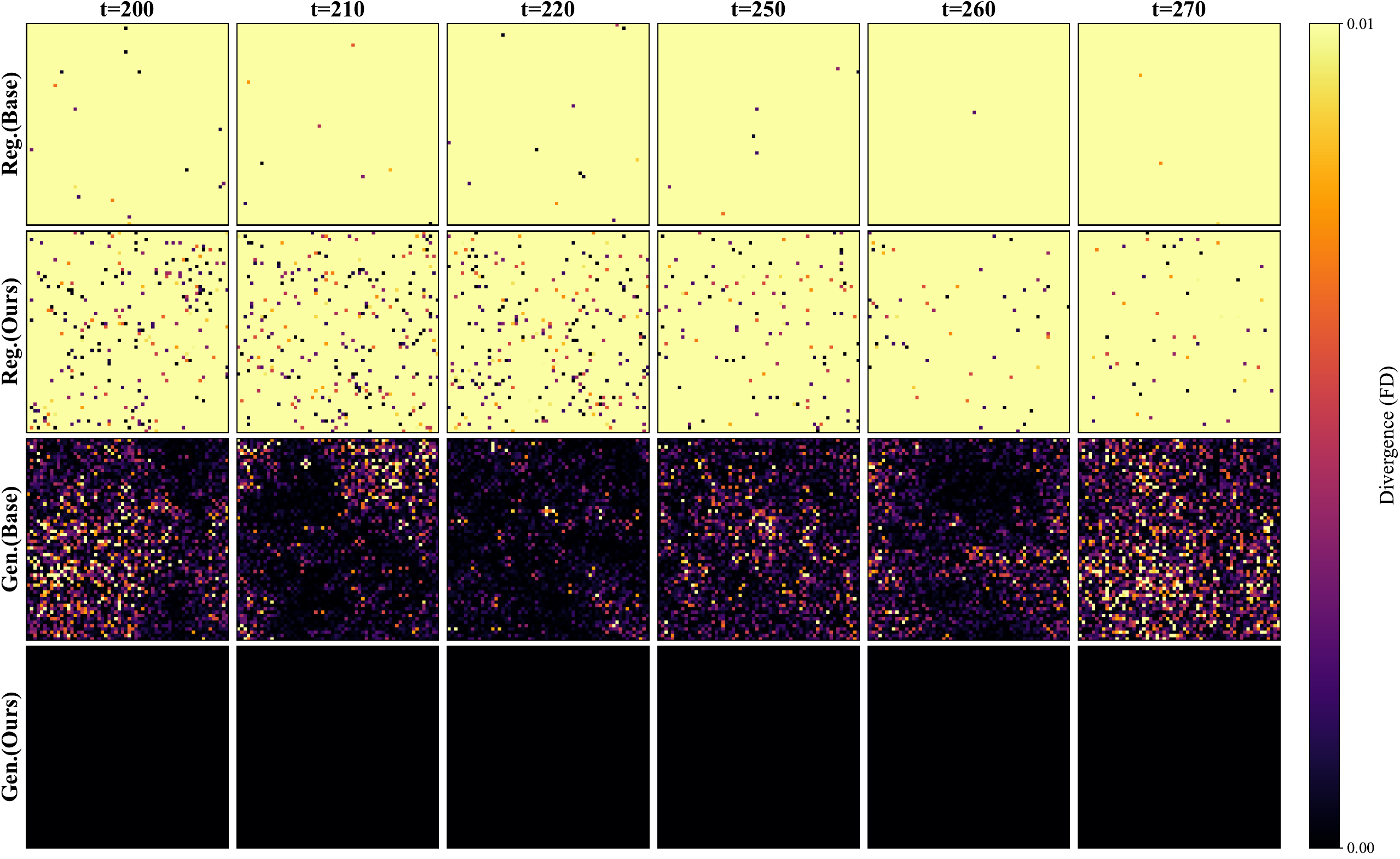} \caption{MSE Divergence}
    \end{subfigure}
    \hfill
    \begin{subfigure}[b]{0.46\textwidth}
        \includegraphics[width=\linewidth]{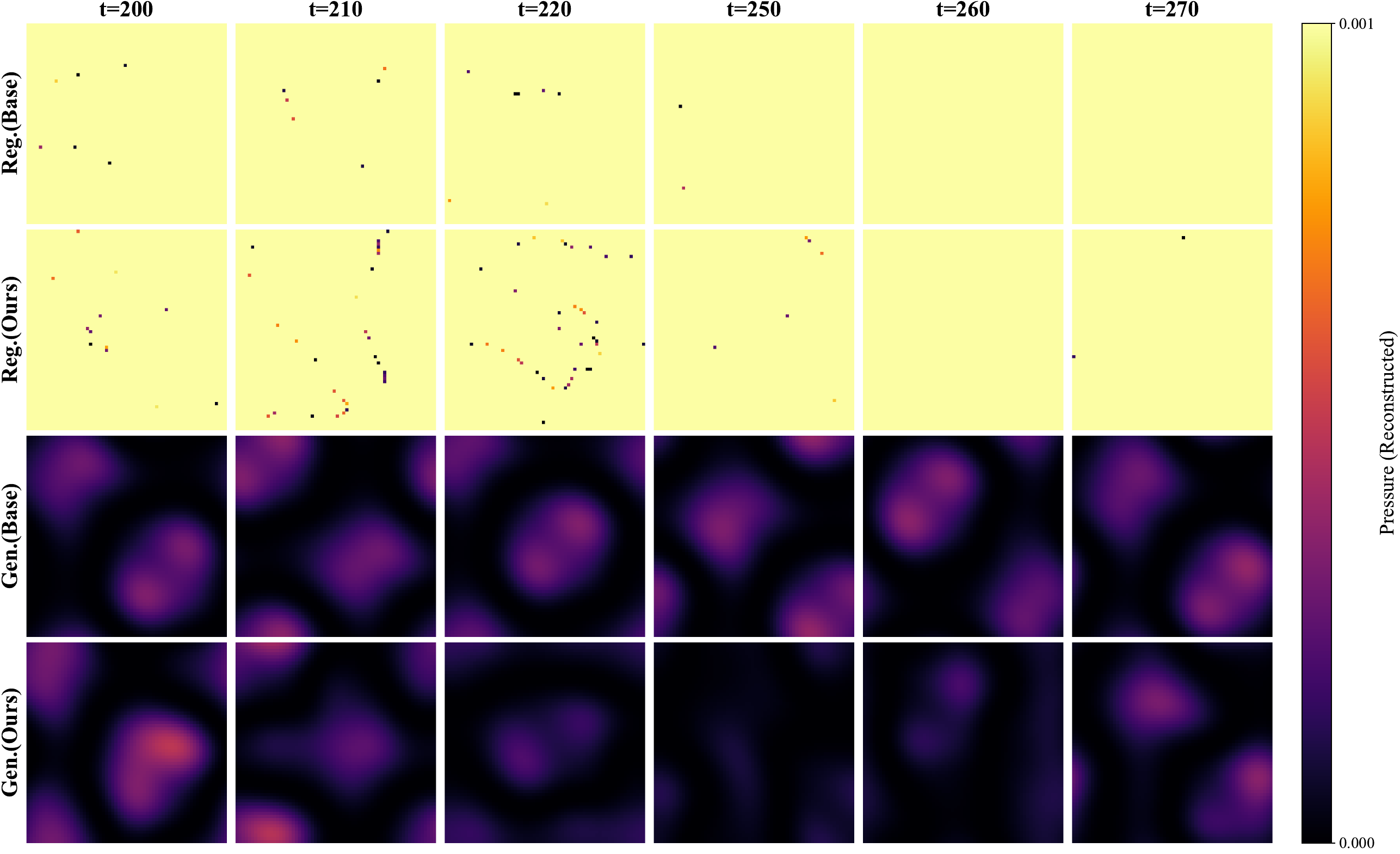} \caption{MSE Pressure}
    \end{subfigure}

    \begin{subfigure}[b]{0.46\textwidth}
        \includegraphics[width=\linewidth]{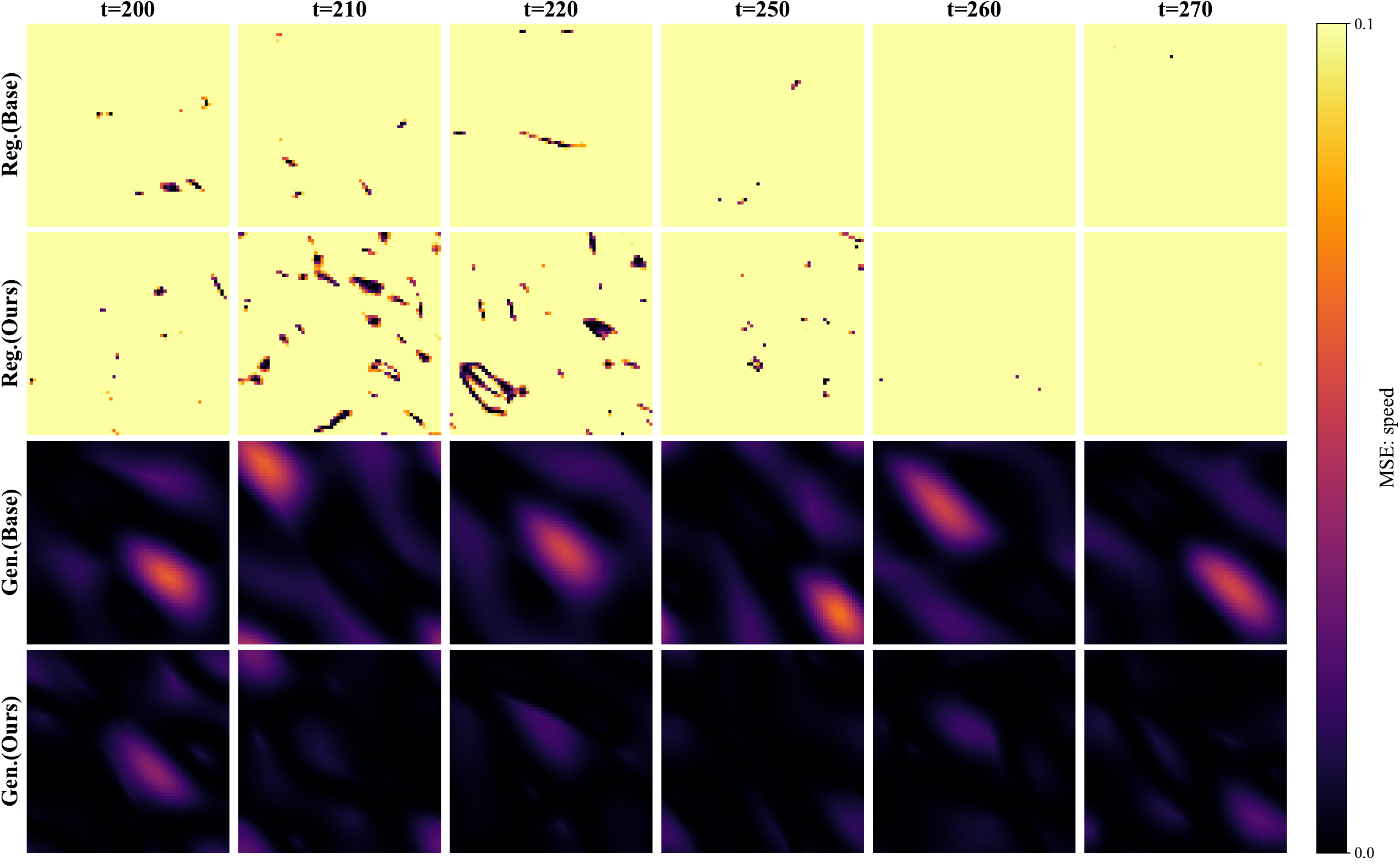} \caption{MSE Speed}
    \end{subfigure}
    \hfill
    \begin{subfigure}[b]{0.46\textwidth}
        \includegraphics[width=\linewidth]{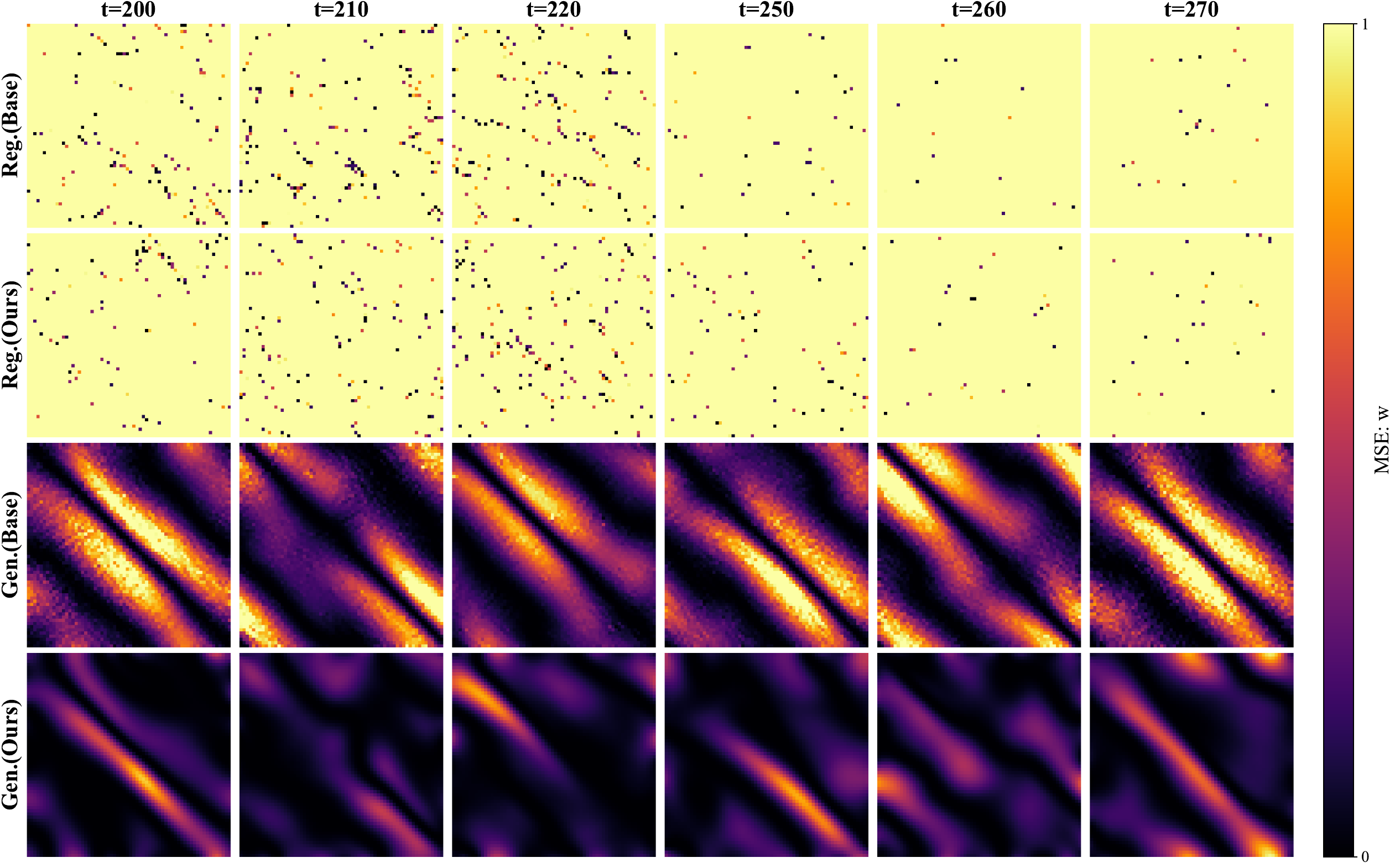} \caption{MSE $\omega$}
    \end{subfigure}
    
    \caption{\textbf{Long-term Stability (Extrapolation).} 
    Top: Enstrophy spectrum. 
    Rows 2-4: Baseline errors saturate, while ours remain bounded.}
    \label{fig:long_mse}
\end{figure*}

\end{document}